%% file: main.tex
\documentclass{article}

\PassOptionsToPackage{numbers, compress}{natbib}
\usepackage[preprint]{neurips_2026}

\usepackage[utf8]{inputenc}
\usepackage[T1]{fontenc}
\usepackage{hyperref}
\usepackage{url}
\usepackage{booktabs}
\usepackage{amsfonts}
\usepackage{amsmath}
\usepackage{amssymb}
\usepackage{amsthm}
\usepackage{mathtools}
\usepackage{nicefrac}
\usepackage{microtype}
\usepackage{xcolor}
\usepackage{graphicx}
\usepackage{multirow}
\usepackage{subcaption}
\usepackage{bm}
\usepackage{algorithm}
\usepackage{algorithmic}
\usepackage{float}
\usepackage{placeins}

\usepackage{enumitem}

\usepackage{inconsolata}

\newtheorem{proposition}{Proposition}
\newtheorem{assumption}{Assumption}

\newtheorem{remark}{Remark}

\newtheorem{definition}{Definition}

\newcommand{\Ltrain}{\mathcal{L}_{\mathrm{train}}}
\newcommand{\Lval}{\mathcal{L}_{\mathrm{val}}}
\newcommand{\Ainit}{A_{\mathrm{init}}}
\newcommand{\Aphi}{A_{\phi}}
\newcommand{\RIGR}{R_{\mathrm{IGR}}}

\title{Bilevel Graph Structure Learning, Revisited: \\ Inner-Channel Origins of the Reported Gain}
\author{
  Minkyoung Kim\textsuperscript{1} \quad Beakcheol Jang\textsuperscript{1}\thanks{Corresponding author} \\
  \textsuperscript{1}Graduate School of Information, Yonsei University \\ 
\href{mailto:minky@yonsei.ac.kr}{{minky@yonsei.ac.kr}}, \href{mailto:bjang@yonsei.ac.kr}{{bjang@yonsei.ac.kr}}
}

\begin{document}

\maketitle

\input{sections/00_abstract}
\input{sections/01_intro}
\input{sections/02_related}

\input{sections/03_framework}
\input{sections/04_method}

\input{sections/05_findings}
\input{sections/06_discussion}
\input{sections/07_conclusion}

\newpage
\bibliographystyle{plainnat}
\bibliography{references}

\appendix
\input{appendix_main}

\end{document}

%% file: sections/00_abstract.tex
\begin{abstract}
Bilevel graph structure learning is widely understood to improve graph neural networks by jointly optimizing model parameters and a learned graph structure, with the resulting performance gain attributed to the rewired adjacency. We find that this attribution may be overstated: training-dynamics effects in the inner loop, rather than the rewiring itself, capture a substantial share of the gain. To establish this, we introduce frozen-$\phi$, a control that freezes the graph while retaining the inner-loop training schedule. This decomposes the bilevel gain into an inner channel of $T$-step training dynamics with implicit gradient regularization and a graph channel of the graph rewiring itself. On spatio-temporal flow forecasting the inner channel matches or exceeds the full bilevel pipeline, accounting for $78$--$101\%$ of the gain; on node classification it accounts for $37$--$44\%$ under a Bernoulli edge-level parameterization. We also verify that classical spectral diagnostics can dissociate from task gain. We propose frozen-$\phi$ as a standardized diagnostic for bilevel graph structure learning, with graph distillation as a method-agnostic complement. A three-precondition framework further predicts the sign of the bilevel gain on all six benchmarks.
\end{abstract}

%% file: sections/01_intro.tex
\section{Introduction}
\label{sec:intro}

Bilevel graph structure learning (GSL) optimizes the graph and the model jointly through an outer loop on a validation objective, and the learned adjacency is credited with the resulting improvement~\citep{franceschi2019learning,chen2020iterative}. The bilevel procedure differs from vanilla training in two ways that could each account for its reported gain. The first is graph modification, executed by the outer loop. The second is an inner loop that runs $T$ gradient steps per outer iteration, which has known implicit regularization properties~\citep{barrett2021implicit,smith2021origin} unrelated to the graph. Existing evaluations treat the bilevel gain as flowing entirely from the first mechanism. We ask whether this attribution is correct.

We introduce the \emph{frozen-$\phi$ control}: the same bilevel training loop, with the same $T$-step inner optimization and warmup schedule, but with the graph parameter $\phi$ held fixed at its initial value. The control isolates extended inner optimization from graph modification, holding every other variable constant. This is the GSL analogue of the frozen-controller experiments that prompted reassessment of attribution in neural architecture search, where freezing the architecture parameter while preserving the training pipeline revealed that pipeline choices, rather than searched architectures, explained much of the gain~\citep{li2019random,yang2020nas}. Our goal is to reframe which component of bilevel GSL is actually responsible for the improvement.

A productive theoretical framework has characterized over-squashing through graph curvature~\citep{topping2022understanding}, commute time and effective resistance~\citep{digiovanni2023oversquashing,black2023understanding}, and the spatio-temporal Jacobian factorization for decoupled STGNNs~\citep{marisca2025oversquashing}. Our decomposition is complementary to this framework: the over-squashing bounds describe worst-case sensitivity propagation between nodes, while our control measures how much of bilevel's empirical performance lift is attributable to graph modification at all.

\paragraph{Contributions.} Our contributions are as follows.
\begin{enumerate}[leftmargin=*,topsep=0pt,itemsep=2pt]

\item We show that, across both spatio-temporal flow forecasting and node classification, a substantial fraction of reported bilevel-GSL gains is attributable to inner-loop training dynamics rather than to the learned graph structure, with the inner channel accounting for $78$--$101\%$ on flow forecasting and $37$--$44\%$ on node classification.

\item We corroborate this attribution through three independent diagnostics (edge corruption that varies graph quality, alternative training schedules including weight decay and end-to-end joint training, and convergent triangulation on PeMS07), complemented by a $T$-sweep that distinguishes two inner-loop training regimes (mini-batch reuse versus full-batch reset), all unified by an implicit-gradient-regularization mechanism.
\item We propose frozen-$\phi$ as a standardized diagnostic for gradient-based bilevel GSL that modifies rather than replaces the initial adjacency, with graph distillation as a method-agnostic complement, accompanied by a three-precondition framework on the predefined adjacency that predicts when bilevel rewiring helps.
\end{enumerate}

\paragraph{What this paper does not claim.} We do not claim that learned graph structure is useless in bilevel GSL, nor that bilevel methods cannot be improved; we claim that, under standard reporting practice, the inner loop's contribution to measured gains is much larger than typically credited.

%% file: sections/02_related.tex
\section{Related Work}
\label{sec:related}

\paragraph{Bilevel and related graph structure learning.} Bilevel GSL parameterizes the adjacency jointly with model weights via nested optimization with an outer validation objective. LDS~\citep{franceschi2019learning} introduced a per-edge Bernoulli parameterization; IDGL~\citep{chen2020iterative} iteratively learns a continuous adjacency from multi-head weighted cosine similarity of node embeddings under a single (non-nested) training objective. GEN~\citep{wang2021graph} uses EM-based alternation, and Pro-GNN~\citep{jin2020graph} couples graph learning with adversarial defense (gains primarily on perturbed graphs). LCGS~\citep{hu2022efficient} is the closest prior bilevel method using a continuous (dual-normalized) adjacency rather than LDS's Bernoulli parameterization, with head-to-head comparison against LDS on Cora and Citeseer. Single-level alternatives~\citep{fatemi2021slaps,liu2022sublime,shang2021discrete} and ST-domain analogues~\citep{jiang2023spatio,zhang2020stgsl,tang2022stlgsl,cini2023sparse,manenti2025learning} use joint training rather than bilevel optimization. Adaptive STGNNs internalize graph learning into the forward pass: AGCRN~\citep{bai2020adaptive} uses an embedding-based internal adjacency, while GraphWaveNet~\citep{wu2019graph} uses such an internal matrix alongside diffusion (random-walk normalized) transitions of a pre-defined adjacency. The frozen-$\phi$ decomposition is defined specifically for bilevel GSL with gradient-based outer updates and incremental graph modification (Section~\ref{sec:framework}); single-level and adaptive approaches lie outside this scope.

\paragraph{Over-squashing and graph rewiring.} \citet{alon2021bottleneck} identified over-squashing in message-passing networks. Subsequent theory connected the phenomenon to graph curvature~\citep{topping2022understanding}, information contraction~\citep{banerjee2022oversquashing}, commute time~\citep{digiovanni2023oversquashing}, and effective resistance~\citep{black2023understanding}, motivating task-agnostic rewiring preprocessors~\citep{gasteiger2019diffusion,karhadkar2023fosr,nguyen2023revisiting,black2023understanding}. \citet{tori2025effectiveness} found that curvature-based rewiring gains often reflect hyperparameter sweep outliers, underscoring the importance of controlled evaluation. \citet{marisca2025oversquashing} extended the framework to spatio-temporal GNNs with a layer-wise multiplicative Jacobian factorization for time-then-space and time-and-space architectures, the structural foundation on which our architectural scope conditions rest. Our methodology shares the controlled-evaluation principle of~\citet{tori2025effectiveness} but isolates a distinct confound: the inner-loop training schedule, rather than hyperparameter sensitivity, in the bilevel optimization regime.

\paragraph{The attribution gap.} No prior study has isolated the contribution of graph modification from the training dynamics introduced by bilevel optimization, despite extensive theoretical and empirical work on graph rewiring. Recent benchmarks~\citep{li2023gslb,zhou2023opengsl} independently observe that GSL methods do not consistently outperform vanilla GNN counterparts on standard benchmarks, but do not identify the mechanism. The frozen-$\phi$ control fills this gap: it provides the first decomposition that separates training dynamics from graph modification within bilevel GSL.

%% file: sections/03_framework.tex
\section{The Frozen-\texorpdfstring{$\phi$}{phi} Decomposition}
\label{sec:framework}

\subsection{Spatio-temporal Jacobian bound}
\label{sec:framework-jac}

An STGNN with $L$ decoupled layers alternates spatial message passing (MP) and temporal convolution. Following~\citet{marisca2025oversquashing}, the sensitivity of node $v$'s output to node $u$'s input factorizes as
\begin{equation}
\label{eq:jacobian}
\left\|\nabla^u_i h^{v(L)}_t\right\| \leq \underbrace{B_T \cdot (R^{L_T L})_{i,0}}_{\text{temporal}} \cdot \underbrace{B_S \cdot (S^{L_S L})_{uv}}_{\text{spatial}},
\end{equation}
where $i$ and $t$ are the temporal indices of the input and output features respectively (with $i < t$), $R$ encodes temporal topology, $S = (\theta_u/\theta_m) I + c_1 \mathrm{diag}(\hat{A}^{\top} \mathbf{1}) + c_2 \hat{A}$ is the spatial MP matrix with $\theta_u, \theta_m$ parameterizing the self-loop and neighbor-aggregation weights, and $B_T, B_S$ are weight-dependent constants. The factorization is structurally guaranteed in decoupled architectures but broken in coupled ones. We assume $A$ symmetric, matching the distance-based adjacencies of all ST datasets used (Section~\ref{sec:method}).

We extend the worst-case propagation framework of~\citet{topping2022understanding,digiovanni2023oversquashing,marisca2025oversquashing} to the ST setting via a uniform spectral tightening of the spatial Jacobian factor (Proposition~\ref{prop:jac}, Appendix~\ref{app:proofs}; informally, the spatial factor $(S^{L_S L})_{uv}$ converges more uniformly to its stationary limit as $\lambda_2$ of the LCC subgraph increases, reducing entrywise variation of sensitivity across node pairs). For long-range pairs initially below the stationary value, this homogenization corresponds to increased sensitivity, aligning with the over-squashing relief mechanism of these works. The bound therefore predicts that increasing $\lambda_2$ should improve forecasting; we test this prediction empirically (Section~\ref{sec:discussion}) and find spectral metrics dissociate from task gain on these benchmarks, which motivates the frozen-$\phi$ control as an alternative attribution channel.

\subsection{Definition and three-way comparison}
\label{sec:framework-def}

Following the principle of freezing the outer parameter while preserving the inner training pipeline, used as a diagnostic ablation in neural architecture search~\citep{li2019random,yang2020nas}, we define the frozen-$\phi$ control for bilevel GSL.

Consider a bilevel GSL procedure with outer parameter $\phi$ controlling the graph learner and inner parameter $\theta$ controlling the backbone GNN. The outer loop updates $\phi$ by computing the validation gradient at the inner solution $\theta^*(\phi) = \arg\min_\theta \Ltrain(\theta; \phi)$, where the inner minimization is typically truncated to $T$ gradient steps. Following~\citet{liu2019darts}'s first-order approximation ($\xi=0$), we evaluate $\nabla_\phi \Lval(\theta^*(\phi), A_\phi)$ treating $\theta^*$ as constant with respect to $\phi$, avoiding second-order differentiation through the inner trajectory. The \emph{frozen-$\phi$ control} is the ablation in which the outer update is disabled while the $T$-step inner loop is retained in full: $\phi$ is held at its initialization throughout training and $\theta$ is optimized by exactly the same inner procedure used in the full bilevel method.

The control yields three comparable configurations: \emph{vanilla} trains the backbone on $\Ainit$ with a single-step update; \emph{frozen-$\phi$} fixes $\phi$ at initialization while retaining the $T$-step inner loop; \emph{bilevel} runs the original method in its published form. We refer to the contribution of inner-loop training under a fixed graph as the \emph{inner channel}, and the contribution of graph modification by the outer loop as the \emph{graph channel}. Let $\mathcal{M}$ denote the evaluation metric (smaller is better: MAE for forecasting, error rate for classification). Define
\begin{equation}
\Delta_{\mathrm{inner}} = \mathcal{M}_\mathrm{Vanilla} - \mathcal{M}_\mathrm{Frozen\text{-}\phi}, \quad \Delta_{\mathrm{graph}} = \mathcal{M}_\mathrm{Frozen\text{-}\phi} - \mathcal{M}_\mathrm{Bilevel},
\label{eq:decomp}
\end{equation}
with $\Delta_{\mathrm{total}} = \Delta_{\mathrm{inner}} + \Delta_{\mathrm{graph}}$. The \emph{inner share} $\Delta_{\mathrm{inner}}/\Delta_{\mathrm{total}}$ summarizes how much of the bilevel gain is attributable to inner-loop training dynamics under a fixed graph. At $T = 1$ the frozen-$\phi$ configuration collapses to vanilla, providing a negative control.

The control is agnostic to graph parameterization. In our spatio-temporal bilevel implementation, $\phi$ controls continuous edge weights via $A_\phi = \Ainit \odot \mathrm{softmax}(W_\phi)$. In LDS, $\phi$ parameterizes Bernoulli edge probabilities $\theta_{ij} \in [0, 1]$ from which a binary adjacency is sampled, permitting edge addition, removal, and flickering; the frozen-$\phi$ variant for LDS initializes $\theta$ at the original adjacency $A$ following the LDS protocol~\citep{franceschi2019learning}. The three-way comparison is well-defined in both cases. We treat the spatio-temporal regime as primary and report node classification as a parameterization robustness check.

\subsection{Implicit gradient regularization in the bilevel inner loop}
\label{sec:framework-igr}

The decomposition isolates an inner channel whose magnitude can be large even when the graph is held fixed. We give a mechanistic account building on~\citet{barrett2021implicit,smith2021origin}: \citet{barrett2021implicit} showed that full-batch GD implicitly regularizes a modified loss with $\RIGR = (\eta/4)\|g\|^2$, and \citet{smith2021origin} extended this to SGD with random shuffling, deriving the additional minibatch term that yields $\RIGR = (\eta/4)(\|g\|^2 + \mathrm{tr}(\Sigma)/B)$. We give a Proposition characterizing the implicit gradient regularization arising in the bilevel inner loop, with distinct accumulation patterns under mini-batch reuse and full-batch reset.

\begin{proposition}[Expected implicit regularization in the inner loop]
\label{prop:igr}
Let $\Ltrain(\theta; \phi)$ be twice continuously differentiable and $\beta$-smooth in $\theta$, let the inner loop perform $T$ gradient updates $\theta_{t+1} = \theta_t - \eta \hat{g}_t$ where $\hat{g}_t$ is a mini-batch stochastic estimate of $\nabla_\theta \Ltrain$ with mean $g_t$ and per-example gradient covariance $\Sigma_t$ (so $\hat{g}_t$ itself has covariance $\Sigma_t/B$ for batch size $B$), and let $\phi$ be held fixed within the inner loop. Then under standard smoothness and step-size conditions ($\eta < 1/\beta$), the SGD trajectory tracks the continuous flow of a modified loss $\tilde{\Ltrain}(\theta; \phi) = \Ltrain(\theta; \phi) + \RIGR(\theta, \eta, B)$, with per-step implicit regularization
\begin{equation}
\label{eq:rigr}
\RIGR(\theta, \eta, B) = \frac{\eta}{4}\left(\|g\|^2 + \frac{\mathrm{tr}(\Sigma)}{B}\right) + O(\eta^2).
\end{equation}
The per-step regularization coefficient $\eta/4$ in the modified loss is fixed; over $T$ inner steps, the trajectory follows the continuous flow of $\tilde{\Ltrain}_{\mathrm{SGD}}$ for total time $T\eta$, so the deviation of $\theta_T$ from the gradient flow of the unmodified loss $\Ltrain$ grows with $T$ whenever the per-step gradient statistics remain bounded away from zero. Freezing $\phi$ does not alter the per-step implicit regularization mechanism: $\RIGR$ is a functional of the inner-loop gradient trajectory at the current $\phi$, and outer-loop updates to $\phi$ do not change its functional form.
\end{proposition}

A proof appears in Appendix~\ref{app:proofs}. Two qualitative consequences follow, corresponding to the two training regimes we study. \textbf{Mini-batch reuse regime}: our ST bilevel implementation carries $\theta$ across outer iterations and reuses the same mini-batch across $T$ inner steps; the per-step gradient is deterministic within a batch and $\RIGR$ takes the GD form $(\eta/4)\|g_t\|^2$ with fixed coefficient, but the trajectory traverses the modified flow of $\tilde{\Ltrain}_{\mathrm{GD}}$ for time $T\eta$, predicting a frozen-$\phi$ loss that varies monotonically with $T$. \textbf{Full-batch reset regime}: LDS~\citep{franceschi2019learning} uses full-batch gradients and reinitializes GCN parameters at each outer iteration, so freezing $\phi$ collapses $\theta_T$ to a deterministic function of $T$ that plateaus at inner convergence and the frozen-$\phi$ loss is $T$-invariant past the plateau.

\subsection{Scope conditions}
\label{sec:framework-scope}
The frozen-$\phi$ decomposition applies under three conditions: (i) the bilevel procedure uses a gradient-based outer update with $\phi$ parameterizing the graph continuously or via reparameterized discrete variables; (ii) the graph is modified incrementally from $\Ainit$ rather than replaced from scratch; and (iii) the backbone consumes an external adjacency. LDS~\citep{franceschi2019learning} and our spatio-temporal bilevel implementation satisfy all three. IDGL~\citep{chen2020iterative} (single-level) violates (i); GEN~\citep{wang2021graph} (EM-based replacement) violates (ii); Pro-GNN~\citep{jin2020graph} produces negative clean-graph gain on Cora, leaving the share decomposition undefined (Appendix~\ref{app:nc_mechanism}). Coupled backbones such as DCRNN~\citep{li2018diffusion} embed spatial aggregation inside GRU gates, falling outside the regime where the factored Jacobian of~\citet{marisca2025oversquashing} is structurally guaranteed; adaptive backbones pose a structural exclusion of (iii), since AGCRN~\citep{bai2020adaptive} uses only an internal adaptive matrix and GraphWaveNet~\citep{wu2019graph} learns an internal one alongside the external adjacency. We therefore restrict the frozen-$\phi$ analysis to decoupled polynomial-filter backbones (DiffConv, ChebConv, MPGRU). For node classification, we evaluate on Cora and Citeseer, on which LDS has a published and validated configuration; heterophilous benchmarks~\citep{platonov2023critical} fall outside our scope: no validated bilevel GSL configuration exists, and tuning from scratch would risk the hyperparameter confound of~\citet{tori2025effectiveness}.

\paragraph{Method-agnostic complement.} For methods outside the frozen-$\phi$ scope, graph distillation provides a complementary diagnostic: transfer $\Aphi$ from a completed GSL run into vanilla training of a target backbone and measure improvement over $\Ainit$. The two diagnostics measure different quantities: distillation conflates graph modification with graph-model co-adaptation by re-training $\theta$ from scratch on $\Aphi$, while frozen-$\phi$ holds the inner-loop schedule fixed and isolates only the graph modification. Distillation has weaker resolution than frozen-$\phi$ but applies to any GSL method producing positive gain on a clean graph, and is therefore the appropriate diagnostic when frozen-$\phi$ is inapplicable (Appendix~\ref{app:distillation}).

%% file: sections/04_method.tex
\section{Experimental Setup}
\label{sec:method}

\paragraph{Bilevel formulation.} We formulate task-aligned rewiring as bilevel optimization:
\begin{equation}
\min_\phi \Lval(\theta^*(\phi), A_\phi) \quad \mathrm{s.t.} \quad \theta^*(\phi) = \arg\min_\theta \Ltrain(\theta, A_\phi),
\end{equation}
where a policy $\pi_\phi$ reweights the adjacency: $A_\phi = A \odot \mathrm{softmax}_{\mathrm{row}}(W_\phi)$, preserving sparsity and degree normalization in the spatio-temporal regime. We use~\citet{liu2019darts}'s first-order approximation ($\xi=0$): the inner-loop $\theta$ updates treat $A_\phi$ as detached from the computation graph, and the outer gradient is evaluated as $\nabla_\phi \Lval(\theta, A_\phi)$ at the resulting $\theta$ without backpropagating through the inner trajectory. Training has two phases: warmup ($E_w$ epochs, $\theta$ only) then bilevel (alternating $T$ inner steps for $\theta$ and one outer step for $\phi$). For node classification we use LDS~\citep{franceschi2019learning} with its Bernoulli edge parameterization, full-batch inner updates, and GCN inner parameters reset per outer iteration. The full algorithm is in Appendix~\ref{app:algorithm}.

\paragraph{Datasets and backbones.} Spatio-temporal experiments use the Torch Spatiotemporal framework~\citep{cini2022tsl} on six benchmarks (PeMS04/07/08 flow, METR-LA and PeMS-BAY speed, AirQuality). Primary backbone is DiffConv ($K{=}2$ bidirectional diffusion, the standard configuration); ChebConv and MPGRU are decoupled secondary backbones, DCRNN~\citep{li2018diffusion} is the coupled backbone, and AGCRN and GraphWaveNet are adaptive references (Section~\ref{sec:framework-scope}). Node classification uses the GSLB fork~\citep{li2023gslb} with LDS on Cora and Citeseer (homophilous, the regime where LDS was validated); the GCN inner backbone uses GSLB's per-dataset tuned configuration (hidden $128$ on Cora, $64$ on Citeseer).

\paragraph{Protocol.} ST runs report 5 seeds, NC runs 10 seeds, with $T = 10$ inner steps for ST (100 epochs) and $T = 5$ for NC (GSLB default). We report mean $\pm$ std with paired $t$-tests for channel decompositions. Full setup, seed lists, and statistical reporting conventions are in Appendices~\ref{app:datasets},~\ref{app:hyperparam},~\ref{app:reproducibility}.

%% file: sections/05_findings.tex
\section{Findings}
\label{sec:findings}

\subsection{Does spatial rewiring improve spatio-temporal forecasting?}
\label{sec:find-rewiring}

Across six datasets with the DiffConv backbone (full results in Appendix Table~\ref{tab:survey}), bilevel achieves $3.8$--$6.4\%$ MAE improvement on flow datasets (PeMS04/07/08) with non-overlapping seed distributions, while static rewiring methods (FoSR, BORF, SDRF, GTR) are within noise of vanilla on all six datasets, consistent with prior work on static rewiring~\citep{tori2025effectiveness}. On speed datasets, the pattern diverges: PeMS-BAY shows a modest $-0.7\%$ while METR-LA shows $+2.5\%$ (bilevel actively harms performance). We return to this split in Section~\ref{sec:find-scope}.

\subsection{Training dynamics dominate bilevel gains on flow forecasting}
\label{sec:find-decomp}

We apply the frozen-$\phi$ decomposition to the spatio-temporal regime with DiffConv across three flow-signal datasets and to the node classification regime with LDS~\citep{franceschi2019learning} on Cora and Citeseer. Figure~\ref{fig:hero} shows the schematic and channel breakdown. Table~\ref{tab:decomp} reports the three-way comparison and resulting channel shares. The principal finding is that the inner channel accounts for $78$--$101\%$ of the bilevel gain on flow-signal datasets where bilevel improves over vanilla, and $44\%$ on Citeseer under a distinct Bernoulli edge-level parameterization.

\begin{figure}[h!]
\centering
\begin{subfigure}[c]{0.5\textwidth}
  \includegraphics[width=\textwidth]{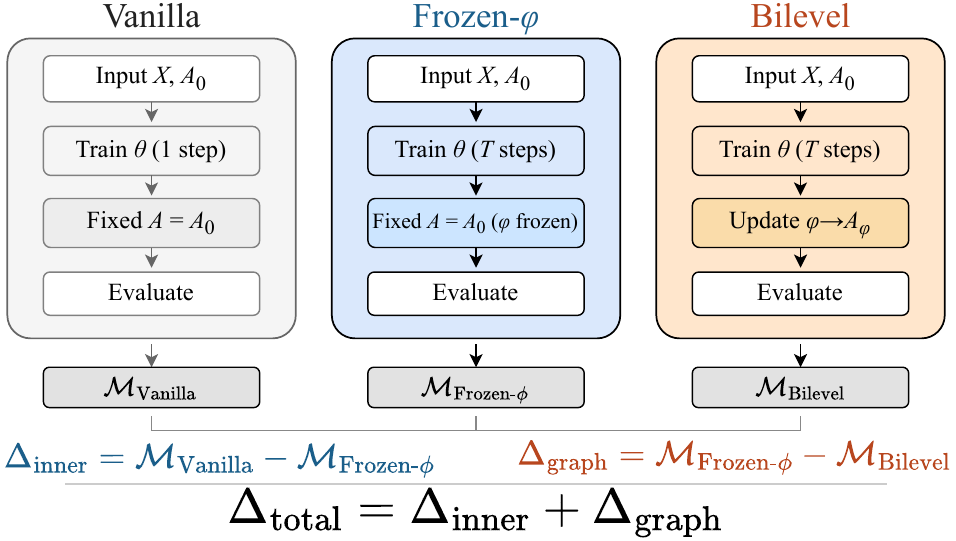}
\end{subfigure}
\hfill
\begin{subfigure}[c]{0.48\textwidth}
  \includegraphics[width=\textwidth]{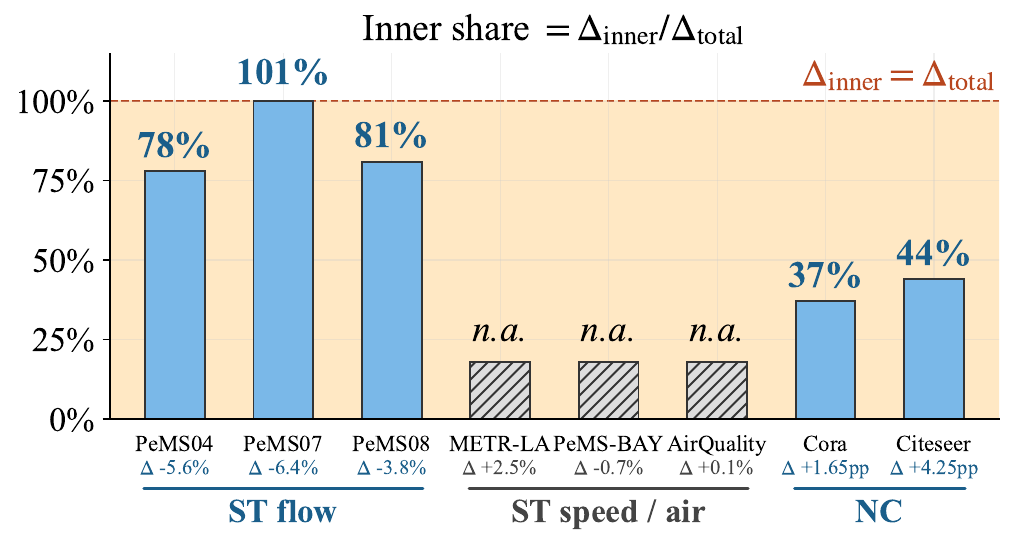}

\end{subfigure}
\caption{\textbf{Frozen-$\phi$ isolates training dynamics from graph modification.} (Left)~The control disables the outer loop while retaining the $T$-step inner loop, decomposing the bilevel gain as $\Delta_{\mathrm{total}} = \Delta_{\mathrm{inner}} + \Delta_{\mathrm{graph}}$, where $\Delta_{\mathrm{inner}} = \mathcal{M}_\mathrm{Vanilla} - \mathcal{M}_\mathrm{Frozen\text{-}\phi}$ measures training dynamics with the graph held constant and $\Delta_{\mathrm{graph}} = \mathcal{M}_\mathrm{Frozen\text{-}\phi} - \mathcal{M}_\mathrm{Bilevel}$ measures graph modification with the inner schedule held constant. (Right)~Flow forecasting: inner channel accounts for 78--101\%; LDS on Cora/Citeseer: more balanced (37--44\%). Hatched: decomposition n/a (Table~\ref{tab:decomp}). $\Delta$: relative MAE (ST); pp (NC).}
\label{fig:hero}
\end{figure}

\begin{table}[h!]
\centering
\caption{Main decomposition. ST: test MAE (DiffConv). NC: accuracy (LDS+GCN). Inner share $= \Delta_{\mathrm{inner}}/\Delta_{\mathrm{total}}$ per~\eqref{eq:decomp}; n/a denotes datasets where $\Delta_{\mathrm{total}} \leq 0$ (METR-LA, AirQuality) or where the inner and graph channels have opposite signs (PeMS-BAY); see Section~\ref{sec:find-scope}. PeMS07 frozen-$\phi$ vs.\ bilevel within seed noise ($101\%$). Cora bootstrap CI wide (small total gain, App.~\ref{app:nc_full}).}
\label{tab:decomp}
\small
\begin{tabular}{llcccccc}
\toprule
Regime & Dataset & Vanilla & Frozen-$\phi$ & Bilevel & Inner & Graph & $p_{\mathrm{total}}$ \\
\midrule
\multirow{6}{*}{ST}
 & PeMS04    & $20.925_{\pm.043}$ & $20.011_{\pm.034}$ & $19.754_{\pm.036}$ & 78\% & 22\% & $<0.001$ \\
 & PeMS07    & $22.277_{\pm.038}$ & $20.839_{\pm.021}$ & $20.848_{\pm.035}$ & 101\% & $-1\%$  & $<0.001$ \\
 & PeMS08    & $15.621_{\pm.057}$ & $15.146_{\pm.079}$ & $15.031_{\pm.071}$ & 81\% & 19\% & $<0.001$ \\
 & METR-LA   & $3.240_{\pm.006}$  & $3.292_{\pm.010}$  & $3.323_{\pm.016}$  & n/a  & n/a & n/a \\
 & PeMS-BAY  & $1.665_{\pm.002}$  & $1.670_{\pm.006}$  & $1.654_{\pm.002}$  & n/a  & n/a & n/a \\
 & AirQuality & $27.631_{\pm.214}$ & $27.863_{\pm.302}$ & $27.652_{\pm.229}$ & n/a & n/a & n/a \\
\midrule
\multirow{2}{*}{NC}
 & Cora      & $81.10_{\pm.67}\%$ & $81.71_{\pm.55}\%$ & $82.75_{\pm.84}\%$ & 37\% & 63\% & $<0.001$ \\
 & Citeseer  & $69.28_{\pm.88}\%$ & $71.13_{\pm.54}\%$ & $73.53_{\pm.83}\%$ & 44\% & 56\% & $<0.0001$ \\
\bottomrule
\end{tabular}
\end{table}

\paragraph{The inner channel accounts for $78$--$101\%$ on flow datasets.} Freezing $\phi$ while retaining $T = 10$ accounts for 78\%, 101\%, and 81\% of the full bilevel improvement on PeMS04, PeMS07, and PeMS08 respectively. PeMS07 is extreme: frozen-$\phi$ and bilevel test MAEs differ by $0.009$, well within the per-seed standard deviation ($\sigma \approx 0.035$, Table~\ref{tab:decomp}); the graph-channel point estimate is statistically indistinguishable from zero. The seed distributions for frozen-$\phi$ and vanilla are clearly separated on PeMS07 and PeMS08 and nearly separated on PeMS04 (Appendix~\ref{app:seed_distributions}). On METR-LA and AirQuality bilevel does not improve over vanilla ($\Delta_{\mathrm{total}} \leq 0$). On PeMS-BAY the bilevel improvement over vanilla is small ($-0.011$ MAE) and the inner and graph channels carry opposite signs (frozen-$\phi$ is slightly worse than vanilla, while bilevel recovers); we therefore mark its decomposition as n/a in Table~\ref{tab:decomp} and examine all three scope-boundary cases in Section~\ref{sec:find-scope}.

We additionally evaluate the frozen-$\phi$ decomposition on node classification with LDS~\citep{franceschi2019learning} on Cora and Citeseer. The choice of LDS is structural: frozen-$\phi$ requires gradient-based bilevel optimization, incremental modification of an initial adjacency, and positive empirical gain on the clean graph (Section~\ref{sec:framework-scope}). Among published GSL methods for node classification, only LDS satisfies all three; GEN~\citep{wang2021graph} uses EM-based replacement of the adjacency rather than incremental modification, Pro-GNN~\citep{jin2020graph} is designed for adversarial defense and produces negative clean-graph gain ($\Delta_{\text{total}} = -3.66$ pp on Cora), and IDGL~\citep{chen2020iterative} is single-level. The graph-distillation complement (Appendix~\ref{app:distillation}) extends the analysis to methods outside this scope. We treat NC as a parameterization-robustness check across a distinct training regime (full-batch reset vs.\ mini-batch reuse) and ST as the primary regime.

\subsection{Two training regimes produce distinct mechanistic signatures}
\label{sec:find-mechanism}

Varying $T$ provides mechanistic evidence for Proposition~\ref{prop:igr}: the mini-batch reuse regime predicts the frozen-$\phi$ loss depends on $T$; the full-batch reset regime predicts it is invariant past inner convergence. We sweep $T \in \{1, 3, 5, 10, 20\}$ in both regimes (Figure~\ref{fig:tsweep}; numerical values in Appendix~\ref{app:hyperparam}).

\paragraph{Mini-batch reuse drives $T$-dependence; full-batch reset yields $T$-invariance.} On PeMS04, three predictions of Proposition~\ref{prop:igr} hold: at $T = 1$ the frozen-$\phi$ MAE matches vanilla within $0.001$, well below the per-seed noise floor of $\sigma \approx 0.04$ (Table~\ref{tab:hp_sens}); frozen-$\phi$ MAE decreases monotonically with $T$ through $T = 10$ as $\RIGR \propto T$ grows; full bilevel tracks frozen-$\phi$ at every $T$ offset by a small approximately $T$-independent graph-channel contribution. On Cora and Citeseer the frozen-$\phi$ accuracy is instead nearly flat across $T \in \{3, 5, 10, 20\}$ while full bilevel maintains higher accuracy than frozen-$\phi$ at every $T$ (clearly improving with $T$ on Citeseer; within seed noise on Cora), the predicted full-batch reset signature (per-outer-iteration parameter reset collapses the inner trajectory to a $T$-invariant function past convergence).

\begin{figure}[h!]
\centering
\includegraphics[width=1.0\textwidth]{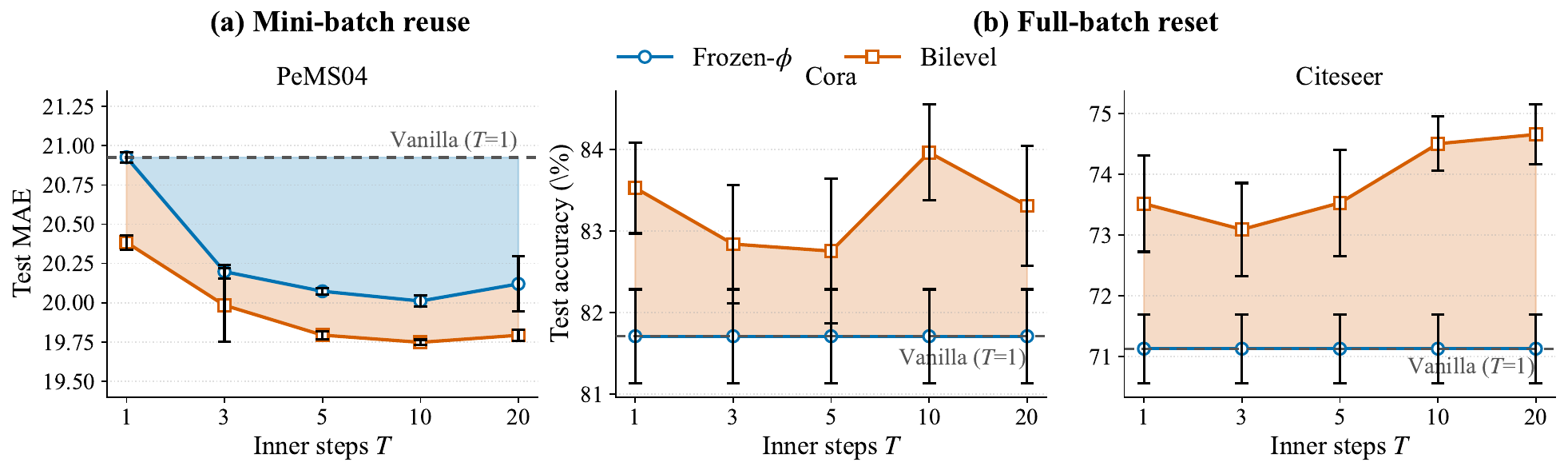}
\caption{\textbf{Two training regimes produce distinct mechanistic signatures.} (a) On PeMS04, frozen-$\phi$ MAE decreases monotonically with $T$ as $\RIGR \propto T$ accumulates within reused mini-batches; at $T = 1$ frozen-$\phi$ collapses to vanilla. (b) On Cora and Citeseer, frozen-$\phi$ accuracy is $T$-invariant past inner convergence: full-batch gradients with per-outer-iteration parameter reset render the inner trajectory $T$-independent. Both signatures match the predictions of Proposition~\ref{prop:igr}.}
\label{fig:tsweep}
\end{figure}

\subsection{Triangulation and scope across regimes}
\label{sec:find-scope}

Three independent diagnostics, edge corruption on Cora, alternative training schedules on PeMS04, and convergent triangulation on PeMS07, corroborate the inner-channel attribution and delineate where the graph channel becomes operative. The unified six-dataset interpretation through three preconditions on the predefined adjacency is developed in Section~\ref{sec:discussion} and Appendix~\ref{app:framework}.

\paragraph{Graph quality determines channel balance.} We hold the node classification parameterization and training regime fixed and vary the quality of the initial graph (precondition~(i) above): starting from Cora, we randomly rewire a fraction $r \in \{0, 0.10, 0.25, 0.50, 0.75\}$ of edges, preserving total edge count, and run the three-way decomposition at each level.

\begin{figure}[h!]
\centering
\includegraphics[width=0.94\textwidth]{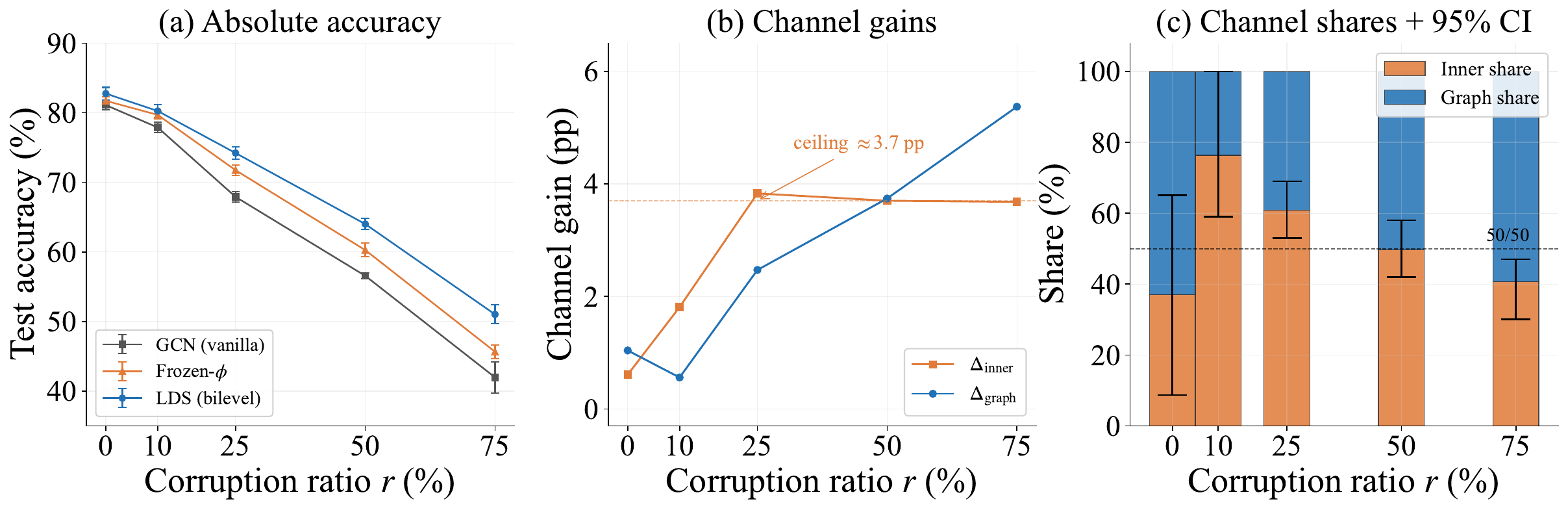}
\caption{Three-way decomposition under edge corruption on Cora. Inner channel saturates near 3.7 pp for $r \geq 0.25$; graph channel grows approximately linearly with $r$ once corruption exceeds the $r = 0$ noise floor. 50/50 crossover near $r = 0.50$.}
\label{fig:corruption}
\end{figure}

The inner channel grows from 0.61 pp at $r = 0$ to 1.81 pp at $r = 0.10$ and plateaus near 3.7 pp across $r \in \{0.25, 0.50, 0.75\}$ (CV $=$ 2.2\% across the plateau). The graph channel grows approximately linearly from 0.56 pp at $r = 0.10$ to 5.37 pp at $r = 0.75$ (Pearson 0.98 with $r$ over $r \geq 0.10$). The channel share crosses 50/50 near $r = 0.50$ and ends graph-dominated at $r = 0.75$; all channel tests reach $p < 0.001$ at every $r \geq 0.10$. On clean Cora the inner channel is marginal ($p = 0.069$) and the graph channel reaches $p < 0.005$. Graph quality, operationalized through controlled corruption, directly determines the channel balance within a single regime: the inner channel saturates at an architecture-dependent ceiling while the graph channel scales with structural slack. The same logic applies to ST: distance-based adjacencies are not optimized for the forecasting task and exhibit large task-initial mismatch, similar to the corrupted Cora regime.

\paragraph{Weight decay and end-to-end joint training do not replicate the inner channel.} Explicit $\ell_2$ regularization on PeMS04 with $\lambda \in \{10^{-4}, 10^{-3}, 10^{-2}\}$ does not close the gap (Appendix~\ref{app:wd_sweep}, Table~\ref{tab:wd}): even the best setting ($\lambda = 10^{-4}$) is marginally worse than vanilla, while frozen-$\phi$ at $T = 10$ achieves an approximately $0.91$ MAE improvement over vanilla. The IGR mechanism of Proposition~\ref{prop:igr} involves gradient norms and stochastic variance rather than parameter norms. A $10\times$-compute vanilla baseline ($1000$ epochs) closes part of the inner-channel gap but leaves the bilevel residual unaffected (Appendix~\ref{app:compute_matched}), supporting the trajectory-length interpretation of IGR. An end-to-end baseline that jointly optimizes $W_\phi$ and $\theta$ on the training loss without a separate outer loop also fails: on PeMS07 bilevel achieves $-1.43$ MAE while E2E achieves $-0.07$, a twenty-fold gap that cannot be attributed to graph modification (which contributes a near-zero share on PeMS07), leaving the $T$-step inner schedule itself as the operative mechanism.

\paragraph{PeMS07: three diagnostics converge on pure training dynamics.} We observe a suggestive anti-correlation between modification magnitude and improvement across the six datasets ($\rho = -0.77$, $p = 0.072$, $N = 6$, marginal at conventional thresholds): PeMS07 shows the largest gain ($\Delta$ MAE $= -6.4\%$) yet the smallest modifications (99\% of edges have $|\Delta w| < 0.1$). Three independent diagnostics on PeMS07 all yield small graph-channel point estimates: frozen-$\phi$ accounts for 101\% of the bilevel gain, graph distillation recovers only 1\%, and end-to-end reweighting is $20\times$ weaker than bilevel (Figure~\ref{fig:triangulation_pems07} in Appendix~\ref{app:triangulation}). Bilevel's success regime on flow forecasting is selective minimal intervention, not aggressive structural rewiring (Appendix~\ref{app:mod_magnitude}).

\paragraph{The inner channel itself fails on speed signals.} Bilevel is neutral on DCRNN and harmful on METR-LA (Appendix~\ref{app:backbone_scope}); frozen-$\phi$ on METR-LA is also worse than vanilla ($\Delta = +0.052$ MAE, $\approx 5\sigma_F$) and on PeMS-BAY marginally worse ($\Delta = +0.005$, within $1\sigma_F$), so the inner channel itself fails on speed signals rather than being canceled by a harmful outer update. Per-distance Jacobian measurements (Appendix~\ref{app:jacobian_cross}) show bilevel relaxes the short-to-long sensitivity ratio on five of six datasets, with METR-LA the exception (vanilla $95\times$ vs.\ bilevel $335\times$), consistent with bilevel actively concentrating propagation in a regime where it harms forecasting.

%% file: sections/06_discussion.tex
\section{Discussion}
\label{sec:discussion}

\paragraph{When does bilevel rewiring help?} The six-dataset gain pattern admits a unified explanation through three preconditions on the predefined adjacency: \textbf{(i) structural slack} ($A_{\text{init}}$ is suboptimal for the task-relevant structure), \textbf{(ii) loss identifiability} (the validation loss has informative gradients in $\phi$ near $A_{\text{init}}$), and \textbf{(iii) parametrization expressiveness} (the policy class can represent the true dependency); formalized in Appendix~\ref{app:framework_def}, per-dataset mapping in Table~\ref{tab:preconditions}. Flow datasets satisfy all three; speed datasets fail (i)--(ii) because kinematic-wave physics~\citep{lighthill1955kinematic} aligns the distance graph with the dynamics and per-node embeddings absorb the residual variance; AirQuality fails (iii) because the standard distance-based adjacency is block-diagonal with zero inter-city edges (Appendix~\ref{app:airquality_anatomy}). The framework correctly predicts the sign of the bilevel gain on all six benchmarks, including the three on which bilevel does not improve over the clean-graph baseline; Appendix~\ref{app:decision_tree} summarizes the resulting practitioner decision flow.

\paragraph{Methodological recommendation.} We propose the frozen-$\phi$ ablation as a standard control in graph structure learning. Concretely, any paper reporting gains from a learned graph should report:
\begin{enumerate}
[label=(\roman*),leftmargin=2em,itemsep=0pt,topsep=2pt]
\item the frozen-$\phi$ test metric under the same $T$-step inner schedule as the full bilevel method,
\item the inner share $\Delta_{\mathrm{inner}}/\Delta_{\mathrm{total}}$ from~\eqref{eq:decomp},
\item a $T$-sweep checking whether the frozen-$\phi$ result is $T$-dependent or $T$-invariant.
\end{enumerate}
Absent these controls, reported GSL gains are confounded with training dynamics introduced by the inner optimization schedule. For methods outside the scope of the frozen-$\phi$ control, graph distillation provides a method-agnostic complement (Appendix~\ref{app:distillation}).

\paragraph{Relation to bilevel optimization beyond GSL.} Implicit regularization in bilevel optimization has been studied in meta-learning~\citep{vicol2022implicit,rajeswaran2019meta}. Our contribution is not to discover this phenomenon but to (i) show that it explains a substantial share of gains attributed to graph modification in bilevel GSL, where this attribution has not been previously questioned, and (ii) provide a diagnostic that GSL practitioners can run on their own methods. The diagnostic principle, isolating outer-parameter contribution by freezing the outer parameter while preserving inner-loop training, potentially applies to bilevel optimization beyond GSL, though we leave such extensions to future work.

\paragraph{Spectral metrics dissociate from task gain.} A natural alternative explanation for the divergence between bilevel and static rewiring is spectral: that bilevel optimizes a graph-spectral objective (larger spectral gap $\lambda_2$ or narrower effective spectral support) that static methods do not match. We tested this directly on the LCC-restricted subgraph that every method shares (Appendix~\ref{app:spectral_negative}). The data shows a clean dissociation. GTR is a strong spectral-gap maximizer: $\lambda_2$ increases by $2.6$--$15.5\times$ on every benchmark with a numerically resolved baseline (Table~\ref{tab:lambda2_main}), yet GTR matches vanilla on test MAE. Bilevel decreases $\lambda_2$ on four of six datasets and produces the largest forecasting gains on flow (between $-3.8\%$ and $-6.4\%$) without raising $\lambda_2$. Spectral gap optimization is therefore decoupled from task improvement on these benchmarks.

\begin{table}[h!]
\centering
\caption{Spectral gap $\lambda_2$ ($\times 10^3$) on the LCC subgraph (identical node set across methods). GTR maximizes $\lambda_2$ by $2.6$--$15.5\times$ on every benchmark with a numerically resolved baseline (PeMS07 $\lambda_2$ is below display precision so the GTR multiplier is ill-defined) yet matches vanilla MAE; bilevel decreases $\lambda_2$ on four of six datasets while producing the largest gains on flow. Bilevel reports mean $\pm$ std across seeds. Full table including FoSR/BORF/SDRF and $W_\epsilon$ measurements in Appendix~\ref{app:spectral_negative}.}
\label{tab:lambda2_main}
\small
\begin{tabular}{lccccc}
\toprule
Dataset & Original & GTR & Bilevel & $\Delta \lambda_2$ bilevel & $\Delta$ MAE bilevel \\
\midrule
PeMS04   & $6.0$ & $55.0$ & $3.0_{\pm 2.0}$ & $\downarrow$ & $-5.6\%$ \\
PeMS07   & $0.0$ & $8.0$  & $0.0_{\pm 0.0}$ & $-$         & $-6.4\%$ \\
PeMS08   & $2.0$ & $31.0$ & $1.0_{\pm 1.0}$ & $\downarrow$ & $-3.8\%$ \\
METR-LA  & $8.0$ & $21.0$ & $5.0_{\pm 1.0}$ & $\downarrow$ & $+2.5\%$ \\
PeMS-BAY & $5.0$ & $13.0$ & $1.0_{\pm 1.0}$ & $\downarrow$ & $-0.7\%$ \\
AirQuality & $1.0$ & $11.0$ & $2.0_{\pm 0.0}$ & $\uparrow$  & $+0.1\%$ \\
\bottomrule
\end{tabular}
\end{table}

\paragraph{Open questions.} Several questions remain open. We highlight two: first, whether the three-precondition framework predicts the sign of the graph-channel contribution on heterophilous benchmarks (e.g., Squirrel, Chameleon, Actor); no validated bilevel GSL configuration currently exists in this regime, and applying LDS to heterophilous benchmarks without tuning would risk reproducing the failure mode of~\citet{tori2025effectiveness}. Second, whether discrete-edge bilevel methods designed for ST forecasting~\citep{shang2021discrete} fall under the mini-batch reuse regime (predicting $T$-dependent inner-channel gains as in our ST setting) or the full-batch reset regime (predicting $T$-invariant frozen-$\phi$ as in NC); the answer determines whether their reported gains derive from the same training-dynamics mechanism we identify here.

\paragraph{Limitations.} Our findings do not claim the over-squashing bounds of~\citet{topping2022understanding,digiovanni2023oversquashing,marisca2025oversquashing} are incorrect; they bound a different quantity than the practical constraint we identify in spatio-temporal forecasting. Consistent with the diagnostic framing, we identify the confound and provide tools for measuring it; the path to mitigation is left to future work. Our decomposition measures contribution, not necessity, of each channel. NC evaluation is limited to Cora and Citeseer. Bilevel incurs $\sim 15\times$ training overhead on PeMS04 DiffConv (Appendix~\ref{app:efficiency}); inference cost is unchanged.

%% file: sections/07_conclusion.tex
\section{Conclusion}
\label{sec:conclusion}

The frozen-$\phi$ control isolates the contribution of graph modification within bilevel GSL by freezing the graph while retaining the inner-loop training schedule. On six benchmarks, the inner channel matches or exceeds the full bilevel pipeline on spatio-temporal flow forecasting and accounts for a substantial share of the gain on node classification, indicating that inner-loop training dynamics contribute substantially more to bilevel GSL gains than is typically credited. We propose frozen-$\phi$ as a standard ablation, with graph distillation as a method-agnostic complement; the diagnostic principle of freezing the outer parameter while preserving inner-loop training may apply to bilevel optimization beyond GSL.

%% file: appendix_main.tex
\newpage
\appendix

\input{appendix_01_setup}
\input{appendix_02_inner_channel}
\input{appendix_03_graph_channel}
\input{appendix_10_spectral_negative}

\input{appendix_04_architecture}
\input{appendix_05_triangulation}

\input{appendix_06_node_classification}
\input{appendix_07_per_horizon}

\input{appendix_08_R}
\input{appendix_09_efficiency_repro}

%% file: appendix_01_setup.tex
\section{Algorithm}
\label{app:algorithm}

\begin{algorithm}[h!]
\caption{First-Order Bilevel Rewiring for STGNNs}
\label{alg:bilevel}
\begin{algorithmic}[1]
\REQUIRE Spatial graph $A$, STGNN $f_\theta$, policy $\pi_\phi$, warmup epochs $E_w$, total epochs $N_{\mathrm{ep}}$, inner steps $T$
\STATE Initialize $\theta$, $\phi$
\FOR{epoch $= 1, \ldots, N_{\mathrm{ep}}$}
\FOR{batch $(X, Y)$ in training set}
\STATE $A_\phi \leftarrow \pi_\phi(A)$ \COMMENT{Reweight spatial graph}
\FOR{$t = 1, \ldots, T$}
\STATE $\theta \leftarrow \theta - \eta_\theta \nabla_\theta \Ltrain(f_\theta(X, A_\phi), Y)$ \COMMENT{Inner step on same batch}
\ENDFOR
\IF{epoch $> E_w$}
\STATE Sample validation batch $(X_v, Y_v)$
\STATE Compute $\nabla_\phi \Lval(f_\theta(X_v, A_\phi), Y_v)$ \COMMENT{$\xi=0$; $\theta$ treated as constant w.r.t.\ $\phi$}
\STATE $\phi \leftarrow \phi - \eta_\phi \nabla_\phi \Lval$ \COMMENT{Outer step}
\ENDIF
\ENDFOR
\ENDFOR
\RETURN Trained $\theta$, learned $A_\phi$
\end{algorithmic}
\end{algorithm}

The first-order approximation at the outer step corresponds to $\xi=0$ in the DARTS notation~\citep{liu2019darts}: $A_\phi$ is detached from the computation graph during the $T$ inner $\theta$ updates (so $\theta$ is not tracked as a function of $\phi$), and the outer gradient is computed via a fresh forward pass at the post-inner $\theta$, treating $\theta$ as constant. This avoids the Hessian-vector product required by second-order bilevel differentiation.

The frozen-$\phi$ control uses the identical procedure but skips the outer-step block, so $\phi$ is never updated and the graph remains at $\Ainit$ throughout training. This preserves the $T$-step inner optimization schedule while isolating the effect of graph modification. All $T$ inner steps operate on the same mini-batch in the spatio-temporal implementation: the training batch is sampled once per outer iteration and reused across inner steps, so $T = 10$ does not expose the model to additional data relative to $T = 1$. For LDS on node classification, the inner loop uses full-batch gradients and the GCN inner parameters are reinitialized at each outer iteration (full-batch reset regime in Section~\ref{sec:framework-igr}). The frozen-$\phi$ initialization for LDS sets $\phi \leftarrow A$ following the original LDS protocol~\citep{franceschi2019learning}: the Bernoulli edge probabilities are held at the original adjacency, and only the GCN parameters are optimized.

\section{Proofs}
\label{app:proofs}

\subsection{Proof of Proposition~\ref{prop:igr}}

\begin{assumption}[Regularity conditions]
\label{ass:regularity}
(A1) $\Ltrain(\cdot; \phi)$ is $\beta$-smooth: $\|\nabla^2_\theta \Ltrain\| \leq \beta$ for all $\theta$ on the trajectory. (A2) The learning rate satisfies $\eta < 1/\beta$. (A3) Within $T$ inner steps, $\|\theta_T - \theta_0\| = O(T\eta)$ remains within a region where $\Ltrain$ is $C^3$-smooth with bounded third derivatives, ensuring validity of the $O(\eta^2)$ truncation in the backward error expansion (as invoked by~\citet{barrett2021implicit} and~\citet{smith2021origin}). (A4) Per-step stochastic gradient statistics $\|g_t\|^2$ and $\mathrm{tr}(\Sigma_t)/B$ are bounded.
\end{assumption}

\begin{proof}
\citet{barrett2021implicit} establish via backward error analysis that a single gradient descent step with step size $\eta$ on $\Ltrain$ maps to the exact flow of a modified loss $\tilde{\Ltrain} = \Ltrain + (\eta/4)\|\nabla \Ltrain\|^2 + O(\eta^2)$. \citet{smith2021origin} extend this to SGD and show (their Eq. 22) that the modified loss for stochastic gradient descent with per-batch gradients $\hat{g}$ and batch size $B$ is $\tilde{\Ltrain}_{\mathrm{SGD}} = \Ltrain + (\eta/4)(\|g\|^2 + \mathrm{tr}(\Sigma)/B) + O(\eta^2)$, where $g$ is the per-step mean of $\hat{g}$ and $\Sigma$ denotes the per-example gradient covariance (so $\hat{g}$ has covariance $\Sigma/B$). With $\phi$ held fixed within the inner loop, the per-step $\RIGR$ is a functional of the inner-loop gradient trajectory at the current $\phi$, and outer-loop updates to $\phi$ do not change its functional form: freezing $\phi$ does not alter the per-step implicit regularization mechanism.

For $T$ consecutive steps with fixed $\phi$, the trajectory tracks the continuous flow of $\tilde{\Ltrain}_{\mathrm{SGD}}$ for time $T\eta$. The per-step regularization coefficient $\eta/4$ does not change with $T$; rather, the trajectory traverses the modified loss landscape for proportionally longer time, and the path integral $\int_0^{T\eta} \RIGR(\theta(s), \eta, B)\, ds$ grows with $T$. Under Assumption~\ref{ass:regularity} (A4), this integral is bounded above and below by $T$-linear expressions in the gradient statistics.
\end{proof}

\begin{remark}[Mini-batch reuse amplification]
When all $T$ inner steps share the same mini-batch, the per-step gradient is deterministic and $\RIGR$ reduces to the GD form $(\eta/4)\|g_t\|^2$ at fixed coefficient. Across outer iterations the mini-batch changes and the variance term contributes through the SGD form. The trajectory under fixed-batch inner steps follows the modified flow of $\tilde{\Ltrain}_{\mathrm{GD}}$ for time $T\eta$, so $\theta_T$ drifts $T$-monotonically away from the gradient flow of $\Ltrain$, consistent with the mini-batch reuse regime in the main text.
\end{remark}

\subsection{Proof of Proposition~\ref{prop:jac}}

\begin{proposition}[Uniform spectral tightening of the spatial Jacobian factor]
\label{prop:jac}
Let $G' = (V, E', A')$ satisfy $\lambda_2(\tilde{L}') \geq \lambda_2(\tilde{L})$ and $\delta' \leq \delta$ (degree heterogeneity does not increase). Then the effective spectral mixing rate of the spatial MP matrix satisfies $\rho_{\mathrm{eff}}' \leq \rho_{\mathrm{eff}}$, and for any $K \geq 1$ the spatial factor admits a uniformly tighter spectral residual bound: $\sup_{u,v} |(S'^{K})_{uv} - S^{\infty}_{uv}| \leq C (\rho_{\mathrm{eff}}')^{K} \leq C \rho_{\mathrm{eff}}^{K}$, where $S^{\infty}$ denotes the stationary limit of $S^K$ on the LCC and $C$ is a constant depending on the spectral geometry. Consequently, the entrywise variation of $(S^{L_S L})_{uv}$ across node pairs is reduced under $G'$.
\end{proposition}

\paragraph{Spectral control of the spatial MP matrix.} Decompose $S = S_{\mathrm{reg}} + E$ with $S_{\mathrm{reg}} = \alpha I + c_2 \hat{A}$ symmetric with eigenpairs $(\mu_j, \varphi_j)$, so $S_{\mathrm{reg}}$ has eigenvalues $\bar{\sigma}_j = \alpha + c_2 \mu_j$. The perturbation $E = c_1 (\mathrm{diag}(\hat{A}^{\top} \mathbf{1}) - \bar{d} I)$ is diagonal with operator norm $\|E\|_2 = c_1 \max_j |(\hat{A}^{\top} \mathbf{1})_j - \bar{d}| =: \delta$. By Weyl's perturbation inequality, the eigenvalues of $S$ satisfy $|\sigma_j - \bar{\sigma}_j| \leq \delta$. Combining this with standard spectral convergence, $|(S^K)_{uv} - \pi_{uv}| \leq C \cdot \rho_{\mathrm{eff}}^K$ where $\rho_{\mathrm{eff}} = (\bar{\sigma}_2 + \delta)/(\bar{\sigma}_1 - \delta)$. Substituting $\bar{\sigma}_2 = (\alpha + c_2) - c_2 \lambda_2$ and $\bar{\sigma}_1 = \alpha + c_2$ yields the stated expression.

\paragraph{Main argument.} The ST bound factors as $B_T \cdot (R^{L_T L})_{i,0} \cdot B_S \cdot (S^{L_S L})_{uv}$. The temporal factor is unchanged by spatial rewiring. Under $\lambda_2' \geq \lambda_2$ and $\delta' \leq \delta$, the numerator $\bar{\sigma}_2 + \delta = (\alpha + c_2 - c_2 \lambda_2) + \delta$ of $\rho_{\mathrm{eff}}$ decreases and the denominator $\bar{\sigma}_1 - \delta = (\alpha + c_2) - \delta$ increases, so $\rho_{\mathrm{eff}}' \leq \rho_{\mathrm{eff}}$. The standard spectral residual bound $\sup_{u,v} |(S^K)_{uv} - S^{\infty}_{uv}| \leq C \rho_{\mathrm{eff}}^K$ then yields the uniformly tighter bound stated in the proposition for any $K \geq 1$. The spatial Jacobian factor $(S^{L_S L})_{uv}$ thus converges more uniformly to its stationary limit under $G'$, reducing the entrywise variation of $\|\nabla^u_i h^{v(L)}_t\|$ across node pairs.

\paragraph{Connection to over-squashing relief.} The proposition 
concerns the uniform spectral residual rather than pointwise entry 
values. For distant pairs $(u, v)$ with $(S^K)_{uv} < S^{\infty}_{uv}$ 
at finite $K$ (the typical over-squashing regime), faster mixing under 
$G'$ moves $(S'^K)_{uv}$ toward $S^{\infty}_{uv}$ from below, 
consistent with the over-squashing relief direction of 
\citet{topping2022understanding,digiovanni2023oversquashing,marisca2025oversquashing}.
The pointwise direction is therefore expected to flip between near 
and distant pairs, while the uniform residual decreases monotonically; 
our proposition formalizes the latter (which the proof establishes 
directly) while the former is empirically tested through the spectral 
dissociation analysis in Section~\ref{sec:discussion}.

\section{Dataset Statistics}
\label{app:datasets}

\begin{table}[h!]
\centering
\caption{Dataset statistics. ST traffic datasets use 5-minute intervals with 12-step input/output horizons; AirQuality uses 1-hour intervals (24/24). NC datasets use standard Planetoid splits \citep{yang2016planetoid}. Edge-level node homophily $h$ for NC datasets follows the definition of \cite{pei2020geomgcn}; both NC datasets are strongly homophilous, matching the operating regime of LDS \citep{franceschi2019learning}. Heterophilous evaluation is out of scope (Section~\ref{sec:framework-scope}).}
\label{tab:datasets}
\small
\begin{tabular}{llcccll}
\toprule
Regime & Dataset & Nodes & Edges & Mean deg. & Time steps / features & Homophily \\
\midrule
\multirow{6}{*}{ST}
 & PeMS04    & 307 & 268   & 1.7  & 16{,}992 & --- \\
 & PeMS07    & 883 & 395   & 0.9  & 28{,}224 & --- \\
 & PeMS08    & 170 & 183   & 2.2  & 17{,}856 & --- \\
 & METR-LA   & 207 & 757   & 7.3  & 34{,}272 & --- \\
 & PeMS-BAY  & 325 & 1{,}184 & 7.3  & 52{,}116 & --- \\
 & AirQuality & 437 & 2{,}699 & 12.4 & 8{,}760 & --- \\
\midrule
\multirow{2}{*}{NC}
 & Cora      & 2{,}708 & 5{,}429 & 4.0  & 1{,}433 feat./7 cls & $h{=}0.83$ \\
 & Citeseer  & 3{,}327 & 4{,}732 & 2.8  & 3{,}703 feat./6 cls & $h{=}0.71$ \\
\bottomrule
\end{tabular}
\end{table}

Traffic flow datasets use measurements from the Caltrans Performance Measurement System; adjacency is constructed via a thresholded Gaussian kernel on road-network distances. Traffic speed datasets use the same construction~\citep{li2018diffusion}. AirQuality uses PM2.5 monitoring stations with distance-based adjacency \citep[threshold 0.1]{zheng2015airquality}. Cora and Citeseer use the standard Planetoid splits as used by LDS and GSLB. Data splits are 70/10/20 for train/validation/test on all ST datasets. Bilevel optimization uses the validation split for outer-loop graph optimization; all method comparisons use held-out test MAE/accuracy exclusively.

%% file: appendix_02_inner_channel.tex
\section{Hyperparameter Sensitivity (T-sweep)}
\label{app:hyperparam}

We sweep inner steps $T$ on PeMS04 with the DiffConv backbone, holding all other hyperparameters at the default in Table~\ref{tab:hp_st}. Both the frozen-$\phi$ control and full bilevel show monotonic improvement with $T$ through $T = 10$, with frozen-$\phi$ alone capturing most of the gain at every $T$. Table~\ref{tab:hp_sens} reports the per-$T$ values used in the $T$-sweep (Section~\ref{sec:find-mechanism}); first-order bilevel is robust across $T \in [3, 10]$. Figure~\ref{fig:tsweep} in the main text plots the corresponding curves.

\begin{table}[h!]
\centering
\caption{$T$-sweep on PeMS04 (DiffConv). Vanilla reference: $20.925 \pm 0.043$ MAE. Default $T = 10$ marked with $\dagger$.}
\label{tab:hp_sens}
\small
\begin{tabular}{lccc}
\toprule
Inner steps $T$ & Frozen-$\phi$ MAE & Bilevel MAE & $\Delta$ vs vanilla (Bilevel) \\
\midrule
$T = 1$           & $20.926_{\pm.033}$ & $20.384_{\pm.046}$ & $-0.541$ \\
$T = 3$           & $20.198_{\pm.043}$ & $19.984_{\pm.234}$ & $-0.941$ \\
$T = 5$           & $20.073_{\pm.022}$ & $19.794_{\pm.026}$ & $-1.131$ \\
$T = 10^{\dagger}$ & $20.011_{\pm.034}$ & $19.754_{\pm.036}$ & $-1.171$ \\
$T = 20$          & $20.120_{\pm.175}$ & $19.793_{\pm.037}$ & $-1.132$ \\
\bottomrule
\end{tabular}
\end{table}

At $T = 1$ frozen-$\phi$ collapses to vanilla within noise (negative control), confirming the mini-batch reuse regime of Proposition~\ref{prop:igr}: when the inner loop reduces to a single step, the IGR mechanism does not accumulate. Across $T \in \{3, 5, 10, 20\}$ frozen-$\phi$ tracks bilevel within $\sim 0.3$ MAE of the bilevel value, with a small approximately $T$-independent gap explained by the graph channel.

\section{Inner-Step Implementation and Fair Comparison}
\label{app:inner_step}

A central implementation detail of the ST bilevel procedure is that all $T$ inner steps within a single outer iteration operate on the same mini-batch: the training batch is sampled once per outer iteration and reused across the $T$ gradient steps before the next batch is drawn. This appendix documents the resulting computational budget and provides empirical evidence for the accumulated parameter displacement mechanism that underlies the mini-batch reuse regime in Proposition~\ref{prop:igr}.

\subsection{Computational Budget}
\label{app:inner_step_budget}

Table~\ref{tab:fair_comparison} clarifies that frozen-$\phi$ ($T = 10$) and vanilla ($T = 1$) share the same epoch count, learning-rate schedule, and total training data exposure. The cosine annealing scheduler steps once per epoch, not per inner step, so both settings traverse an identical learning-rate trajectory. The only difference is per-batch optimization depth: vanilla performs one gradient step per batch, frozen-$\phi$ performs ten gradient steps per batch. Both settings see exactly $100|\mathcal{D}|$ samples over the full training run, ruling out additional data exposure as an explanation for the inner channel.

\begin{table}[h!]
\centering
\caption{Computational budget comparison. Vanilla ($T = 1$) and frozen-$\phi$ ($T = 10$) share the same epoch count, learning-rate schedule, and total training data; the only difference is per-batch optimization depth. This rules out additional data exposure as an explanation for the inner channel.}
\label{tab:fair_comparison}
\small
\begin{tabular}{lccccc}
\toprule
Setting & Epochs & Batches$/$epoch & Steps$/$batch & LR schedule & Total data seen \\
\midrule
Vanilla ($T = 1$)        & 100 & $|\mathcal{D}|$ & 1  & Cosine, $T_{\max}{=}100$ & $100|\mathcal{D}|$ \\
Frozen-$\phi$ ($T = 10$) & 100 & $|\mathcal{D}|$ & 10 & Cosine, $T_{\max}{=}100$ & $100|\mathcal{D}|$ \\
\bottomrule
\end{tabular}
\end{table}

\subsection{Gradient Norm Analysis}
\label{app:inner_step_grad_norm}

To characterize the inner-step mechanism we instrument the inner loop and record the pre-clipping gradient norm at each step $t \in \{0, \ldots, 9\}$ across the first three epochs on PeMS04 (DiffConv, seed$=$42). Figure~\ref{fig:grad_norm} shows the within-batch profile averaged over batches in epoch 0, together with per-batch training loss trajectories for five representative batches. Figure~\ref{fig:grad_norm_epoch} shows the same gradient-norm profile stratified by epoch.

\begin{figure}[h!]
\centering
\includegraphics[width=0.85\linewidth]{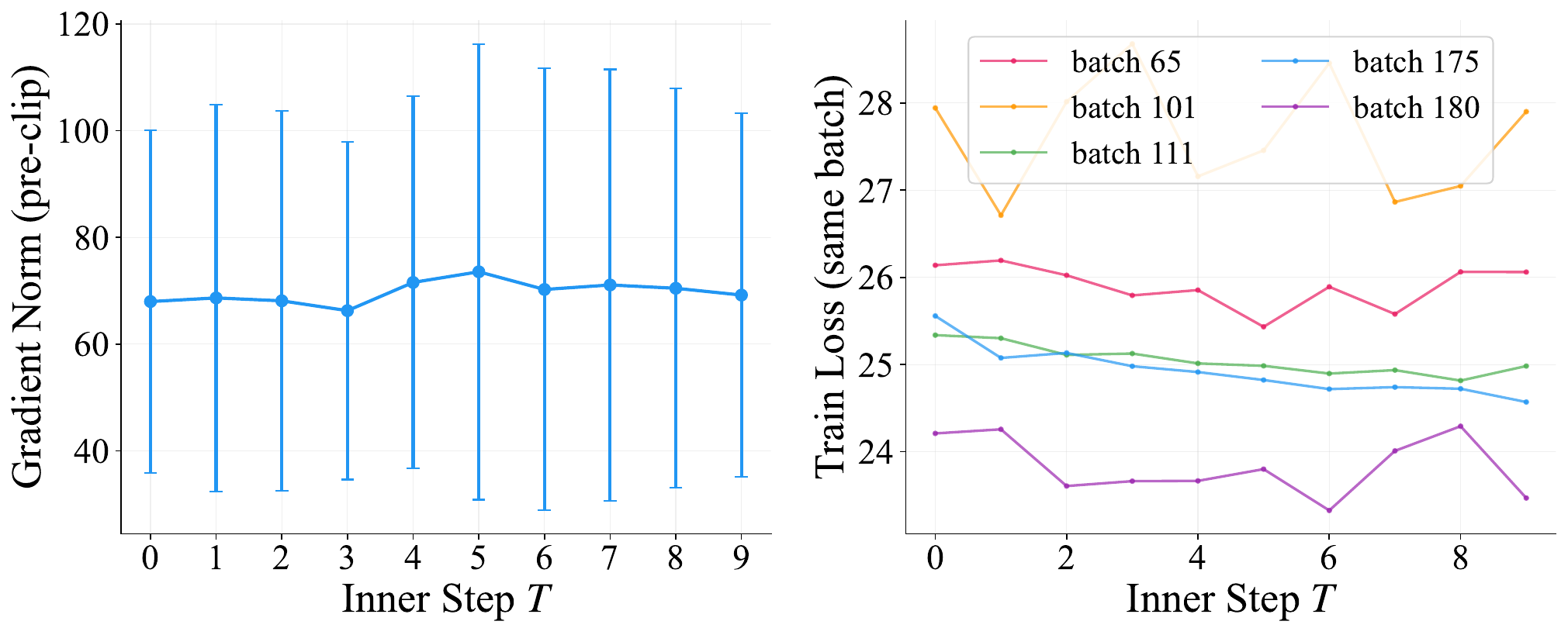}
\caption{PeMS04 (DiffConv, seed$=$42), epoch 0. Left: pre-clipping gradient norm during the inner loop, averaged across batches. The norm is approximately constant across all 10 inner steps, indicating that each step provides a parameter update of comparable magnitude rather than driving per-batch convergence. Right: per-batch training loss over inner steps for five representative batches; loss trajectories are approximately flat or slowly decreasing, confirming that $T = 10$ does not produce intra-batch convergence.}
\label{fig:grad_norm}
\end{figure}

\begin{figure}[h!]
\centering
\includegraphics[width=0.65\linewidth]{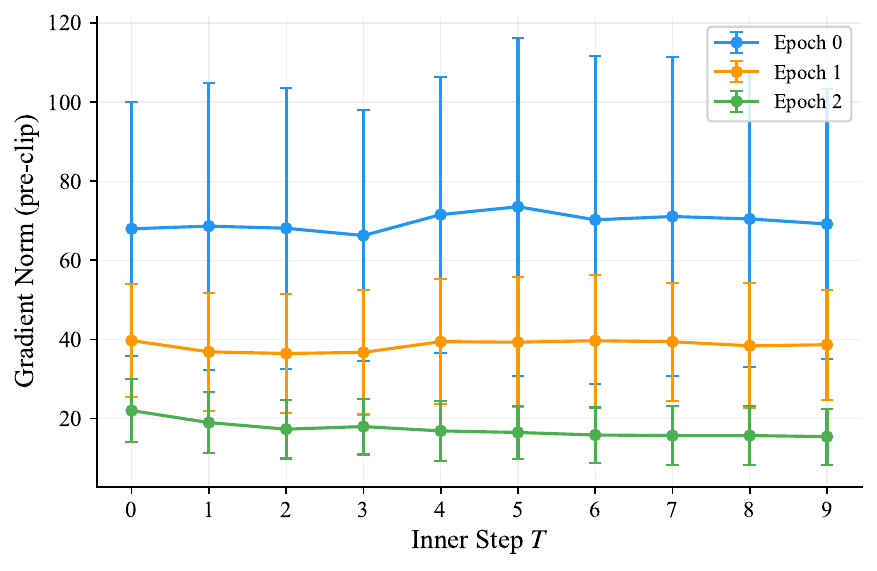}
\caption{Pre-clipping gradient norm across inner steps stratified by epoch (PeMS04, DiffConv, seed$=$42). The overall norm decreases substantially across epochs (epoch 0: $\approx 68$; epoch 1: $\approx 37$; epoch 2: $\approx 22$), but the within-epoch profile remains flat at every epoch. Gradient-norm reduction is driven by inter-epoch learning, not by intra-batch convergence.}
\label{fig:grad_norm_epoch}
\end{figure}

Two observations follow. First, gradient norms are approximately constant across inner steps within each epoch (Figure~\ref{fig:grad_norm}, left), ruling out the hypothesis that $T = 10$ drives per-batch convergence: if the model were converging on each batch, the within-batch gradient norms would decrease monotonically from step 0 to step 9, but they do not. Each inner step instead provides a parameter update of comparable magnitude. Second, gradient norms decrease substantially across epochs (Figure~\ref{fig:grad_norm_epoch}), reflecting the normal training-time convergence that occurs in vanilla training as well.

The mechanism underlying the inner-channel benefit is therefore accumulated parameter displacement: each batch receives $T$ gradient steps of similar magnitude in approximately the same direction, yielding a $T$-fold larger effective per-batch step relative to vanilla. This accelerates optimization without exposing the model to additional data and without altering the learning-rate schedule, and it is exactly the mini-batch reuse setting in which the SGD-form implicit regularization of Proposition~\ref{prop:igr} accumulates linearly over the $T$ steps within a fixed batch (Section~\ref{sec:framework-igr}).

\subsection{Loss Curve Comparison}
\label{app:inner_step_loss_curves}

Figure~\ref{fig:loss_curves} compares training and validation MAE across the three configurations on PeMS04 (DiffConv, seed$=$42).

\begin{figure}[h!]
\centering
\includegraphics[width=\linewidth]{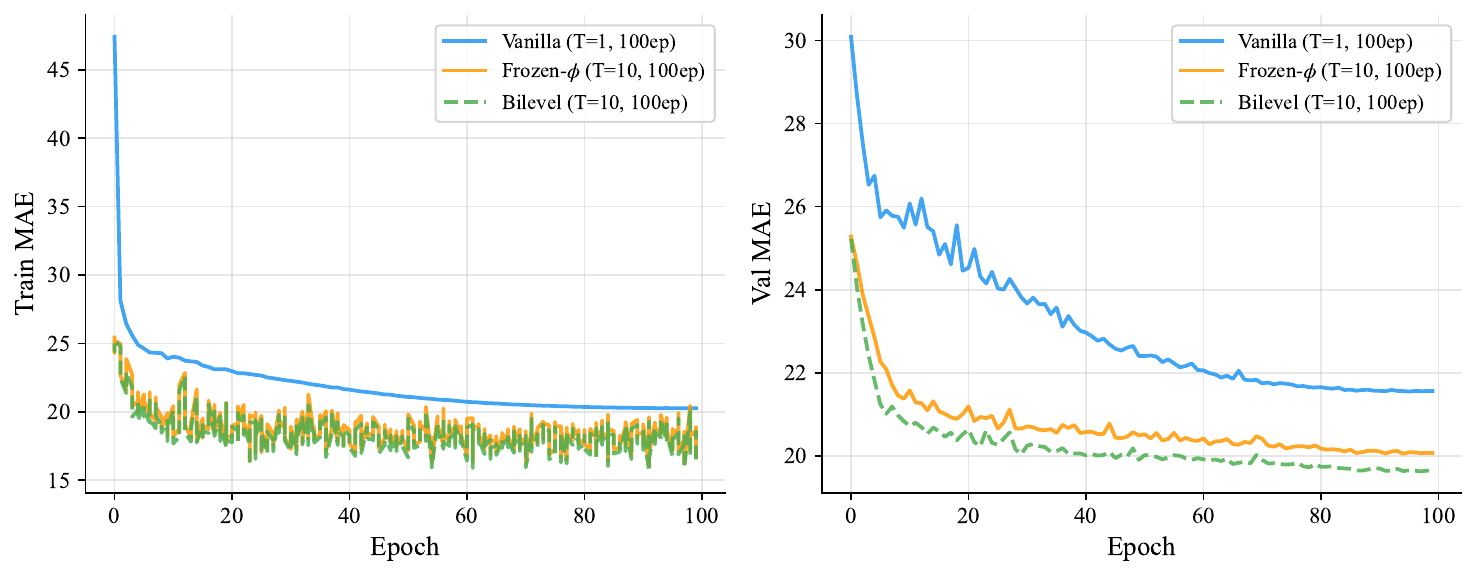}
\caption{Training and validation MAE on PeMS04 (DiffConv, seed$=$42). Vanilla ($T = 1$) converges slowly; frozen-$\phi$ and bilevel (both $T = 10$) converge rapidly under the same epoch count and the same learning-rate schedule. By epoch 10, both $T = 10$ configurations reach the validation MAE that vanilla achieves only near epoch 60. The gap arises from per-batch multi-step optimization, not from additional data exposure.}
\label{fig:loss_curves}
\end{figure}

Both frozen-$\phi$ and full bilevel converge substantially faster than vanilla: by epoch 10 they reach the validation MAE that vanilla achieves only near epoch 60 of training. Validation curves track training curves without divergence, indicating that the faster convergence translates to better generalization rather than to overfitting. The visualization complements the gradient-norm analysis above: the same accumulated-displacement mechanism that produces flat within-batch gradient profiles also produces faster across-epoch convergence in test-relevant metrics under the standard data and learning-rate budget (Table~\ref{tab:fair_comparison}).

\section{Weight Decay Sweep}
\label{app:wd_sweep}

\begin{table}[h!]
\centering
\caption{Weight decay sweep. ST PeMS04 column reports DiffConv test MAE. NC columns report vanilla GCN with weight decay added to the inner Adam optimizer (no LDS). The best ST setting ($\lambda = 10^{-4}$) does not reach vanilla; on NC the best $\lambda$ improves the GCN baseline but does not match the LDS bilevel result ($82.75\%$ on Cora, $73.53\%$ on Citeseer). The ``$\lambda{=}0$ (vanilla)'' NC row reports a 5-seed run with weight decay disabled and differs from the main-text vanilla GCN baseline (Table~\ref{tab:decomp}, Cora $81.10\%$ / Citeseer $69.28\%$), which uses GSLB's per-dataset tuned weight decay ($5\times 10^{-4}$, see Table~\ref{tab:hp_nc}) over 10 seeds.}
\label{tab:wd}
\small
\begin{tabular}{lccc}
\toprule
Weight decay $\lambda$ & PeMS04 (ST) & Cora (NC, GCN+WD) & Citeseer (NC, GCN+WD) \\
\midrule
$0$ (vanilla)  & $20.925_{\pm.043}$ & $79.40_{\pm.50}\%$ & $64.44_{\pm1.18}\%$ \\
$10^{-4}$      & $21.001_{\pm.082}$ & $80.64_{\pm.29}\%$ & $67.08_{\pm.68}\%$ \\
$10^{-3}$      & $21.888_{\pm.114}$ & $80.84_{\pm.91}\%$ & $69.14_{\pm.90}\%$ \\
$10^{-2}$      & $23.428_{\pm.031}$ & $48.88_{\pm2.63}\%$ & $36.16_{\pm5.56}\%$ \\
\bottomrule
\end{tabular}
\end{table}

\begin{figure}[h!]
\centering
\includegraphics[width=0.55\linewidth]{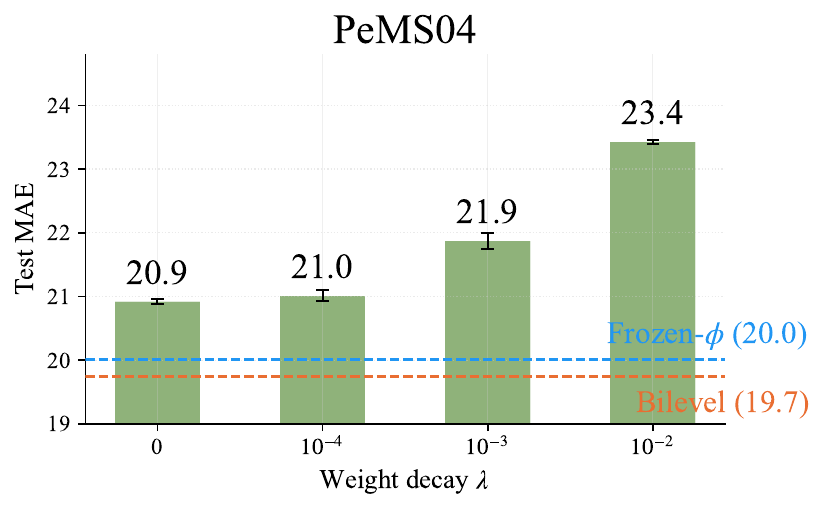}
\caption{ST weight decay sweep (PeMS04, DiffConv). No weight decay setting reaches frozen-$\phi$ or bilevel: $\lambda = 10^{-4}$ is marginally worse than vanilla and larger $\lambda$ degrades rapidly.}
\label{fig:wd_st}
\end{figure}

\begin{figure}[h!]
\centering
\includegraphics[width=\linewidth]{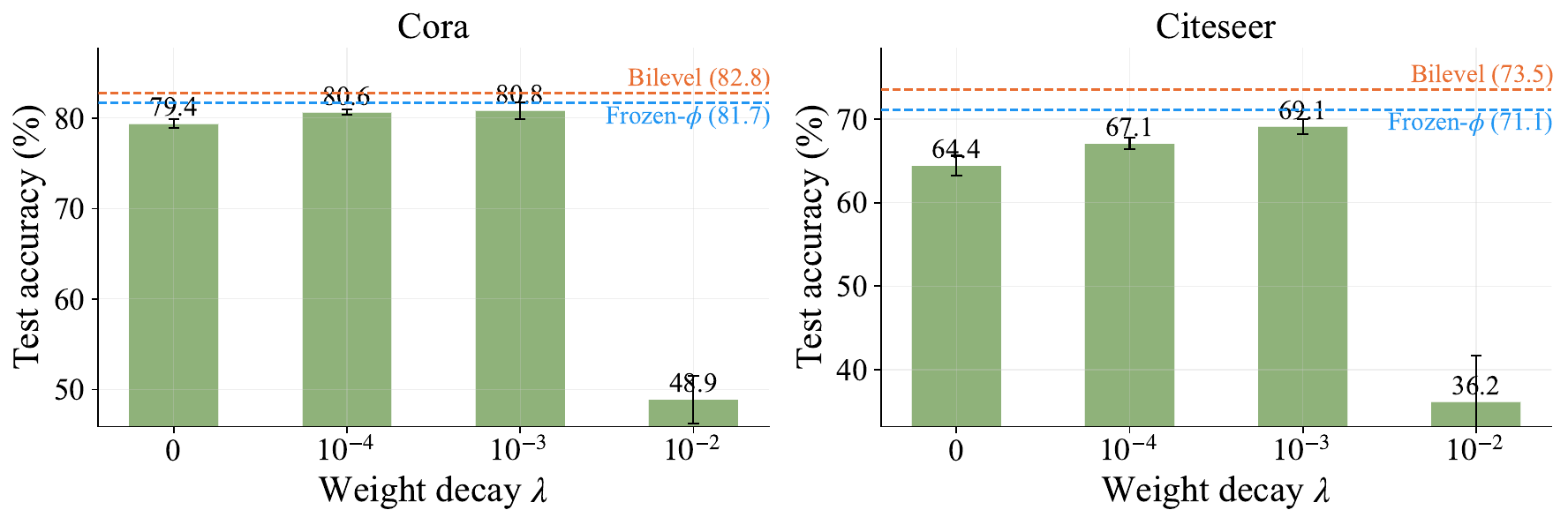}
\caption{NC weight decay sweep (Cora, Citeseer). Weight decay closes part of the GCN-to-LDS gap but does not reach frozen-$\phi$ on either dataset (Cora: $\lambda = 10^{-3}$ at $80.84\%$ vs.\ frozen-$\phi$ $81.71\%$ vs.\ LDS $82.75\%$; Citeseer: $69.14\%$ vs.\ $71.13\%$ vs.\ $73.53\%$). The residual gap to frozen-$\phi$ is tighter than in ST because the NC inner channel itself is small ($37$--$44\%$) under the full-batch reset regime.}
\label{fig:wd_nc}
\end{figure}

\paragraph{ST.} On PeMS04, no weight decay setting matches the frozen-$\phi$ MAE of $20.011$ at $T = 10$: the smallest $\lambda$ of $10^{-4}$ gives MAE $21.001$, marginally worse than vanilla, while $\lambda = 10^{-3}$ and $10^{-2}$ degrade rapidly (Figure~\ref{fig:wd_st}). Explicit $\ell_2$ regularization on parameter norms is qualitatively different from the gradient-norm-based implicit regularization of Proposition~\ref{prop:igr}, and the inability of weight decay to replicate the frozen-$\phi$ effect is consistent with this functional difference.

\paragraph{NC.} The same mechanism check on NC shows a weaker but qualitatively consistent pattern (Figure~\ref{fig:wd_nc}). Weight decay closes part of the GCN-to-LDS gap on both datasets but does not reach frozen-$\phi$: on Cora the best setting ($\lambda = 10^{-3}$) reaches $80.84\%$ versus frozen-$\phi$ $81.71\%$ and LDS $82.75\%$; on Citeseer $\lambda = 10^{-3}$ reaches $69.14\%$ versus frozen-$\phi$ $71.13\%$ and LDS $73.53\%$. Parameter-norm regularization recovers a fraction of the inner-channel benefit but cannot replicate it, and the residual graph-channel gap (frozen-$\phi$ to LDS) is unaffected by weight decay. The NC inner channel itself is small ($37$--$44\%$) under the full-batch reset regime, so the residual gap between weight decay and frozen-$\phi$ ($0.87$pp Cora, $1.99$pp Citeseer) is correspondingly tighter than the ST gap ($\sim 1$ MAE).

\section{Compute-matched vanilla baseline}
\label{app:compute_matched}

The frozen-$\phi$ control retains the $T{=}10$ inner schedule, performing $10\times$ as many parameter updates per batch as the $T{=}1$ vanilla baseline. To isolate whether the inner channel reflects implicit gradient regularization (IGR) versus simple under-convergence of the vanilla baseline at the standard 100-epoch budget, we run vanilla DiffConv on PeMS04 for 1000 epochs ($10\times$ compute), with the cosine learning-rate schedule extended proportionally ($T_{\max}=1000$).

\begin{table}[h!]
\centering
\caption{Compute-matched comparison on PeMS04 (DiffConv). All settings share data exposure per epoch and learning-rate schedule shape; only the epoch count and inner-step count differ. Frozen-$\phi$ and bilevel use the standard 100-epoch budget with $T{=}10$ inner steps; the 1000-epoch vanilla matches the bilevel total parameter-update count without inner-loop reuse.}
\label{tab:compute_matched}
\small
\begin{tabular}{lccc}
\toprule
Setting & Epochs & Updates/batch & Test MAE \\
\midrule
Vanilla              & 100  & 1  & $20.925_{\pm.043}$ \\
Vanilla (10$\times$) & 1000 & 1  & $20.189_{\pm.024}$ \\
Frozen-$\phi$        & 100  & 10 & $20.011_{\pm.034}$ \\
Bilevel              & 100  & 10 & $19.754_{\pm.036}$ \\
\bottomrule
\end{tabular}
\end{table}

\paragraph{Trajectory-length interpretation.} The 1000-epoch vanilla result lies between the 100-epoch vanilla baseline and frozen-$\phi$, closing $80\%$ of the inner-channel gap (from $0.914$ MAE down to a residual $0.178$ MAE relative to frozen-$\phi$). We interpret this through the lens of~\citet{smith2021origin}: per-step IGR accumulates over the optimization trajectory, so two settings that traverse comparable trajectory lengths exhibit comparable accumulated regularization. Frozen-$\phi$ $T{=}10$ at $100$ epochs and vanilla $T{=}1$ at $1000$ epochs both perform $\sim 10^4$ gradient steps and reach a similar regime; the former does so within the standard wall-clock budget, the latter requires $10\times$ wall-clock. Both manifest IGR; frozen-$\phi$ achieves it more efficiently within the standard training budget. This refines, rather than refutes, the inner-channel attribution: the $78$--$101\%$ inner-share figure (Table~\ref{tab:decomp}) reflects matched-compute attribution in the standard $100$-epoch regime, which is the operating point at which bilevel GSL methods are typically reported.

\paragraph{Bilevel residual is unaffected by extended vanilla training.} The bilevel result ($19.754$) remains $0.435$ MAE below the $1000$-epoch vanilla, while the standard frozen-$\phi$ vs.\ bilevel gap is $0.257$ MAE. The graph-channel contribution measured in the main text is therefore not an artifact of comparing against an under-trained vanilla baseline; the gap to bilevel persists under $10\times$ vanilla compute. Combined with the end-to-end joint-training baseline (Section~\ref{sec:find-scope}, $20\times$ weaker than bilevel on PeMS07) and the graph-distillation diagnostic (Appendix~\ref{app:distillation}), the graph-channel contribution is robust to compute-matching as a confound.

\section{Seed-Level Distributions for the Frozen-\texorpdfstring{$\phi$}{phi} Decomposition}
\label{app:seed_distributions}

\begin{figure}[h!]
\centering
\includegraphics[width=0.9\linewidth]{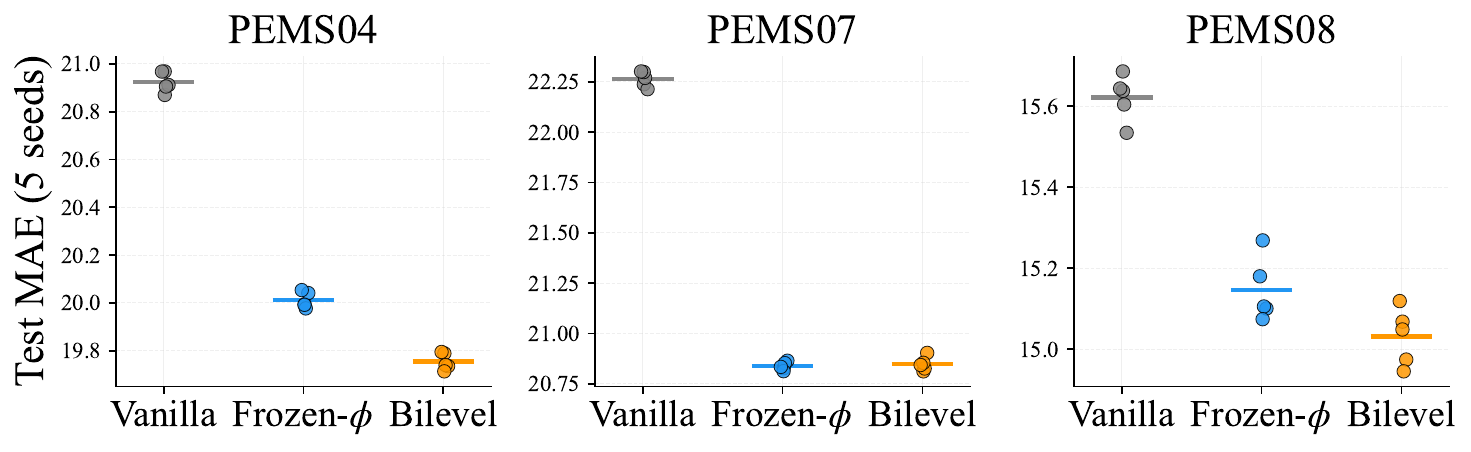}
\caption{Per-seed test MAE for vanilla, frozen-$\phi$, and full bilevel on PeMS04, PeMS07, and PeMS08. Frozen-$\phi$ and vanilla seed distributions are clearly separated on PeMS07 and PeMS08; on PeMS04 they nearly separate (frozen-$\phi$ max $20.960$ versus vanilla min $20.870$, gap $-0.09$\,MAE). The pattern supports the central decomposition claim that the inner channel provides a clean, seed-stable contribution distinct from per-seed noise.}
\label{fig:seed_scatter}
\end{figure}

%% file: appendix_03_graph_channel.tex
\section{Full Six-Dataset Survey}
\label{app:survey}

\begin{table}[h!]
\centering
\caption{Test MAE ($\downarrow$) on six benchmarks (DiffConv). Best bold, second underlined. Static rewiring methods achieve negligible changes on flow datasets where bilevel achieves $3.8$--$6.4\%$ improvement.}
\label{tab:survey}
\small
\begin{tabular}{lcccccc}
\toprule
Method & PeMS04 & PeMS07 & PeMS08 & METR-LA & PeMS-BAY & AirQuality \\
\midrule
Vanilla & $20.925_{\pm.043}$ & $22.277_{\pm.038}$ & $15.621_{\pm.057}$ & $\mathbf{3.240}_{\pm.006}$ & $\underline{1.665}_{\pm.002}$ & $\underline{27.631}_{\pm.214}$ \\
FoSR    & $20.934_{\pm.020}$ & $\underline{22.269}_{\pm.011}$ & $\underline{15.606}_{\pm.074}$ & $\underline{3.247}_{\pm.010}$ & $1.661_{\pm.002}$ & $27.661_{\pm.204}$ \\
BORF    & $20.939_{\pm.042}$ & $22.467_{\pm.024}$ & $15.613_{\pm.062}$ & $3.344_{\pm.011}$ & $1.707_{\pm.002}$ & $\mathbf{27.538}_{\pm.260}$ \\
SDRF    & $21.092_{\pm.030}$ & $22.502_{\pm.023}$ & $15.820_{\pm.055}$ & $3.340_{\pm.004}$ & $1.706_{\pm.002}$ & $27.797_{\pm.299}$ \\
GTR     & $\underline{20.921}_{\pm.053}$ & $22.268_{\pm.023}$ & $15.667_{\pm.062}$ & $3.257_{\pm.013}$ & $\mathbf{1.662}_{\pm.003}$ & $27.639_{\pm.377}$ \\
Bilevel & $\mathbf{19.754}_{\pm.036}$ & $\mathbf{20.848}_{\pm.035}$ & $\mathbf{15.031}_{\pm.071}$ & $3.323_{\pm.016}$ & $1.654_{\pm.002}$ & $27.652_{\pm.229}$ \\
\midrule
$\Delta$ (\%) & $-5.6$ & $-6.4$ & $-3.8$ & $+2.5$ & $-0.7$ & $+0.1$ \\
\bottomrule
\end{tabular}
\end{table}

\section{Spectrum visualization: GTR maximizes \texorpdfstring{$\lambda_2$}{lambda2}, bilevel does not}
\label{app:spectral_band}

Figure~\ref{fig:spectrum_6} shows the eigenvalue spectrum of the normalized Laplacian on the LCC subgraph for the original adjacency, GTR (the strongest spectral-gap maximizer among static methods on this benchmark suite), and the bilevel-learned adjacency on each of the six ST datasets. The figure visualizes the dissociation documented numerically in Appendix~\ref{app:spectral_negative}: GTR shifts the spectrum upward on every benchmark, increasing $\lambda_2$ on the LCC by $2.6$--$15.5\times$, yet matches vanilla MAE; bilevel produces qualitatively different redistribution that does not maximize $\lambda_2$ (often decreases it) while producing the largest forecasting gains on flow. FoSR, BORF, and SDRF produce minimal spectral change on the LCC subgraph and are reported numerically in Table~\ref{tab:spectral_negative_lambda} and~\ref{tab:spectral_negative_we}. The same three regime partition that appears in the main results is visible: (i) flow datasets (PeMS04/07/08) show bilevel-driven spectral redistribution distinct from GTR's gap-maximization, with a forecasting gain; (ii) speed datasets (METR-LA, PeMS-BAY) show bilevel modification at a similar magnitude as GTR but no task gain, indicating that spectral redistribution by itself does not control forecasting gain; (iii) AirQuality shows minimal spectral change combined with no task effect.

\begin{figure}[h!]
\centering
\includegraphics[width=\linewidth]{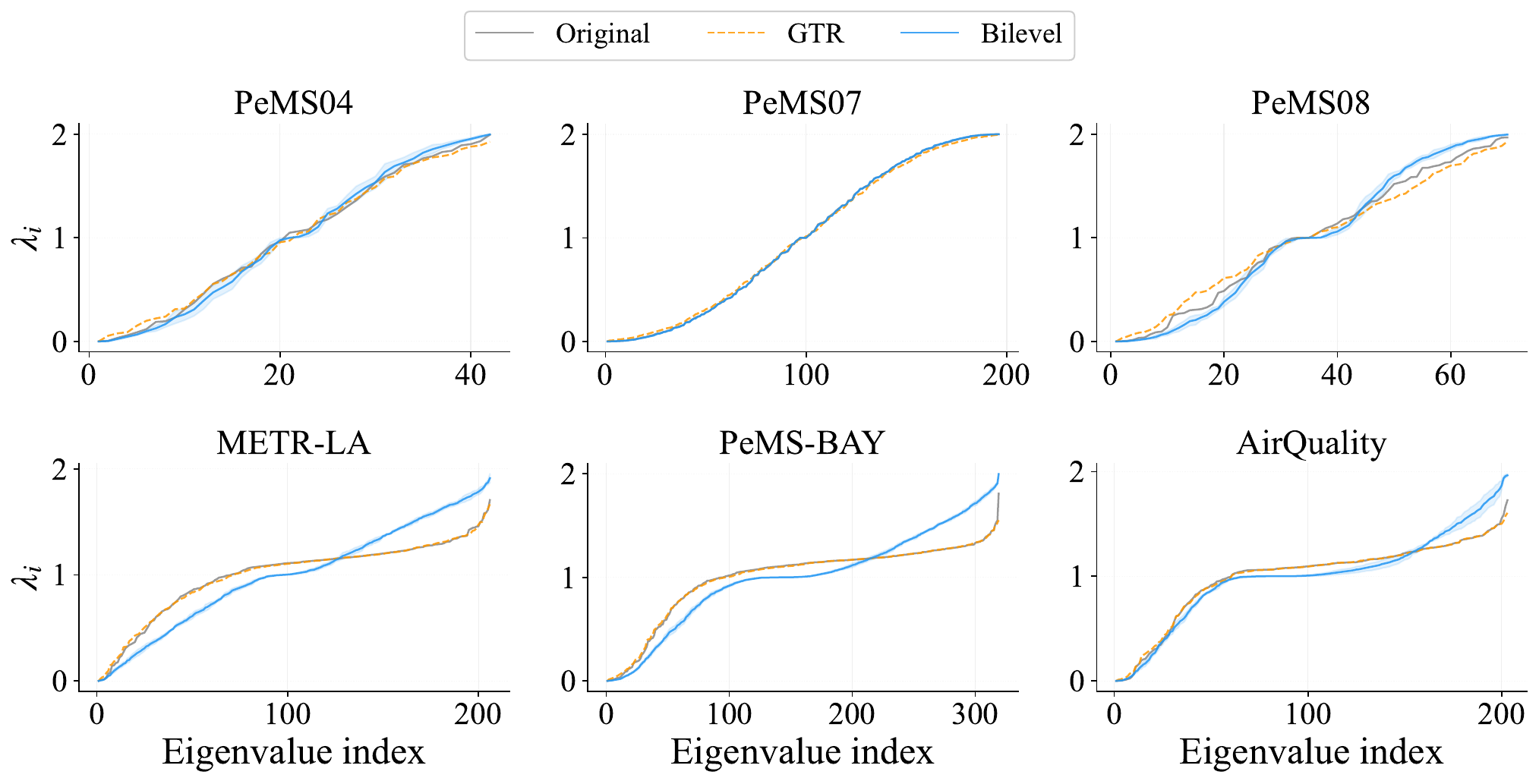}
\caption{Eigenvalue spectrum of the normalized Laplacian on the LCC subgraph (identical node set across methods) for six ST datasets, comparing the original adjacency, GTR (the strongest spectral-gap maximizer), and the bilevel-learned adjacency. GTR shifts the spectrum upward by $2.6$--$15.5\times$ in $\lambda_2$ on every benchmark (Table~\ref{tab:spectral_negative_lambda}) yet matches vanilla MAE; bilevel produces qualitatively different redistribution that does not maximize $\lambda_2$. FoSR, BORF, and SDRF produce minimal spectral change on the LCC subgraph and are reported numerically in Appendix~\ref{app:spectral_negative}.}
\label{fig:spectrum_6}
\end{figure}
\FloatBarrier

\section{Edge-Level Modifications}
\label{app:edge_modifications}

The frequency-based view of how bilevel reshapes the graph's Laplacian (Figure~\ref{fig:spectrum_6}) is also visible at the edge level: Figure~\ref{fig:weight_dist_6} shows the distribution of bilevel-learned edge weights $A_\phi = A \odot \mathrm{softmax}_{\mathrm{row}}(W_\phi)$ versus the original distance-based Gaussian-kernel adjacency for each of the six ST datasets. On flow-signal datasets (PeMS04/07/08), the originally unimodal weight distribution becomes sharply bimodal under bilevel: a large mass concentrated near zero (effective pruning of many edges) and a smaller mass at the typical retained-edge weight. The bimodal pruning at the edge level is the direct algebraic source of the spectral redistribution shown in Figure~\ref{fig:spectrum_6}: removing the many small-weight edges that contribute to mid-frequency content leaves a sparser graph whose spectrum is correspondingly compressed. On speed-signal datasets bilevel produces qualitatively similar bimodality with substantially less sharp separation between modes, and on AirQuality the distribution collapses to a near-single-mode form rather than bimodal restructuring. Together with the spectrum visualizations and the per-dataset Jacobian measurements of Section~\ref{app:jacobian_cross}, the three views yield the same dataset taxonomy: flow signals admit selective minimal intervention with task gain; speed signals admit modification without task alignment; AirQuality admits modification without Jacobian-level effect.

\begin{figure}[h!]
\centering
\includegraphics[width=\linewidth]{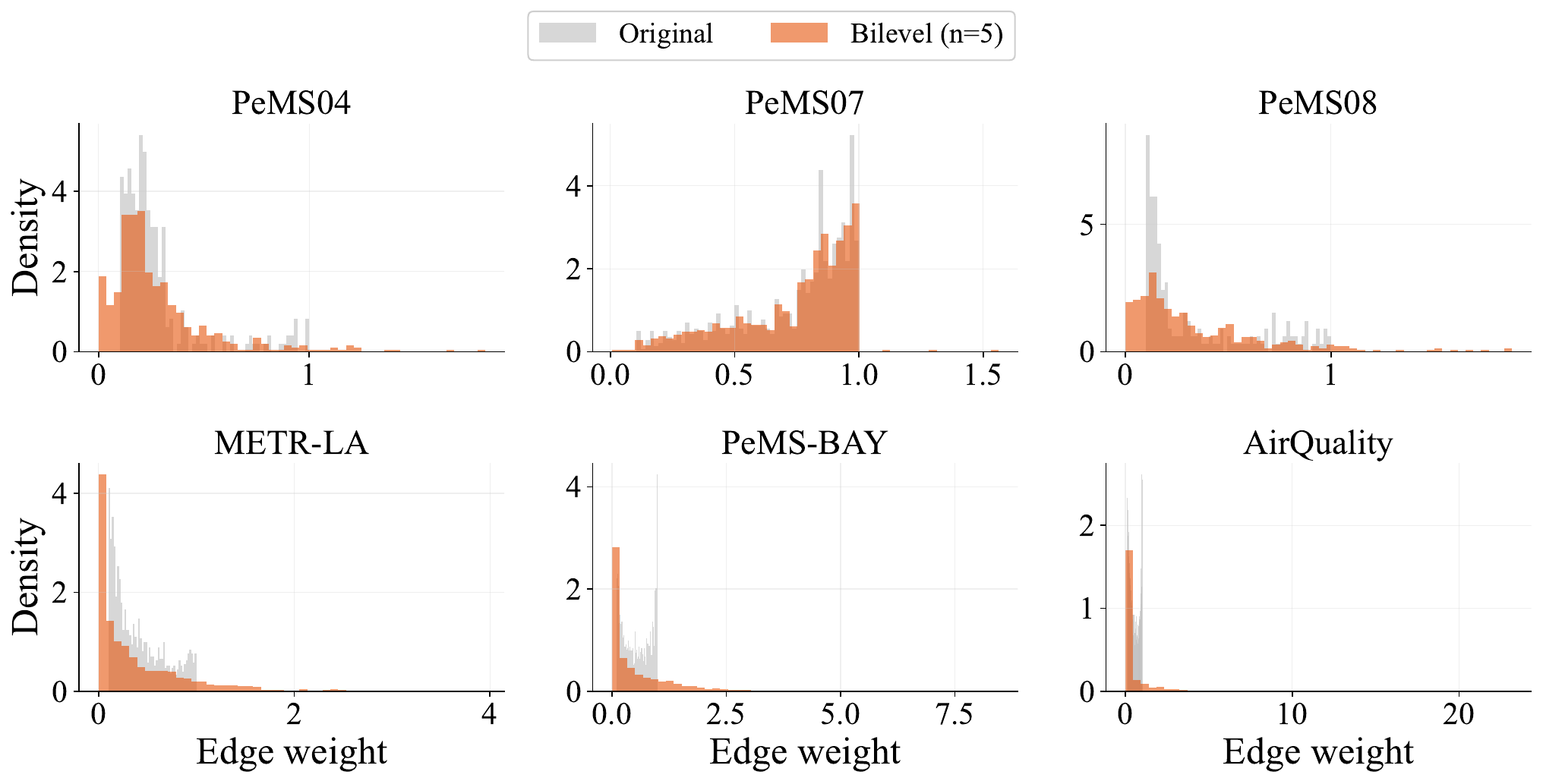}
\caption{Bilevel-learned edge weight distributions versus the original distance-based Gaussian-kernel adjacency on six ST datasets. Flow datasets (PeMS04/07/08) show sharply bimodal post-bilevel distributions; speed datasets (METR-LA, PeMS-BAY) show less sharp bimodality with no task gain; AirQuality shows near-single-mode collapse.}
\label{fig:weight_dist_6}
\end{figure}
\FloatBarrier

\section{Jacobian Cross-Regime Analysis}
\label{app:jacobian_cross}

The Jacobian measurement provides a complementary view of how graph modification affects information propagation. We compute the input-output Jacobian norm $\|\partial \hat{y}^v_T / \partial x^u_0\|_2$ stratified by graph distance between source node $u$ and target node $v$, for both ST and NC regimes (Figure~\ref{fig:jacobian_st_6} for ST; Table~\ref{tab:nc_jacobian} and Figure~\ref{fig:jacobian_nc} for NC). The two regimes exhibit distinct Jacobian signatures (ST: selective short-range amplification driven by per-pair task-aligned reweighting; NC: uniform mild reduction on Citeseer and aggressive sparsification on Cora) yet converge on the same negative diagnostic: neither shows the systematic long-range Jacobian amplification that would be expected if the bilevel gain originated from over-squashing relief.

\begin{figure}[h!]
\centering
\includegraphics[width=\linewidth]{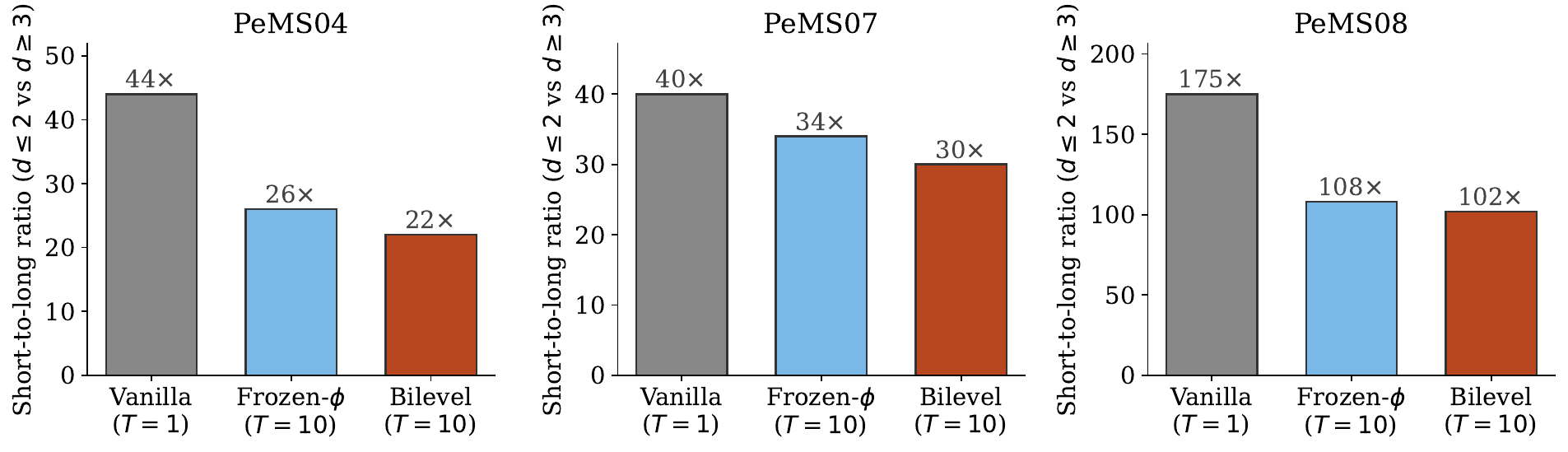}
\caption{Input-output Jacobian norm stratified by graph distance for vanilla, frozen-$\phi$, and bilevel models on the three flow datasets. Vanilla short-to-long ($d \leq 2$ vs $d \geq 3$) ratios are $44\times$ (PeMS04), $40\times$ (PeMS07), $175\times$ (PeMS08); frozen-$\phi$ ratios are $26\times$, $34\times$, $108\times$; bilevel ratios are $22\times$, $30\times$, $102\times$. The vanilla-to-frozen-$\phi$ shift (same adjacency, only inner schedule differs from $T{=}1$ to $T{=}10$) accounts for the majority of the vanilla-to-bilevel reduction (PeMS04: $82\%$; PeMS07: $60\%$; PeMS08: $92\%$); the additional frozen-$\phi$-to-bilevel shift reflects how weights trained against $A_\phi$ behave when evaluated on $\Ainit$. Long-range pairs ($d \geq 6$) are numerically zero on every dataset; the $K{=}2$, $L{=}4$ DiffConv has a maximal $8$-hop receptive field. Wide error bars on bilevel reflect selective task-aligned amplification of specific short-range pairs whose identity varies seed-to-seed, not measurement noise. We restrict the frozen-$\phi$ control to the three datasets where bilevel improves over vanilla and the channel decomposition is well-defined; speed and AirQuality datasets fall outside this scope (Section~\ref{sec:find-scope}).}
\label{fig:jacobian_st_6}
\end{figure}

\FloatBarrier

\paragraph{Node classification: Jacobian comparison for LDS.} For LDS on Cora and Citeseer we compute the same Jacobian quantity using the inference-time GCN under (i) the original adjacency and (ii) the LDS-learned adjacency, stratified by graph distance $d \in \{1, 2, 3, 4, {\geq}5\}$. Table~\ref{tab:nc_jacobian} reports both the mean and the per-pair median (a more robust statistic) at each distance. The two datasets exhibit qualitatively different patterns. On Citeseer the mean ratio (LDS $\div$ original) is below $1.0$ at every distance ($0.86$--$0.93$) and the median ratio is similarly near $1.0$ ($0.97$--$0.99$): LDS produces a uniform mild reduction with no distance-dependent amplification. On Cora the mean ratio drops sharply at short range ($0.28$ at $d{=}1$, $0.18$ at $d{=}2$) and rises above $1.0$ at $d{=}4$ ($3.50$) with a small nonzero norm at $d{\geq}5$ where the original is identically zero, but the per-pair median is identically zero at every distance under LDS. The mean-driven signal is therefore outlier-driven, consistent with the strongly bimodal Bernoulli edge distribution on Cora reported in Appendix~\ref{app:nc_mechanism} (Figure~\ref{fig:lds_edge_dist}, $72\%$ of $\theta_{ij}$ below $0.01$): LDS aggressively prunes the original adjacency and adds a small number of long-range edges that contribute to the mean of a few outlier pairs without producing systematic long-range amplification across the typical pair. Neither pattern is consistent with the over-squashing-relief mechanism, which would predict ratios above $1.0$ at large distances under both robust statistics. We interpret the LDS gain as task-discriminative edge placement plus inner-channel implicit regularization rather than over-squashing relief.

\begin{table}[h!]
\centering
\caption{NC Jacobian norms by graph distance. Columns: mean and per-pair median (a robust statistic that suppresses outlier-pair contributions). Ratio $=$ LDS $\div$ Original. On Citeseer both statistics are robust and below $1.0$ at every distance. On Cora the mean ratio appears amplified at $d{=}4$ ($3.50$) and $d{\geq}5$ (LDS produces a small nonzero value where the original is identically zero), but the per-pair median is identically zero at every distance under LDS, indicating the mean signal is driven by a small number of outlier pairs from LDS-added long-range edges rather than systematic relief. Neither dataset shows the long-range Jacobian amplification expected under an over-squashing-relief mechanism.}
\label{tab:nc_jacobian}
\small
\begin{tabular}{llcccccc}
\toprule
 & & \multicolumn{3}{c}{Mean} & \multicolumn{3}{c}{Per-pair median} \\
\cmidrule(lr){3-5} \cmidrule(lr){6-8}
Dataset & Distance & Original & LDS & Ratio & Original & LDS & Ratio \\
\midrule
\multirow{5}{*}{Cora}
 & $d{=}1$    & $140.963$ & $40.072$ & $0.284$ & $97.426$ & $0.000$ & $0.000$ \\
 & $d{=}2$    & $14.645$  & $2.687$  & $0.183$ & $4.415$  & $0.000$ & $0.000$ \\
 & $d{=}3$    & $2.415$   & $1.511$  & $0.626$ & $0.807$  & $0.000$ & $0.000$ \\
 & $d{=}4$    & $0.254$   & $0.889$  & $3.502$ & $0.052$  & $0.000$ & $0.000$ \\
 & $d{\geq}5$ & $0.000$   & $0.508$  & $-$     & $0.000$  & $0.000$ & $-$     \\
\midrule
\multirow{5}{*}{Citeseer}
 & $d{=}1$    & $317.125$ & $281.418$ & $0.887$ & $180.377$ & $175.981$ & $0.976$ \\
 & $d{=}2$    & $53.747$  & $47.968$  & $0.892$ & $12.478$  & $12.364$  & $0.991$ \\
 & $d{=}3$    & $8.143$   & $7.541$   & $0.926$ & $2.097$   & $2.020$   & $0.963$ \\
 & $d{=}4$    & $0.899$   & $0.770$   & $0.856$ & $0.199$   & $0.193$   & $0.968$ \\
 & $d{\geq}5$ & $0.000$   & $0.000$   & $-$     & $0.000$   & $0.000$   & $-$     \\
\bottomrule
\end{tabular}
\end{table}

\begin{figure}[h!]
\centering
\includegraphics[width=\linewidth]{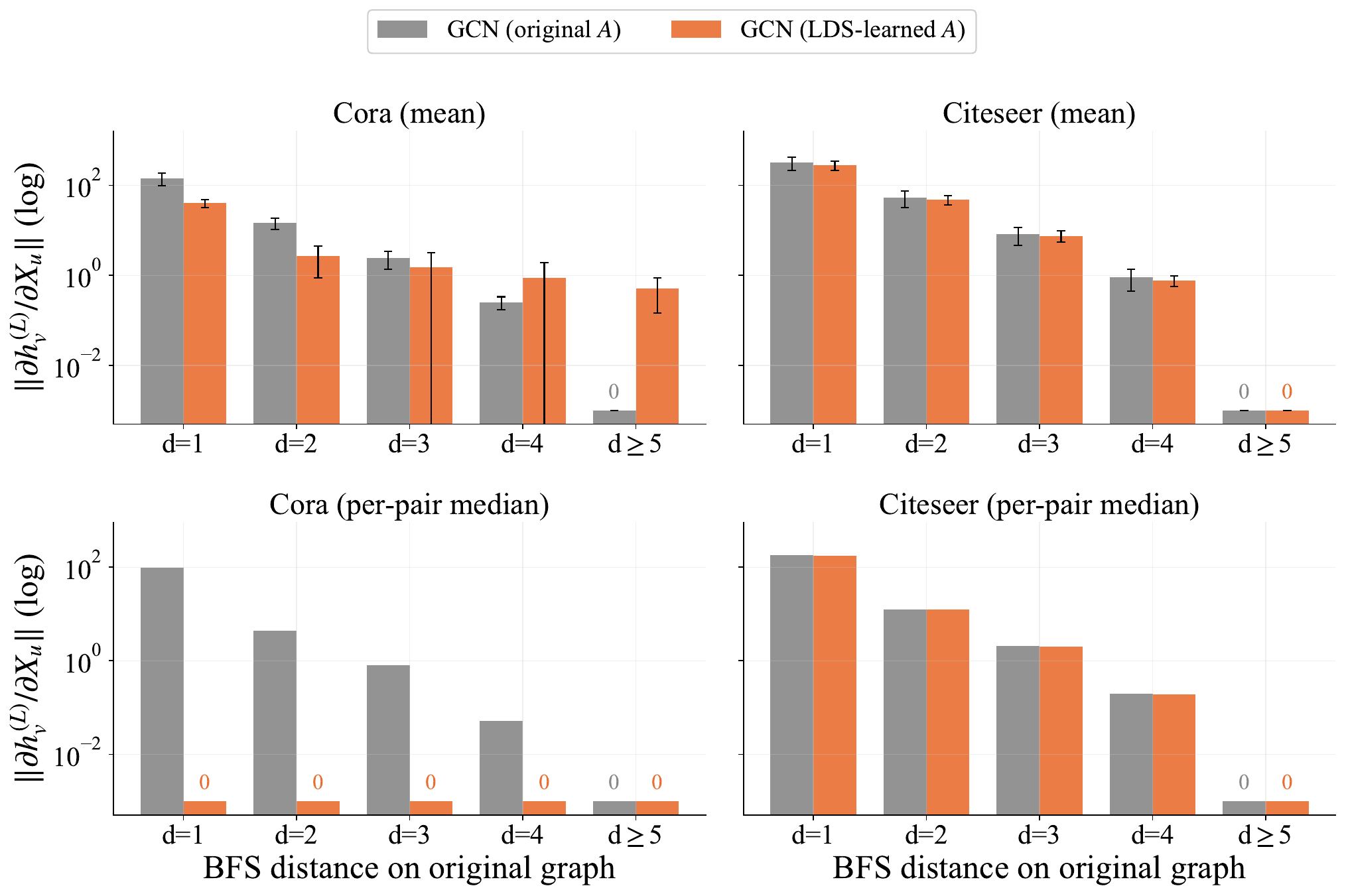}
\caption{NC Jacobian by graph distance. Top row: mean (matches Table~\ref{tab:nc_jacobian} mean columns). Bottom row: per-pair median, which suppresses outlier-pair contributions. On Citeseer both statistics show LDS at or below the original at every distance. On Cora the mean panel appears to show LDS amplification at $d{=}4$ and $d{\geq}5$, but the median panel reveals that the LDS Jacobian norm is identically zero at every distance under the typical pair: the apparent mean amplification is driven by a small number of LDS-added long-range edges that contribute to a few outlier pairs (max per-seed standard deviation $\sigma \approx 22.30$ at $d{=}4$, dwarfing the across-seed mean of $0.89$). Neither the mean nor the median panel exhibits the systematic long-range Jacobian amplification that would be expected if the LDS gain originated from over-squashing relief. Bars labeled ``$0$'' are at the display floor.}
\label{fig:jacobian_nc}
\end{figure}
\FloatBarrier

%% file: appendix_10_spectral_negative.tex
\section{Spectral metrics dissociate from task improvement}
\label{app:spectral_negative}

A natural alternative explanation for the divergence between bilevel and static rewiring is spectral: that bilevel optimizes a graph-spectral objective (larger spectral gap $\lambda_2$ or narrower effective spectral support $W_\epsilon$) that static methods do not match. We tested this directly on six benchmarks. The data shows a clean dissociation between graph-spectral metrics and task improvement: GTR maximizes $\lambda_2$ on every benchmark by an order of magnitude yet produces no MAE gain, while bilevel achieves the largest forecasting gains on flow without maximizing $\lambda_2$ (and often decreasing it).

\paragraph{LCC-aligned measurement protocol.} The standard ST adjacencies are highly fragmented under the threshold $\tau = 0.1$ used by the standard implementation (35 components on PeMS04, 93 on PeMS07, 33 on AirQuality). Static rewiring methods that add edges (FoSR, BORF) merge components and change the largest connected component (LCC); whole-graph $\lambda_2$ is identically zero for any disconnected graph and is therefore an uninformative quantity. We restrict every comparison to the original-graph LCC node subset: spectra are computed on the subgraph induced by the LCC of the unmodified adjacency, identical for every method. Bilevel preserves LCC structure exactly on all bilevel runs (incremental edge reweighting does not change connectivity); static methods that connect new components are evaluated on the restricted subset, dropping any newly-added inter-component edges. Whole-graph $\lambda_2$ is reported separately in Table~\ref{tab:spectral_negative_lambda}.

\paragraph{Spectral gap dissociates from MAE.} Table~\ref{tab:spectral_negative_lambda} reports $\lambda_2$ on the LCC subgraph for the original adjacency, four static rewiring methods, and the bilevel-learned adjacency. GTR is a clean spectral-gap maximizer: $\lambda_2$ increases by 2.6$\times$ to 15.5$\times$ across the five benchmarks where the original $\lambda_2$ is numerically resolved (PeMS07 has a near-zero baseline). Bilevel decreases $\lambda_2$ on four of six datasets, and where it increases (AirQuality), the increase does not produce a forecasting gain. The cross-method comparison against MAE (rightmost column, from Table~\ref{tab:survey}) shows no relationship between $\lambda_2$ change and task improvement: GTR maximizes $\lambda_2$ everywhere with zero MAE gain, while bilevel produces $-3.8$ to $-6.4\%$ on flow datasets without raising $\lambda_2$.

\begin{table}[h!]
\centering
\caption{Spectral gap $\lambda_2$ ($\times 10^3$) on the original-graph LCC subgraph. Whole-graph $\lambda_2$ is identically zero for disconnected graphs and is reported in the rightmost column for the four datasets where the LCC differs from the whole graph. GTR achieves 2.6 to 15.5$\times$ increase in LCC $\lambda_2$ on every benchmark with a numerically resolved baseline yet produces no MAE gain (Table~\ref{tab:survey}); bilevel decreases $\lambda_2$ on four of six datasets while producing the largest forecasting gains on flow.}
\label{tab:spectral_negative_lambda}
\small
\begin{tabular}{lcccccccc}
\toprule
Dataset & Original & FoSR & GTR & BORF & SDRF & Bilevel & $\Delta$ MAE bilevel & Whole-graph $\lambda_2$ \\
\midrule
PeMS04   & $6.0$ & $6.0$ & $55.0$ & $13.0$ & $6.0$ & $3.0_{\pm 2.0}$ & $-5.6\%$ & $0$ \\
PeMS07   & $0.0$ & $0.0$ & $8.0$  & $1.0$  & $0.0$ & $0.0_{\pm 0.0}$ & $-6.4\%$ & $0$ \\
PeMS08   & $2.0$ & $2.0$ & $31.0$ & $4.0$  & $3.0$ & $1.0_{\pm 1.0}$ & $-3.8\%$ & $0$ \\
METR-LA  & $8.0$ & $15.0$ & $21.0$ & $8.0$  & $8.0$ & $5.0_{\pm 1.0}$ & $+2.5\%$ & $\approx \lambda_2^{\mathrm{LCC}}$ \\
PeMS-BAY & $5.0$ & $5.0$ & $13.0$ & $9.0$  & $5.0$ & $1.0_{\pm 1.0}$ & $-0.7\%$ & $\approx \lambda_2^{\mathrm{LCC}}$ \\
AirQuality & $1.0$ & $1.0$ & $11.0$ & $3.0$  & $1.0$ & $2.0_{\pm 0.0}$ & $+0.1\%$ & $0$ \\
\bottomrule
\end{tabular}
\\[0.4em]
\footnotesize Values in $\times 10^3$; original $\lambda_2$ for PeMS07 is below display precision but greater than zero, so a multiplier against GTR is ill-defined. METR-LA (2 components) and PeMS-BAY (7 components) have whole-graph spectra dominated by the LCC.
\end{table}

\paragraph{Effective spectral support also dissociates.} Table~\ref{tab:spectral_negative_we} reports $W_\epsilon$ at $\epsilon = 0.05$ on the same LCC subgraph. Static methods produce minimal $W_\epsilon$ change in either direction (typically within $\pm 0.05$). Bilevel widens $W_\epsilon$ on five of six datasets, with the largest widening on the three datasets where it does not produce a forecasting gain (METR-LA $+35.3\%$, PeMS-BAY $+42.2\%$, AirQuality $+43.0\%$); on the three flow datasets where bilevel does succeed, $W_\epsilon$ change is small (PeMS04 $-1.8\%$, PeMS07 $+5.5\%$, PeMS08 $+5.1\%$). $W_\epsilon$ change has no consistent sign correlation with forecasting gain.

\begin{table}[h!]
\centering
\caption{Effective spectral support $W_\epsilon$ at $\epsilon = 0.05$ on the original-graph LCC subgraph. Static methods produce minimal $W_\epsilon$ change. Bilevel widens $W_\epsilon$ most strongly on the three datasets where it does not produce a forecasting gain (METR-LA, PeMS-BAY, AirQuality).}
\label{tab:spectral_negative_we}
\small
\begin{tabular}{lcccccc}
\toprule
Method & PeMS04 & PeMS07 & PeMS08 & METR-LA & PeMS-BAY & AirQuality \\
\midrule
Original & $1.827$ & $1.532$ & $1.838$ & $1.256$ & $1.214$ & $1.181$ \\
FoSR     & $1.827$ & $1.510$ & $1.838$ & $1.255$ & $1.214$ & $1.181$ \\
GTR      & $1.805$ & $1.522$ & $1.821$ & $1.251$ & $1.215$ & $1.186$ \\
BORF     & $1.856$ & $1.419$ & $1.953$ & $1.287$ & $1.220$ & $1.181$ \\
SDRF     & $1.936$ & $1.483$ & $1.879$ & $1.289$ & $1.215$ & $1.179$ \\
Bilevel  & $1.794_{\pm.097}$ & $1.617_{\pm.002}$ & $1.932_{\pm.029}$ & $1.700_{\pm.037}$ & $1.726_{\pm.045}$ & $1.689_{\pm.121}$ \\
\bottomrule
\end{tabular}
\end{table}

\paragraph{Implication.} GTR is a strong spectral-gap maximizer that produces no forecasting gain; bilevel produces the largest forecasting gains on flow without maximizing $\lambda_2$ or narrowing $W_\epsilon$. Spectral gap optimization and effective-support narrowing are therefore dissociated from task improvement on these benchmarks. This does not contradict the worst-case propagation bounds of~\citet{topping2022understanding,digiovanni2023oversquashing,marisca2025oversquashing}, which bound a different quantity (worst-case sensitivity, not task loss); it constrains the space of mechanistic explanations for the bilevel gain by ruling out a natural one. Combined with the inner-channel attribution (Sections~\ref{sec:find-decomp}--\ref{sec:find-mechanism}), this strengthens the conclusion that the bilevel gain on flow forecasting is dominated by inner-loop training dynamics rather than by graph modification at the spectral level.

%% file: appendix_04_architecture.tex
\section{Architecture Selectivity}
\label{app:backbone_scope}

\begin{table}[h!]
\centering
\caption{Backbone scope on PeMS04. Decoupled polynomial-filter backbones converge to a narrow test-MAE floor despite varying vanilla baselines. DCRNN is empirically near-flat in the bilevel channel.}
\label{tab:backbone_scope}
\small
\begin{tabular}{llccc}
\toprule
Backbone & Regime & Vanilla & Bilevel & $\Delta$ MAE \\
\midrule
DiffConv     & Decoupled & $20.925_{\pm.043}$           & $19.754_{\pm.036}$         & $-5.6\%$ \\
ChebConv     & Decoupled & $20.742_{\pm.035}$           & $19.812_{\pm.035}$         & $-4.5\%$ \\
MPGRU        & Decoupled & $21.081_{\pm.019}$           & $19.814_{\pm.038}$           & $-6.0\%$ \\
DCRNN        & Coupled   & $20.154_{\pm.032}$           & $20.056_{\pm.112}$           & $-0.5\%$ \\
\bottomrule
\end{tabular}
\end{table}

The DCRNN coupled backbone consumes an external adjacency but embeds the spatial operation inside GRU gates, breaking the factored Jacobian of~\cite{marisca2025oversquashing}. Bilevel gain on DCRNN is within noise of zero, while cross-backbone graph distillation (DiffConv-learned adjacency transferred to vanilla DCRNN) improves by $-1.8\%$ (Table~\ref{tab:distill_cross}). This is consistent with the coupled gradient pathway preventing the outer loop from updating the graph effectively while graph quality itself remains relevant once provided as a fixed input.

\section{Adaptive Baseline Reference}
\label{app:adaptive_baselines}

\begin{table}[h!]
\centering
\caption{Absolute performance of adaptive spatio-temporal architectures on flow datasets (test MAE, lower is better). AGCRN does not consume an external adjacency in the forward pass~\citep{bai2020adaptive}; GraphWaveNet's external adjacency usage is confounded by its simultaneously learned internal matrix~\citep{wu2019graph}. Neither admits the frozen-$\phi$ decomposition as a well-defined control. The table is provided for reference.}
\label{tab:adaptive}
\small
\begin{tabular}{lccc}
\toprule
Method & PeMS04 & PeMS07 & PeMS08 \\
\midrule
AGCRN (vanilla)        & $19.38^{\ast}$              & $20.57^{\ast}$              & $15.32^{\ast}$ \\
GraphWaveNet (vanilla) & $18.53^{\ast}$              & $20.47^{\ast}$              & $14.40^{\ast}$ \\
Bilevel DiffConv (ours) & $19.754_{\pm.036}$         & $20.848_{\pm.035}$          & $15.031_{\pm.071}$ \\
\bottomrule
\end{tabular}
\\[0.4em]
\footnotesize $^{\ast}$Adaptive backbone values reproduced by~\citet{liu2023staeformer} in their STAEformer evaluation, cross-validated by BasicTS+ and STID. The frozen-$\phi$ decomposition does not apply to adaptive architectures (Section~\ref{sec:framework-scope}); we do not report bilevel measurements for these backbones.
\end{table}

Adaptive STGNN architectures (AGCRN~\citep{bai2020adaptive}, GraphWaveNet~\citep{wu2019graph}) internalize graph learning into the forward pass. AGCRN's adaptive embedding $E E^\top$ replaces external adjacency entirely, and GraphWaveNet's self-adaptive matrix $\tilde{A}_{\mathrm{adp}}$ is jointly learned alongside any external adjacency. Neither admits a well-defined external rewiring channel for the frozen-$\phi$ decomposition (Section~\ref{sec:framework-scope}). We report their absolute performance in Table~\ref{tab:adaptive} from the STAEformer evaluation of~\citet{liu2023staeformer}, for cross-dataset comparability with bilevel DiffConv. The finding that $78$--$101\%$ of the reported bilevel gain on flow forecasting derives from inner-loop training dynamics holds within the external-rewiring regime that bilevel GSL methods claim to occupy, where ``where do gains come from'' is a meaningful question.

\section{Graph Distillation}
\label{app:distillation}

Frozen-$\phi$ and graph distillation are complementary diagnostics with different scopes. Frozen-$\phi$ measures the contribution of graph modification while \emph{holding inner-loop training constant}: $\theta$ is trained on $\Ainit$ with the same $T$-step inner loop. Distillation measures the contribution of graph modification \emph{plus graph-model co-adaptation}: $\theta$ is re-trained from scratch on $\Aphi$, partially reproducing the joint adaptation of $\theta$ and $\Aphi$ that occurred during bilevel training. Two diagnostically distinguishable patterns follow. When bilevel optimization succeeds, distillation tends to exceed the frozen-$\phi$ graph share by the magnitude of co-adaptation: on PeMS04 with DiffConv, distillation recovers $44\%$ of the bilevel gain while frozen-$\phi$ assigns $22\%$ to the graph channel, locating the $22$-point gap as the empirical co-adaptation magnitude on this dataset. When bilevel optimization is blocked by an architectural scope condition, distillation can exceed direct bilevel because distillation bypasses the blocked optimization path: on DCRNN/PeMS04, direct bilevel improves vanilla by only $-0.5\%$ (within noise) while distillation of a DiffConv-learned graph improves by $-1.8\%$. The two cases are not anomalies but expected signatures of the framework. By ``co-adaptation'' we mean the joint solution $(\theta^*, A_{\phi^*})$ reached by the bilevel inner loop, in which $\theta^*$ is tuned to the evolving $A_\phi$ at each outer iteration; distillation re-trains $\theta$ from random initialization on the final $A_\phi$, recovering only the portion of the joint solution that survives this re-initialization.

\begin{table}[h!]
\centering
\caption{Same-backbone graph distillation on decoupled backbones. DiffConv and MPGRU yield consistent graph shares within each dataset (PeMS04: 44/47\%, PeMS07: 1/1\%, PeMS08: 45/48\%). The bilevel-learned adjacency from a bilevel run with the indicated backbone is used as a fixed adjacency for vanilla training of the same backbone. Graph share is the distillation gain as a fraction of the bilevel gain. ChebConv negative shares reflect cross-filter incompatibility: a DiffConv-tuned adjacency does not transfer cleanly across spatial filters. The coupled DCRNN backbone falls outside the frozen-$\phi$ scope (Section~\ref{sec:framework-scope}); cross-backbone DCRNN evidence appears in Table~\ref{tab:distill_cross}.}
\label{tab:distill_same}
\small
\begin{tabular}{llcccc}
\toprule
Dataset & Backbone & Vanilla & Distilled ($A_\phi$) & Bilevel & Graph share \\
\midrule
\multirow{3}{*}{PeMS04}
 & DiffConv         & $20.925_{\pm.043}$ & $20.410_{\pm.026}$ & $19.754_{\pm.036}$ & $44\%$ \\
 & ChebConv         & $20.742_{\pm.035}$ & $20.814_{\pm.023}$ & $19.812_{\pm.035}$ & $-8\%$ \\
 & MPGRU            & $21.081_{\pm.019}$ & $20.486_{\pm.029}$ & $19.814_{\pm.038}$ & $47\%$ \\
\midrule
\multirow{3}{*}{PeMS07}
 & DiffConv         & $22.277_{\pm.038}$ & $22.260_{\pm.024}$ & $20.848_{\pm.035}$ & $1\%$ \\
 & ChebConv         & $32.186_{\pm1.090}$ & $31.702_{\pm.820}$ & $29.450_{\pm1.216}$ & $18\%$ \\
 & MPGRU            & $22.414_{\pm.036}$ & $22.406_{\pm.044}$ & $20.908_{\pm.039}$ & $1\%$ \\
\midrule
\multirow{3}{*}{PeMS08}
 & DiffConv         & $15.621_{\pm.057}$ & $15.353_{\pm.039}$ & $15.031_{\pm.071}$ & $45\%$ \\
 & ChebConv         & $15.425_{\pm.021}$ & $15.542_{\pm.037}$ & $14.898_{\pm.019}$ & $-22\%$ \\
 & MPGRU            & $15.651_{\pm.014}$ & $15.366_{\pm.018}$ & $15.062_{\pm.061}$ & $48\%$ \\
\bottomrule
\end{tabular}
\end{table}

\begin{table}[h!]
\centering
\caption{Cross-backbone graph transfer on PeMS04. The bilevel-learned DiffConv adjacency is transferred as a fixed adjacency for vanilla training of a different target backbone. MPGRU is decoupled (direct bilevel produces $-6.0\%$ on its own); DCRNN is coupled (direct bilevel produces $-0.5\%$, within noise, on PeMS04). The DCRNN positive distillation result is scope-boundary evidence: graph quality remains useful when supplied as a fixed input that bypasses the blocked optimization path.}
\label{tab:distill_cross}
\small
\begin{tabular}{lccc}
\toprule
Target backbone & Vanilla & Distilled ($A_\phi$) & $\Delta$ vs.\ vanilla \\
\midrule
MPGRU (decoupled)  & $21.081_{\pm.019}$ & $20.541_{\pm.050}$ & $-2.6\%$ \\
DCRNN (coupled)    & $20.154_{\pm.032}$ & $19.795_{\pm.014}$ & $-1.8\%$ \\
\bottomrule
\end{tabular}
\end{table}

\paragraph{PeMS07.} On PeMS07 distillation recovers only $1\%$ of the bilevel gain, consistent with the broader triangulation showing near-zero graph contribution on this dataset (Appendix~\ref{app:triangulation}).

On PeMS04 and PeMS08 distillation recovers $44$--$48\%$ on decoupled backbones (DiffConv, MPGRU), larger than the $19$--$22\%$ graph share measured by frozen-$\phi$; the difference reflects the co-adaptation magnitude on these flow datasets. Cross-backbone distillation results in Table~\ref{tab:distill_cross} use compatible decoupled backbones plus DCRNN as the scope-boundary case.

\paragraph{Negative graph shares: cross-filter incompatibility.} ChebConv shows mixed graph shares across flow datasets ($-8\%$ on PeMS04, $+18\%$ on PeMS07, $-22\%$ on PeMS08), reflecting that a DiffConv-tuned adjacency is filter-specific and does not transfer cleanly across spatial filters; the variability is a property of the cross-backbone transfer rather than of the bilevel-learned adjacency itself.

\paragraph{Cora illustrates the conceptual distinction.} Frozen-$\phi$ attributes $63\%$ of the LDS gain to the graph channel on Cora (Table~\ref{tab:decomp}), while distillation recovers approximately $0\%$ at the modal binarization threshold (Table~\ref{tab:nc_distill}, $\tau{=}0.5$). The gap is the empirical magnitude of $\theta$-$A_\phi$ co-adaptation: the bilevel inner loop builds GCN parameters jointly tuned to the evolving $A_\phi$, and re-training a fresh GCN on the binarized $A_\phi$ does not reproduce this joint solution. Continuous-weight ST distillation does not require binarization and shows smaller co-adaptation gaps (PeMS04: frozen-$\phi$ graph $22\%$ vs.\ distillation $44\%$), reflecting the same mechanism with the binarization step removed.

\paragraph{NC scope.} We do not report graph distillation in the node classification regime as a same-protocol Table~\ref{tab:distill_same} entry: LDS produces Bernoulli edge probabilities rather than continuous edge weights, so any distillation procedure requires a binarization threshold whose choice strongly affects the result. NC distillation results across thresholds are reported separately in Table~\ref{tab:nc_distill}.

\paragraph{Distillation as a method-agnostic fallback.} Distillation extends the diagnostic scope to GSL methods outside the frozen-$\phi$ control. Methods that satisfy the third precondition (positive clean-graph gain) but violate the first or second (e.g., single-level methods like IDGL, EM-based methods like GEN if they produced positive clean-graph gain) admit graph distillation as a method-agnostic complement. In our experiments, GEN and Pro-GNN do not produce positive clean-graph gain on Cora (Table~\ref{tab:method_scope}), so neither diagnostic applies to those methods on the clean Cora graph; this is a property of the methods themselves, not of the diagnostic framework.

%% file: appendix_05_triangulation.tex
\section{PeMS07 Triangulation and End-to-End Baseline}
\label{app:triangulation}

Three independent diagnostics on PeMS07 converge on near-zero graph contribution: frozen-$\phi$ recovers $101\%$ of the bilevel gain (Table~\ref{tab:decomp} in main text), distillation recovers $1\%$ (Table~\ref{tab:distill_same}), and the end-to-end reweight baseline below improves vanilla by only $0.07$ MAE versus bilevel's $1.43$ MAE. Each diagnostic addresses a different attribution question; the convergence is the empirical anchor for the inner-channel claim. Figure~\ref{fig:triangulation_pems07} summarizes the three.

\begin{figure}[h!]
\centering
\includegraphics[width=0.78\linewidth]{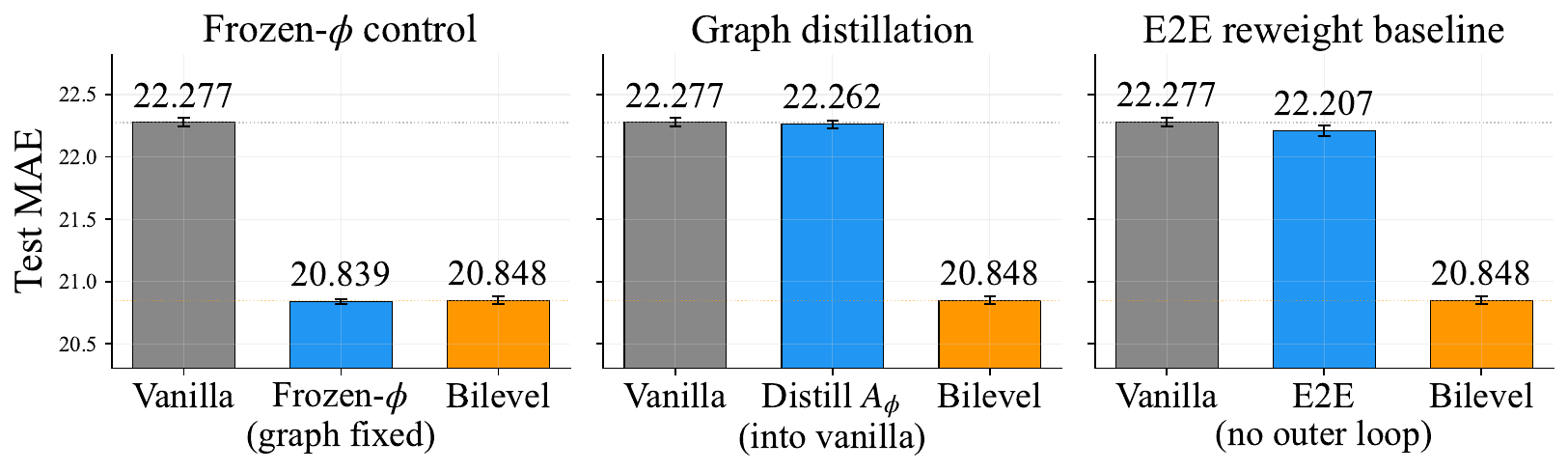}
\caption{Triangulated evidence on PeMS07. Three methodologically independent diagnostics, each answering a different attribution question, all assign near-zero weight to the graph channel: the frozen-$\phi$ control (left, $\sim 101\%$ inner share), graph distillation back into vanilla training (middle, $1\%$ recovery), and an end-to-end reweighting baseline that learns the same parameterization on the training loss without an inner-loop schedule (right, $20\times$ weaker than bilevel).}
\label{fig:triangulation_pems07}
\end{figure}

\begin{table}[h!]
\centering
\caption{End-to-end reweight baseline vs bilevel on flow datasets (DiffConv). E2E reweight uses the same parameterization $A \odot \mathrm{softmax}_{\mathrm{row}}(W_\phi)$ as bilevel but optimizes $W_\phi$ via the training loss with a single optimizer for $(\theta, W_\phi)$, eliminating the inner-loop $T$-step structure and the train/validation split between inner and outer objectives.}
\label{tab:e2e}
\small
\begin{tabular}{lccc}
\toprule
Method & PeMS04 & PeMS07 & PeMS08 \\
\midrule
Vanilla         & $20.925_{\pm.043}$ & $22.277_{\pm.038}$ & $15.621_{\pm.057}$ \\
E2E reweight    & $20.552_{\pm.056}$ & $22.207_{\pm.048}$ & $15.273_{\pm.051}$ \\
Bilevel         & $19.754_{\pm.036}$ & $20.848_{\pm.035}$ & $15.031_{\pm.071}$ \\
\midrule
$\Delta$(E2E - vanilla)     & $-0.373$ & $-0.070$ & $-0.348$ \\
$\Delta$(Bilevel - vanilla) & $-1.171$ & $-1.429$ & $-0.590$ \\
Bilevel/E2E ratio           & $3.1\times$ & $20.4\times$ & $1.7\times$ \\
\bottomrule
\end{tabular}
\end{table}

The E2E baseline is the strongest negative control for graph quality: it learns the same parameterization with the same data, but without the inner-loop multi-step optimization and without the train/validation split between objectives. On PeMS04 and PeMS08 the E2E reweight recovers $32\%$ and $59\%$ of the bilevel gain respectively; on PeMS07 it recovers only $5\%$. The $20.4\times$ bilevel-to-E2E ratio on PeMS07 is the converse of the $1.7\times$ ratio on PeMS08: PeMS07 is the regime where neither the joint train-loss gradient on the same $(\theta, \phi)$ parameterization nor the bilevel-learned graph applied as a fixed input (distillation) can recover the bilevel gain, leaving the inner-loop multi-step training mechanism as the only channel that carries it.

%% file: appendix_06_node_classification.tex
\section{Node Classification Full Results}
\label{app:nc_full}

\paragraph{NC backbone configurations follow GSLB.} We use GSLB's released per-method configurations~\citep{li2023gslb} without modification: the vanilla GCN baseline is a 2-layer \texttt{GCN} encoder (hidden $32$, row-normalized adjacency) and the LDS backbone is a 2-layer \texttt{MetaDenseGCN} (hidden $128$ on Cora / $64$ on Citeseer, symmetric $D^{-1/2}(A{+}I)D^{-1/2}$ normalization), reflecting GSLB's published per-method tuning. The frozen-$\phi$ control matches the LDS configuration exactly (same backbone class, hidden dim, normalization), differing only in whether the outer-loop graph parameter is updated. Consequently, in the NC regime the vanilla$\rightarrow$frozen-$\phi$ comparison reflects the joint effect of (i) backbone capacity, (ii) normalization, and (iii) inner-loop dynamics under GSLB's per-method convention; the frozen-$\phi\rightarrow$bilevel comparison, which isolates the graph channel cleanly under matched architecture, is unaffected by this convention. ST experiments do not exhibit this asymmetry: vanilla, frozen-$\phi$, and bilevel share the same DiffConv backbone, so the ST inner share is a clean attribution.

\paragraph{Parameterization robustness setting.} To test whether the frozen-$\phi$ decomposition extends across graph parameterizations and training regimes, we evaluate LDS~\citep{franceschi2019learning} on Cora and Citeseer with Bernoulli edge probabilities $A \sim \mathrm{Bernoulli}(\sigma(\theta_{ij}))$, full-batch inner updates, and GCN inner parameters reset per outer iteration (full-batch reset regime in Section~\ref{sec:framework-igr}). Citeseer establishes that the decomposition continues to separate channels under a discrete-edge parameterization with a different training regime: the LDS gain over vanilla GCN is $+4.25$ accuracy points (matching the original LDS paper's $+4.32$pp Citeseer gap to within $0.07$pp), the frozen-$\phi$ variant recovers $44\%$ of this gain, and all three channel tests reach $p < 10^{-3}$ (with $p_{\mathrm{total}}$ and $p_{\mathrm{graph}}$ at $p < 10^{-4}$). On Cora the total LDS gain is smaller ($+1.65$pp); the compressed bilevel headroom yields a wide bootstrap $95\%$ CI on the inner-share ratio ($[8.7, 65.1]\%$). Two independent mechanism checks corroborate the decomposition in this regime: the $T$-sweep on both datasets exhibits the predicted full-batch reset signature where frozen-$\phi$ is $T$-invariant past inner convergence (Appendix~\ref{app:nc_tsweep}, Figure~\ref{fig:tsweep}), and the corruption study on Cora (Section~\ref{sec:find-scope}, Figure~\ref{fig:corruption}) shows the decomposition continues to separate channels at every corruption level we test. Independent graph-side evidence appears in Appendices~\ref{app:nc_mechanism} and~\ref{app:nc_distillation}. Under GSLB's per-method convention, the NC inner channel reflects the joint contribution of inner-loop training dynamics and the LDS-tuned backbone; the graph channel ($56$--$63\%$ of total NC gain) is identified cleanly under matched architecture in either reading. The ST inner channel ($78$--$101\%$) is unaffected by this convention since vanilla, frozen-$\phi$, and bilevel share the same DiffConv backbone.

\begin{table}[h!]
\centering
\caption{Node classification detailed results (LDS with GCN inner). Columns $p_{\mathrm{total}}$, $p_{\mathrm{graph}}$, $p_{\mathrm{inner}}$ report paired $t$-test $p$-values for LDS vs GCN, LDS vs frozen-$\phi$, and frozen-$\phi$ vs GCN respectively. Inner share is $\Delta_{\mathrm{inner}} / (\Delta_{\mathrm{inner}} + \Delta_{\mathrm{graph}})$ with $95\%$ bootstrap percentile CI ($B{=}1{,}000$ paired resamples).}
\label{tab:nc_detail}
\resizebox{\textwidth}{!}{%
\small
\begin{tabular}{lccccccc}
\toprule
Dataset & LDS & Frozen-$\phi$ & GCN & Inner share [CI] & $p_{\mathrm{total}}$ & $p_{\mathrm{graph}}$ & $p_{\mathrm{inner}}$ \\
\midrule
Cora     & $82.75_{\pm.84}\%$ & $81.71_{\pm.55}\%$ & $81.10_{\pm.67}\%$ & $37\%$ [8.7, 65.1]    & $<0.001$ & $<0.005$ & $0.069$ \\
Citeseer & $73.53_{\pm.83}\%$ & $71.13_{\pm.54}\%$ & $69.28_{\pm.88}\%$ & $44\%$ [31.1, 54.9]   & $<0.0001$ & $<0.0001$ & $<0.001$ \\
\bottomrule
\end{tabular}%
}
\end{table}

\begin{figure}[h!]
\centering
\includegraphics[width=0.6\linewidth]{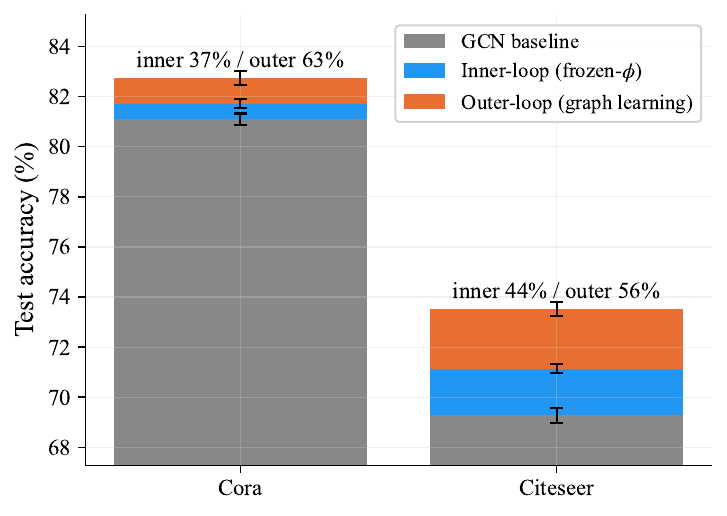}
\caption{Decomposition of LDS gain into inner and graph channels on Cora and Citeseer. Cora: total gain $+1.65$pp (inner $0.61$pp / graph $1.04$pp, $37\%$ inner share). Citeseer: total gain $+4.25$pp (inner $1.85$pp / graph $2.40$pp, $44\%$ inner share). The graph channel contributes the larger fraction on both datasets, with inner-loop training dynamics providing a substantial but minority contribution.}
\label{fig:nc_decomposition}
\end{figure}

\paragraph{Cora inner share confidence interval.} The Cora inner-share point estimate is $37\%$ with a paired $t$-test $p$-value of $0.069$ for the inner-channel test, marginal at conventional thresholds. The nonparametric bootstrap $95\%$ confidence interval is $[8.7, 65.1]\%$. The wide interval reflects the small total gain on Cora ($+1.65$ accuracy points) combined with per-seed variance. On Citeseer the total gain is $+4.25$ points (matching the original LDS paper's $+4.32$pp Citeseer gap to within $0.07$pp), the per-seed noise is proportionally smaller, and the $95\%$ bootstrap CI for the inner share is $[31.1, 54.9]\%$. The Cora point estimate should be read qualitatively rather than as a precise attribution; the evidence supports a substantial inner-channel contribution, with Citeseer providing the tighter attribution due to the larger total gain.

\section{NC Corruption Study}
\label{app:nc_corruption}

Edges are corrupted by random swap: at corruption fraction $r$, a uniformly random subset of $r \cdot |E|$ existing edges is removed, and an equal number of new edges is sampled uniformly from the complement (excluding self-loops), preserving the total edge count and the symmetric structure of the adjacency. The same corruption seed is used across all three conditions (GCN, frozen-$\phi$, LDS) at each $r$, so the GCN, frozen-$\phi$, and LDS rows in Table~\ref{tab:corruption_detail} share an identical corrupted graph at every cell; only the training procedure differs. Per-cell paired $t$-tests (LDS vs.\ frozen-$\phi$ for the graph channel, frozen-$\phi$ vs.\ GCN for the inner channel, and LDS vs.\ GCN for the total) reach $p < 0.001$ at every $r \geq 0.10$; on clean Cora the inner-channel test is marginal ($p = 0.069$) and the graph-channel test reaches $p < 0.005$.

\begin{table}[h!]
\centering
\caption{Three-way decomposition on Cora under controlled edge-swap corruption. Inner share crosses $50\%$ near $r \approx 0.50$; the inner channel saturates at $\sim 3.7$ accuracy points across $r \in \{0.50, 0.75\}$ while the graph channel scales linearly with corruption.}
\label{tab:corruption_detail}
\small
\begin{tabular}{ccccccc}
\toprule
$r$ & GCN & Frozen-$\phi$ & LDS & $\Delta_{\mathrm{inner}}$ & $\Delta_{\mathrm{graph}}$ & Inner share [CI] \\
\midrule
0.00 & $81.10_{\pm.67}\%$ & $81.71_{\pm.55}\%$ & $82.75_{\pm.84}\%$ & $+0.61$ & $+1.04$ & $37\%$ [9, 65] \\
0.10 & $77.88_{\pm.77}\%$ & $79.69_{\pm.53}\%$ & $80.25_{\pm.92}\%$ & $+1.81$ & $+0.56$ & $77\%$ [59, 100] \\
0.25 & $67.91_{\pm.74}\%$ & $71.74_{\pm.73}\%$ & $74.21_{\pm.92}\%$ & $+3.83$ & $+2.47$ & $61\%$ [53, 69] \\
0.50 & $56.56_{\pm.44}\%$ & $60.26_{\pm.98}\%$ & $64.00_{\pm.79}\%$ & $+3.70$ & $+3.74$ & $50\%$ [42, 58] \\
0.75 & $41.97_{\pm2.24}\%$ & $45.65_{\pm.98}\%$ & $51.02_{\pm1.36}\%$ & $+3.68$ & $+5.37$ & $41\%$ [30, 47] \\
\bottomrule
\end{tabular}
\end{table}

Two non-monotonic features warrant explicit comment. First, the inner-channel absolute gain is monotonic in $r$ until it saturates: it grows from $0.61$pp at $r = 0$ to $1.81$ at $r = 0.10$, then plateaus at $3.83$, $3.70$, and $3.68$ across $r \in \{0.25, 0.50, 0.75\}$ (within-plateau coefficient of variation $2.2\%$). The saturation is consistent with the graph-independent character of the implicit gradient regularization mechanism in Proposition~\ref{prop:igr}: once $T$ inner steps are sufficient to recover the inner-loop training benefit, additional graph degradation cannot enlarge the inner channel further. Second, the inner share is non-monotonic in $r$ ($37\%$, $77\%$, $61\%$, $50\%$, $41\%$) even though both absolute channels are well-behaved: at $r = 0$ the small total ($\Delta_{\mathrm{total}} = 1.65$pp) places the ratio in a high-variance regime (bootstrap CI $[8.7, 65.1]\%$); at $r = 0.10$ the inner channel jumps faster than the graph channel because the latter has just emerged above its noise floor. The qualitative picture stable across $r \geq 0.25$ is the one reported in the main text: an inner channel ceiling and a linearly scaling graph channel produce a graph-quality-dependent balance with a $50/50$ crossover near $r = 0.50$.

\section{NC T-sweep}
\label{app:nc_tsweep}

\begin{table}[h!]
\centering
\caption{LDS T-sweep on Cora and Citeseer (test accuracy, \%). The frozen-$\phi$ control is $T$-invariant by definitional construction: the LDS outer loop reinitializes the GCN inner parameters at each outer iteration, so under frozen-$\phi$ the adjacency distribution does not vary with $T$ and the inner GCN sees a constant outer-iteration count rather than a constant-batch reuse. Mini-batch reuse amplification does not apply in the LDS implementation, in contrast to the ST regime; the full-batch reset regime (full-batch with inner-parameter reset) is the operative implementation. The frozen-$\phi$ accuracy listed is the single $T$-invariant value; the $T$-varying numbers in the LDS Bilevel column reflect the outer-loop dynamics, not mini-batch reuse.}
\label{tab:nc_tsweep}
\small
\begin{tabular}{lcccccc}
\toprule
Dataset & $T = 1$ & $T = 3$ & $T = 5$ & $T = 10$ & $T = 20$ & Frozen-$\phi$ \\
\midrule
Cora     & $83.53_{\pm.56}\%$ & $82.84_{\pm.73}\%$ & $82.75_{\pm.84}\%$ & $83.96_{\pm.59}\%$ & $83.31_{\pm.74}\%$ & $81.71_{\pm.55}\%$ \\
Citeseer & $73.52_{\pm.79}\%$ & $73.09_{\pm.76}\%$ & $73.53_{\pm.83}\%$ & $74.50_{\pm.45}\%$ & $74.66_{\pm.49}\%$ & $71.13_{\pm.54}\%$ \\
\bottomrule
\end{tabular}
\end{table}

Figure~\ref{fig:tsweep}(b) in the main text visualizes the LDS T-sweep on Cora and Citeseer.

The contrast with the ST T-sweep (Table~\ref{tab:hp_sens}) is the empirical signature of the mini-batch reuse versus full-batch reset distinction discussed in Section~\ref{sec:framework-igr}. Under the LDS implementation (full-batch gradients, GCN inner reset per outer iteration), $T$ has no lever on the frozen-$\phi$ inner channel, and the LDS bilevel curve with $T$ reflects outer-loop optimization dynamics rather than mini-batch reuse amplification. Under the ST implementation (mini-batch SGD, persistent inner state across outer iterations), $T$ directly controls accumulated parameter displacement on the same batch, and the frozen-$\phi$ MAE varies systematically with $T$. The two T-sweeps therefore correspond to the two regimes of Proposition~\ref{prop:igr}, with the predicted opposing $T$-signatures.

\section{NC Graph Distillation}
\label{app:nc_distillation}

\begin{table}[h!]
\centering
\caption{Node classification graph distillation: the LDS-learned Bernoulli edge-probability matrix is binarized at threshold $\tau$ and used as a fixed adjacency for vanilla GCN training. $\tau = 0.5$ corresponds to the modal binarization (edge present if probability exceeds $0.5$); $\tau = 0.1$ and $\tau = 0.01$ correspond to denser graphs that include edges with low LDS-assigned probability.}
\label{tab:nc_distill}
\small
\begin{tabular}{lcccccc}
\toprule
Dataset & Threshold $\tau$ & GCN base & Distilled & $\Delta$ vs GCN & $\Delta$ vs LDS & LDS reference \\
\midrule
\multirow{3}{*}{Cora}
 & $0.5$  & $81.10\%$ & $81.06_{\pm.69}\%$  & $-0.04$pp & $-1.69$pp & $82.75\%$ \\
 & $0.1$  & $81.10\%$ & $76.84_{\pm13.77}\%$ & $-4.26$pp & $-5.91$pp & $82.75\%$ \\
 & $0.01$ & $81.10\%$ & $76.80_{\pm2.79}\%$  & $-4.30$pp & $-5.95$pp & $82.75\%$ \\
\midrule
\multirow{3}{*}{Citeseer}
 & $0.5$  & $69.28\%$ & $69.28_{\pm.93}\%$  & $+0.00$pp & $-4.25$pp & $73.53\%$ \\
 & $0.1$  & $69.28\%$ & $69.31_{\pm.65}\%$  & $+0.03$pp & $-4.22$pp & $73.53\%$ \\
 & $0.01$ & $69.28\%$ & $72.25_{\pm1.53}\%$ & $+2.97$pp & $-1.28$pp & $73.53\%$ \\
\bottomrule
\end{tabular}
\end{table}

The threshold-dependence of NC distillation makes the protocol qualitatively different from the continuous-weight ST distillation in Appendix~\ref{app:distillation}. On Cora the modal binarization ($\tau = 0.5$) recovers essentially zero of the LDS gain; lowering $\tau$ degrades performance because over-thresholded sparse graphs lose informative edges. On Citeseer the densest setting ($\tau = 0.01$) recovers approximately $70\%$ of the LDS gain, but at the cost of including a large number of low-probability edges. The high standard deviation at Cora $\tau{=}0.1$ ($\pm 13.77$) reflects the brittleness of the binarized graph at intermediate thresholds. NC distillation provides independent evidence that LDS carries task-relevant graph modification on Citeseer; the threshold sensitivity reflects the discrete-edge parameterization rather than the diagnostic itself, and frozen-$\phi$ remains the primary attribution tool in the NC regime.

\section{LDS Edge Probability Distributions and Dirichlet Energy}
\label{app:nc_mechanism}

\begin{figure}[h!]
\centering
\includegraphics[width=\linewidth]{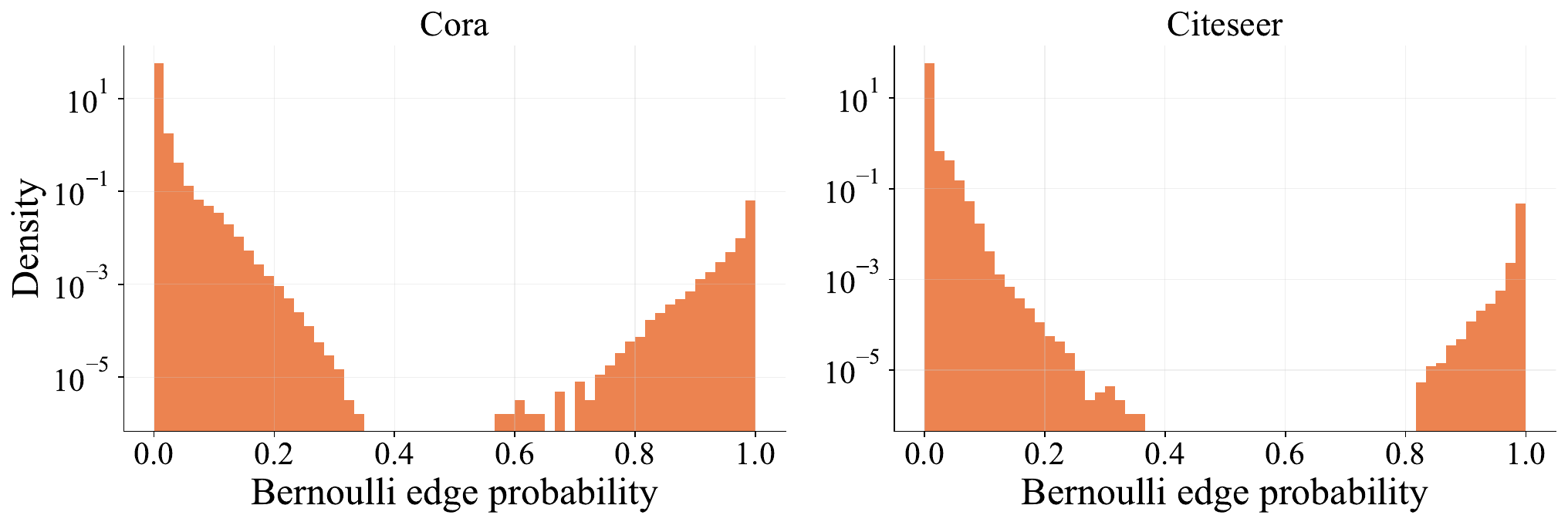}
\caption{Distribution of learned Bernoulli edge probabilities $\theta_{ij}$ from LDS at convergence. Citeseer shows a strongly bimodal distribution, evidence that LDS performs substantive graph modification. Cora concentrates more closely around the original adjacency values, consistent with the smaller total LDS gain on Cora.}
\label{fig:lds_edge_dist}
\end{figure}

The Bernoulli edge probability distributions show that the magnitude of LDS graph modification varies by an order of magnitude across the two datasets. On Citeseer, $89\%$ of all potential edge slots have $\theta_{ij} < 0.01$ and $4\%$ have $\theta_{ij} > 0.9$, producing a strongly bimodal distribution. The corresponding modal binarization (threshold $0.5$) substitutes $34\%$ of original edges and adds $7\%$ of new edges, yielding an adjacency substantively different from the citation graph. On Cora, the distribution concentrates near initial values: $72\%$ of $\theta_{ij}$ are below $0.01$ and only $1\%$ exceed $0.9$, with the modal binarization producing an adjacency nearly identical to the original. This asymmetry provides direct evidence that LDS exercises its graph-modification capacity to different degrees on the two datasets, and is consistent with the larger total LDS gain on Citeseer ($+4.25$ accuracy points) versus Cora ($+1.65$ points).

\paragraph{Dirichlet energy.} The Dirichlet energy $\mathcal{E}(X) = \mathrm{tr}(X^{\top} \tilde{L} X) / \mathrm{tr}(X^{\top} X)$, using the symmetric normalized Laplacian $\tilde{L} = I - D^{-1/2} A D^{-1/2}$, of raw node features on the initial versus LDS-learned adjacency (binarized, threshold $0.01$) increases substantially on both datasets: Cora from $0.8451$ to $0.9508 \pm 0.0016$ ($+12.5\% \pm 0.2\%$), Citeseer from $0.7897$ to $0.9530 \pm 0.0008$ ($+20.7\% \pm 0.1\%$). The increase is consistent across all seeds with negligible per-seed spread, indicating a stable direction of effect rather than a noise-driven artifact. LDS therefore trades feature-graph smoothness for task-discriminative edge placement, the opposite of an over-squashing-targeted rewiring, and consistent with the Jacobian measurement in Section~\ref{app:jacobian_cross} (Table~\ref{tab:nc_jacobian}).

\paragraph{GEN and Pro-GNN scope clarification.} Table~\ref{tab:method_scope} reports vanilla GCN, LDS, GEN, and Pro-GNN on Cora under the same protocol. Neither GEN nor Pro-GNN satisfies the method-level scope conditions of Section~\ref{sec:framework-scope}, but the underlying reasons differ. GEN's EM-based outer step replaces the adjacency with a posterior estimate at each iteration rather than modifying it incrementally from $\Ainit$; the frozen-$\phi$ control assumes that the outer parameter $\phi$ initialized at $\Ainit$ is the same object the procedure would otherwise update incrementally, and this correspondence does not hold for GEN's posterior-replacement step. Pro-GNN couples graph learning with an adversarial defense objective targeting noisy or perturbed graphs, and on the clean Cora graph it produces a negative total gain ($77.44 \pm 0.67\%$ versus vanilla GCN $81.10 \pm 0.67\%$); $\Delta_{\text{total}}$ is negative, so the inner/graph share decomposition is undefined regardless of any frozen-$\phi$ value. 
Both observations are consistent with the scope conditions in Section~\ref{sec:framework-scope} together with the implicit requirement of positive clean-graph gain for the share decomposition to be well-defined. GEN's posterior-replacement violates condition (ii) [incremental modification]; Pro-GNN's negative clean-graph gain on Cora makes $\Delta_{\mathrm{total}} \leq 0$, leaving the channel share undefined regardless of any frozen-$\phi$ value.
For methods outside this scope, the graph-distillation fallback (Appendices~\ref{app:distillation} and~\ref{app:nc_distillation}) remains applicable whenever the GSL method itself produces a positive clean-graph gain, which neither GEN nor Pro-GNN does on Cora.

\begin{table}[h!]
\centering
\caption{Method-level scope comparison on Cora. GEN and Pro-GNN are reported for completeness; neither produces a clean-graph gain that could be decomposed by the frozen-$\phi$ protocol.}
\label{tab:method_scope}
\small
\begin{tabular}{lccl}
\toprule
Method & Cora accuracy & $\Delta_{\text{total}}$ vs GCN & In scope? \\
\midrule
Vanilla GCN & $81.10_{\pm.67}\%$  & --- (baseline) & baseline \\
LDS         & $82.75_{\pm.84}\%$  & $+1.65$ pp     & yes \\
GEN         & $79.54_{\pm1.76}\%$ & $-1.56$ pp     & no (non-incremental EM step) \\
Pro-GNN     & $77.44_{\pm.67}\%$  & $-3.66$ pp     & no (negative $\Delta_{\text{total}}$) \\
\bottomrule
\end{tabular}
\end{table}

%% file: appendix_07_per_horizon.tex
\section{Per-Horizon MAE Breakdown}
\label{app:per_horizon}

We report per-horizon MAE on PeMS04, PeMS07, and PeMS08 across decoupled (ChebConv, DiffConv, MPGRU) and coupled (DCRNN) backbones, and the frozen-$\phi$ control on the primary DiffConv backbone. All values are mean $\pm$ std at the epoch with the best validation MAE. Bilevel gain on decoupled backbones increases monotonically with prediction horizon. We show below that this horizon scaling is reproduced almost entirely by the frozen-$\phi$ inner channel (per-horizon inner share $77$--$101\%$), consistent with the central inner-channel attribution of the bilevel gain in Section~\ref{sec:findings}. DCRNN shows weak horizon-dependent variation on PeMS04 and PeMS08 and uniform degradation on PeMS07, consistent with the architectural scope condition (Section~\ref{sec:framework-scope}).

\begin{table}[h!]
\centering
\caption{Per-horizon MAE on PeMS04. Bilevel gain increases monotonically with prediction horizon for decoupled backbones; DCRNN shows minimal horizon-dependent variation. The percentage in parentheses on each bilevel row is the relative improvement at $h = 12$.}
\label{tab:per_horizon_pems04}
\small
\begin{tabular}{llcccc}
\toprule
Backbone & Method & $h=3$ & $h=6$ & $h=9$ & $h=12$ \\
\midrule
\multirow{2}{*}{ChebConv}
 & Vanilla & $19.028_{\pm.022}$ & $20.570_{\pm.038}$ & $22.043_{\pm.045}$ & $23.726_{\pm.058}$ \\
 & Bilevel & $18.506_{\pm.018}$ & $19.689_{\pm.037}$ & $20.805_{\pm.051}$ & $22.175_{\pm.057}$ ($-6.5\%$) \\
\midrule
\multirow{3}{*}{DiffConv}
 & Vanilla & $19.108_{\pm.014}$ & $20.744_{\pm.026}$ & $22.323_{\pm.045}$ & $24.090_{\pm.063}$ \\
 & Frozen-$\phi$ & $18.614_{\pm.017}$ & $19.881_{\pm.032}$ & $21.076_{\pm.038}$ & $22.472_{\pm.050}$ \\
 & Bilevel & $18.488_{\pm.013}$ & $19.649_{\pm.042}$ & $20.730_{\pm.048}$ & $21.976_{\pm.058}$ ($-8.8\%$) \\
\midrule
\multirow{2}{*}{MPGRU}
 & Vanilla & $19.223_{\pm.009}$ & $20.885_{\pm.017}$ & $22.504_{\pm.034}$ & $24.293_{\pm.049}$ \\
 & Bilevel & $18.570_{\pm.030}$ & $19.713_{\pm.043}$ & $20.753_{\pm.048}$ & $22.003_{\pm.050}$ ($-9.4\%$) \\
\midrule
\multirow{2}{*}{DCRNN}
 & Vanilla & $18.766_{\pm.040}$ & $20.001_{\pm.036}$ & $21.165_{\pm.043}$ & $22.639_{\pm.026}$ \\
 & Bilevel & $18.874_{\pm.095}$ & $19.970_{\pm.097}$ & $20.902_{\pm.157}$ & $22.110_{\pm.096}$ ($-2.3\%$) \\
\bottomrule
\end{tabular}
\end{table}

\begin{table}[h!]
\centering
\caption{Per-horizon MAE on PeMS07. DCRNN degrades at every horizon, while decoupled backbones gain more at longer horizons. ChebConv on PeMS07 has a high vanilla MAE due to a known under-fitting regime of the Chebyshev filter on this graph (overall vanilla $32.186_{\pm1.090}$); the relative gain pattern still scales with horizon.}
\label{tab:per_horizon_pems07}
\small
\begin{tabular}{llcccc}
\toprule
Backbone & Method & $h=3$ & $h=6$ & $h=9$ & $h=12$ \\
\midrule
\multirow{2}{*}{ChebConv}
 & Vanilla & $29.729_{\pm1.115}$ & $32.014_{\pm1.105}$ & $34.093_{\pm1.058}$ & $36.406_{\pm1.022}$ \\
 & Bilevel & $27.257_{\pm1.258}$ & $29.234_{\pm1.224}$ & $31.122_{\pm1.183}$ & $33.227_{\pm1.129}$ ($-8.7\%$) \\
\midrule
\multirow{3}{*}{DiffConv}
 & Vanilla & $20.177_{\pm.033}$ & $22.142_{\pm.040}$ & $23.939_{\pm.043}$ & $25.966_{\pm.055}$ \\
 & Frozen-$\phi$ & $19.208_{\pm.028}$ & $20.752_{\pm.017}$ & $22.147_{\pm.019}$ & $23.727_{\pm.020}$ \\
 & Bilevel & $19.212_{\pm.025}$ & $20.767_{\pm.036}$ & $22.160_{\pm.040}$ & $23.744_{\pm.055}$ ($-8.6\%$) \\
\midrule
\multirow{2}{*}{MPGRU}
 & Vanilla & $20.284_{\pm.023}$ & $22.284_{\pm.033}$ & $24.106_{\pm.058}$ & $26.163_{\pm.071}$ \\
 & Bilevel & $19.288_{\pm.031}$ & $20.829_{\pm.039}$ & $22.186_{\pm.046}$ & $23.787_{\pm.054}$ ($-9.1\%$) \\
\midrule
\multirow{2}{*}{DCRNN}
 & Vanilla & $19.767_{\pm.040}$ & $21.179_{\pm.053}$ & $22.384_{\pm.061}$ & $23.860_{\pm.075}$ \\
 & Bilevel & $20.012_{\pm.049}$ & $21.617_{\pm.088}$ & $22.958_{\pm.041}$ & $24.410_{\pm.076}$ ($+2.3\%$) \\
\bottomrule
\end{tabular}
\end{table}

\begin{table}[h!]
\centering
\caption{Per-horizon MAE on PeMS08. Decoupled backbones gain monotonically with prediction horizon; DCRNN bilevel is essentially flat ($-0.3\%$ at $h = 12$).}
\label{tab:per_horizon_pems08}
\small
\begin{tabular}{llcccc}
\toprule
Backbone & Method & $h=3$ & $h=6$ & $h=9$ & $h=12$ \\
\midrule
\multirow{2}{*}{ChebConv}
 & Vanilla & $14.408_{\pm.015}$ & $15.365_{\pm.021}$ & $16.213_{\pm.032}$ & $17.274_{\pm.052}$ \\
 & Bilevel & $13.980_{\pm.027}$ & $14.823_{\pm.025}$ & $15.603_{\pm.026}$ & $16.569_{\pm.029}$ ($-4.1\%$) \\
\midrule
\multirow{3}{*}{DiffConv}
 & Vanilla & $14.484_{\pm.026}$ & $15.529_{\pm.056}$ & $16.491_{\pm.076}$ & $17.653_{\pm.107}$ \\
 & Frozen-$\phi$ & $14.159_{\pm.066}$ & $15.068_{\pm.073}$ & $15.914_{\pm.093}$ & $16.889_{\pm.071}$ \\
 & Bilevel & $14.082_{\pm.058}$ & $14.970_{\pm.073}$ & $15.759_{\pm.087}$ & $16.701_{\pm.085}$ ($-5.4\%$) \\
\midrule
\multirow{2}{*}{MPGRU}
 & Vanilla & $14.511_{\pm.013}$ & $15.558_{\pm.011}$ & $16.545_{\pm.023}$ & $17.712_{\pm.032}$ \\
 & Bilevel & $14.092_{\pm.050}$ & $15.015_{\pm.069}$ & $15.781_{\pm.073}$ & $16.728_{\pm.072}$ ($-5.6\%$) \\
\midrule
\multirow{2}{*}{DCRNN}
 & Vanilla & $14.306_{\pm.025}$ & $15.186_{\pm.033}$ & $15.964_{\pm.048}$ & $17.026_{\pm.072}$ \\
 & Bilevel & $14.425_{\pm.043}$ & $15.308_{\pm.084}$ & $16.021_{\pm.084}$ & $16.977_{\pm.049}$ ($-0.3\%$) \\
\bottomrule
\end{tabular}
\end{table}

\paragraph{Inner channel reproduces horizon-dependent error growth.} On all three flow datasets the frozen-$\phi$ DiffConv control exhibits monotonic horizon scaling closely paralleling vanilla and bilevel ($h{=}12$ minus $h{=}3$ gaps of $3.86$, $4.52$, and $2.73$ MAE on PeMS04/PeMS07/PeMS08). Computing the inner-channel share per horizon, $\Delta_{\mathrm{inner}}/\Delta_{\mathrm{total}} = (\mathcal{M}_\mathrm{Vanilla} - \mathcal{M}_{\mathrm{Frozen}\text{-}\phi}) / (\mathcal{M}_\mathrm{Vanilla} - \mathcal{M}_\mathrm{Bilevel})$, the share remains approximately invariant across horizons (PeMS04: $77$--$80\%$; PeMS07: $100$--$101\%$; PeMS08: $79$--$82\%$), matching the horizon-aggregated values reported in Table~\ref{tab:decomp}. The horizon-dependent error growth is therefore reproduced almost entirely by the frozen-$\phi$ inner channel without graph modification, consistent with the inner-channel-as-trajectory-regularization interpretation of Proposition~\ref{prop:igr}: the implicit regularization accumulated across the $T$-step inner schedule scales with the difficulty of long-horizon prediction. The complementary DCRNN result (weak horizon scaling on PeMS04, near-flat on PeMS08, uniform degradation on PeMS07) reflects the architectural scope boundary of Section~\ref{sec:framework-scope}: the coupled backbone breaks the factored Jacobian of \eqref{eq:jacobian} and does not benefit from improved spatial graph quality at any horizon.

%% file: appendix_08_R.tex
\section{Three-precondition framework}
\label{app:framework}

This appendix develops the three-precondition framework introduced in Section~\ref{sec:find-scope}, which predicts when bilevel graph structure learning produces forecasting benefit. Section~\ref{app:framework_def} states the preconditions formally. Section~\ref{app:airquality_anatomy} documents the AirQuality (AQ-437) adjacency anatomy that motivates precondition~(iii). Section~\ref{app:precondition_table} classifies the six spatio-temporal datasets by the framework. Section~\ref{app:mod_magnitude} reports the modification-magnitude figure migrated from the main text, and Section~\ref{app:flow_vs_speed} consolidates the flow-versus-speed signal-type analysis.

\subsection{Formal statement of the three preconditions}
\label{app:framework_def}

Let $G = (V, E, A_{\text{init}})$ denote the predefined adjacency available to a spatio-temporal forecasting model on $N$ sensors, and let the underlying data generator induce a true relational structure between sensors. Let $\Phi$ denote the policy class for the bilevel adjacency parameterization (here, a static undirected reweighting of $A_{\text{init}}$).

\begin{definition}[Three preconditions for bilevel-GSL benefit]
\label{def:framework}
We say bilevel GSL is \emph{plausibly beneficial} on $G$ for the prediction task if all three of the following hold.

\textbf{(i) Structural slack.} The predefined adjacency $A_{\text{init}}$ is suboptimal for the task-relevant relational structure. Operationally, structural slack is small when distance-based or domain-prescribed edges already align with the directions of large prediction error gradients, and large when task-coupled sensor pairs experience high commute time or low Dirichlet alignment under $A_{\text{init}}$~\citep{topping2022understanding,digiovanni2023oversquashing,black2023understanding}.

\textbf{(ii) Loss identifiability.} The prediction loss has informative gradients with respect to graph parameters: $\| \nabla_\phi \Lval(\theta^*(\phi), A_\phi) \| > \tau$ for some task-meaningful threshold $\tau$, when evaluated near $A_{\text{init}}$. Identifiability fails when per-node local effects (sensor-specific embeddings, time-of-day patterns, day-of-week patterns) absorb the spatially-extractable variance, leaving the loss flat in the structural direction~\citep{cini2023taming,manenti2025learning}.

\textbf{(iii) Parametrization expressiveness.} The static undirected adjacency class $\{A_\phi : \phi \in \Phi\}$ can in principle represent the true dependency. Expressiveness fails when the true dependency is directional, time-varying, or supported on connectivity that is structurally absent from $A_{\text{init}}$ in a way no reweighting can recover.
\end{definition}

The three preconditions are individually necessary for a non-trivial forecasting gain from a bilevel update of the static undirected adjacency. They are stated as conditions on the adjacency and the loss, not on the bilevel optimization procedure itself: a procedure may fail to find the gain even when (i)--(iii) hold, but the gain cannot exist when any one fails. Empirically, joint satisfaction of (i)--(iii) predicts a positive gain on all flow benchmarks (Table~\ref{tab:preconditions}).

\paragraph{Mapping to the six benchmark regimes.} Flow-signal datasets (PeMS04/07/08) satisfy all three preconditions: distance-based highway adjacency is incomplete with respect to flow conservation, the squared-error loss on traffic counts is informative with respect to inter-sensor structural changes, and a static undirected reweighted graph is an adequate representation. Speed-signal datasets (METR-LA, PeMS-BAY) fail preconditions (i)--(ii): the distance graph is well-aligned with the kinematic-wave physics that governs speed propagation~\citep{lighthill1955kinematic} (so (i) fails), and per-node embeddings absorb the residual variance in modern STGNN backbones (so (ii) fails by the local-effects mechanism of~\citet{cini2023taming}). AirQuality (AQ-437) fails precondition (iii) for the structural reason analyzed in Section~\ref{app:airquality_anatomy}.

\paragraph{Distinguishing data limits from optimization limits.} Failures of bilevel GSL on clean benchmarks admit two distinct diagnoses: (a) the predefined adjacency carries little exploitable slack and no procedure can recover gains beyond the inner channel (data limit), or (b) sufficient slack exists but the gradient-based outer loop fails to find it (optimization limit). The decomposition toolkit distinguishes them. When frozen-$\phi$, distillation, and end-to-end reweighting all converge on near-zero graph contribution (PeMS07: 101\%, 1\%, 5\% respectively for the three diagnostics), the bottleneck is data; when they diverge (bilevel low but distillation high), the bottleneck is optimization. The corruption study on Cora provides the converse evidence: injected slack is captured approximately linearly by the bilevel procedure (graph channel grows from 0.56pp at $r{=}0.10$ to 5.37pp at $r{=}0.75$), demonstrating that the optimizer is not the binding constraint when slack exists. We therefore interpret the small graph channel on clean Cora and the near-zero graph channel on PeMS07 as data-limit regimes rather than optimizer failures.
\subsection{AirQuality adjacency anatomy}
\label{app:airquality_anatomy}

The AirQuality benchmark used throughout the spatio-temporal GSL literature is the 437-station network of~\citet{cini2022filling}, publicly distributed~\citep{cini2022tsl}. Its standard adjacency is constructed by the \texttt{distance} method:
\begin{equation}
\label{eq:aq_adj}
A^{\text{AQ-437}}_{\text{init},ij} = \mathbb{1}\!\left[\exp\!\left(-d_{ij}^2 / \theta^2\right) \geq \tau\right] \cdot \exp\!\left(-d_{ij}^2 / \theta^2\right),
\end{equation}
with $d_{ij}$ a great-circle distance in kilometers, $\tau = 0.1$ a hard threshold, and the bandwidth $\theta$ chosen in the reference implementation as
\begin{equation}
\label{eq:aq_theta}
\theta = \mathrm{std}\!\left(\{d_{ij}\}_{i, j \in \mathrm{AQ\text{-}36}}\right),
\end{equation}
where AQ-36 is the 36-station Beijing-only subset that predates the 437-station release. The standard bandwidth value computed by this rule is $\theta \approx 26.1$ km, fitted to Beijing intra-city distances and then applied to the full 437-station graph spanning $43$ cities across China.

\paragraph{Block-diagonal structure.} The standard adjacency contains 2{,}699 edges, mean degree $12.4$, and \textbf{33 connected components}: of these, 14 are isolated singletons and the remaining 19 form urban clusters that under single-linkage at 80\,km collapse to 13 distinct cities. Every one of the 2{,}699 retained edges connects two stations \emph{within the same urban cluster}; \textbf{the adjacency contains zero inter-city edges}. The intra-city weight median is $0.47$ and the inter-city weight is exactly zero, making the AQ-437 adjacency \emph{strictly} block-diagonal under the canonical bandwidth choice. The mechanism is that intra-Beijing pairwise distances lie around $20$ to $30$\,km (matching $\theta = 26.1$\,km), while inter-city distances exceed several hundred kilometers; the ratio $d^2/\theta^2$ for inter-city pairs reaches values of several hundred, driving the Gaussian kernel below the $\tau = 0.1$ threshold for every inter-city pair.

\paragraph{Reasons no static undirected reweighting can repair this.} A bilevel reweighting policy $A_\phi = A_{\text{init}} \odot \pi_\phi$ scales existing edges; it cannot create edges that the threshold operation~(\ref{eq:aq_adj}) excluded. The 33 components are therefore frozen in the policy class, and any propagation between cities requires either edge addition (outside the $\Phi$ used in the bilevel ST setting) or non-static parameterization (outside the precondition (iii) class). This is the structural reason precondition (iii) fails on AQ-437 and bilevel produces no forecasting benefit, despite producing the largest median modification magnitude across the six benchmarks (Section~\ref{app:mod_magnitude}).

\paragraph{Bandwidth ablation.} To confirm that the failure is attributable to the bandwidth scale rather than to the kernel form or threshold, we recompute the adjacency under five alternative bandwidth choices and report the resulting connectivity in Table~\ref{tab:aq_bandwidth}. Bandwidths fitted to the AQ-36 subset or to per-city distances (both at the $\sim$20\,km scale of intra-city station spacing) yield strictly block-diagonal adjacencies with zero inter-city edges; bandwidths fitted to the full 437-station distance distribution or to global percentiles ($\geq 200$\,km) yield connected adjacencies in which the block structure dissolves.

\begin{table}[h!]
\centering
\caption{AirQuality adjacency under alternative bandwidth choices, all with the standard threshold $\tau = 0.1$. Bandwidths at the intra-city scale (canonical and per-city std) produce strictly block-diagonal adjacencies with zero inter-city edges; bandwidths at the inter-city scale yield connected graphs. The failure is therefore tied to bandwidth scale, not to the kernel form.}
\label{tab:aq_bandwidth}
\small
\begin{tabular}{lcccccc}
\toprule
Bandwidth choice & $\theta$\,(km) & \#comp & \#iso & mean deg & \#edges & \#inter-city \\
\midrule
$\theta = \mathrm{std}(d \mid \mathrm{AQ\text{-}36})$ \quad (canonical) & 26.1   & 33 & 14 & 12.4  & 2{,}699  & \textbf{0} \\
$\theta = $ per-city std                                                       & 20.9   & 44 & 20 &  9.1  & 1{,}999  & \textbf{0} \\
$\theta = $ percentile-based ($p = 25$)                                        & 202.8  &  2 &  0 & 173.7 & 37{,}961 & 7{,}877 \\
$\theta = $ percentile-based ($p = 50$)                                        & 422.3  &  2 &  0 & 231.6 & 50{,}599 & 15{,}815 \\
$\theta = \mathrm{std}(d \mid \mathrm{AQ\text{-}437})$                         & 795.9  &  2 &  0 & 231.6 & 50{,}604 & 15{,}820 \\
$\theta = $ percentile-based ($p = 95$)                                        & 1989.1 &  1 &  0 & 436.0 & 95{,}266 & 60{,}482 \\
\bottomrule
\end{tabular}
\end{table}

\paragraph{Visual evidence.} Figure~\ref{fig:aq_block} reorders the 437 rows and columns of the standard adjacency by city assignment. The block-diagonal structure becomes visually direct: solid squares appear on the diagonal corresponding to intra-city cliques, and off-diagonal entries are uniformly zero.

\begin{figure}[h!]
\centering
\includegraphics[width=0.5\linewidth]{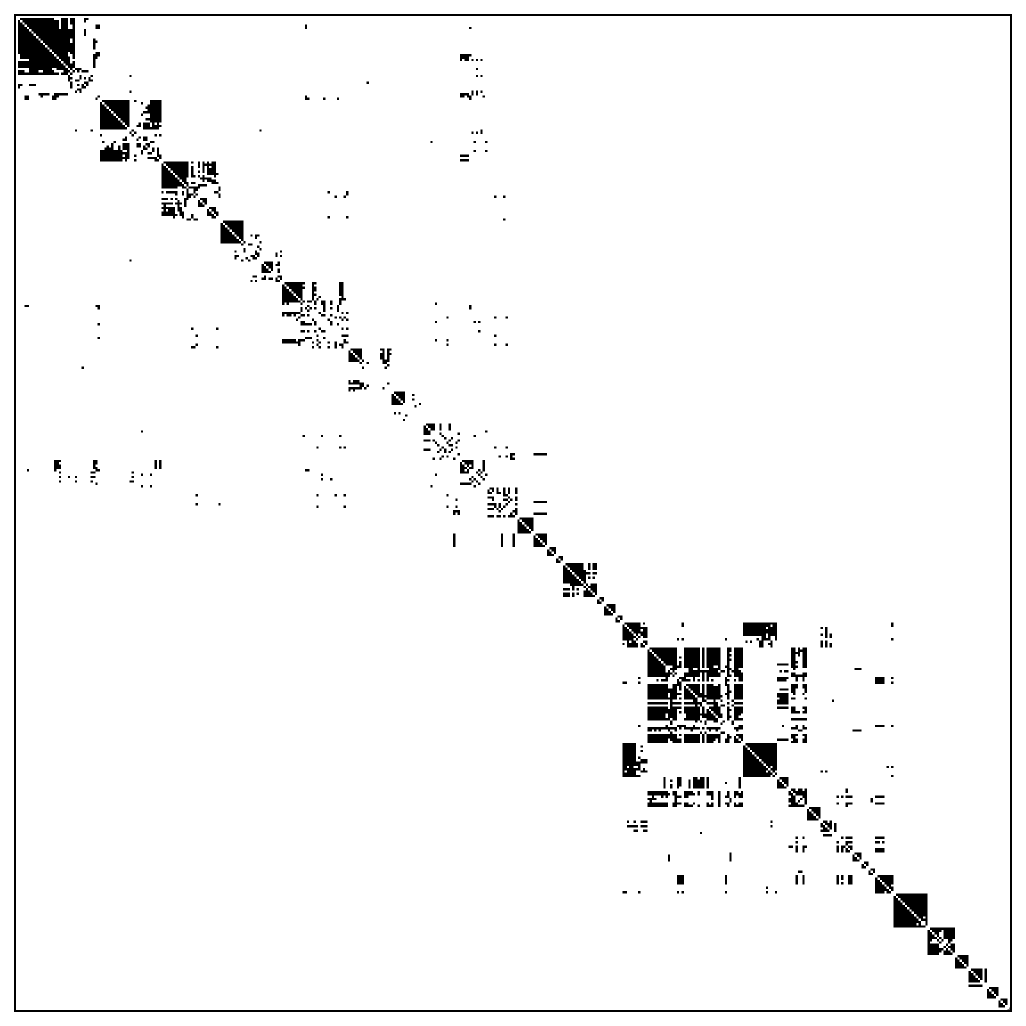}
\caption{AQ-437 adjacency reordered by city assignment (single-linkage at 80\,km, 13 distinct cities). City blocks appear as solid squares on the diagonal; off-diagonal entries are uniformly zero. The figure visualizes the structural failure mode of precondition (iii) on this benchmark.}
\label{fig:aq_block}
\end{figure}

\paragraph{Relation to prior acknowledgments.} The closest prior acknowledgment is~\citet{cini2022filling}, who note in their imputation evaluation that GRIN ``consistently outperforms BRITS in imputing missing values also for sensors corresponding to isolated (disconnected) nodes, i.e., nodes corresponding to stations more than 40\,km away from any other station.'' The disconnection is acknowledged but not quantified, named as block structure, traced to the bandwidth scale, or linked to the failure of bilevel rewiring. \citet{cini2023sparse} run AQ experiments per-city rather than on the joint 437-station graph, implicitly conceding the structural issue without articulating it. To our knowledge the precise diagnosis above has not appeared in print.

\subsection{Dataset classification by the framework}
\label{app:precondition_table}

Table~\ref{tab:preconditions} reports the per-dataset classification together with the predicted and observed bilevel effect; the framework correctly predicts the sign and qualitative magnitude on all six benchmarks. Structural slack and identifiability arguments are summarized in Section~\ref{app:framework_def} and the parametrization-expressiveness failure on AirQuality is documented in Section~\ref{app:airquality_anatomy}. The framework is dataset-level. Architecture-level scope conditions (coupled DCRNN, adaptive AGCRN/GraphWaveNet) are treated separately in Section~\ref{sec:framework-scope} and Appendix~\ref{app:backbone_scope}; for AirQuality, the failure already occurs at the dataset level under the standard backbone, so the architectural axis is not the binding constraint.

\begin{table}[h!]
\centering
\caption{Three-precondition classification of the six ST benchmarks (DiffConv backbone). \emph{Predicted}: \emph{Large gain} when all three hold, \emph{Marginal/Harm} when (i)--(ii) fail, \emph{Null} when (iii) fails. \emph{Observed}: bilevel $\Delta$ MAE\% from Table~\ref{tab:survey}.}
\label{tab:preconditions}
\small
\begin{tabular}{lcccll}
\toprule
Dataset & (i) Slack & (ii) Identifiable & (iii) Expressible & Predicted & Observed \\
\midrule
PeMS04     & Yes & Yes & Yes & Large gain    & $-5.6\%$ \\
PeMS07     & Yes & Yes & Yes & Large gain    & $-6.4\%$ \\
PeMS08     & Yes & Yes & Yes & Large gain    & $-3.8\%$ \\
METR-LA    & No  & No  & Yes & Harm          & $+2.5\%$ \\
PeMS-BAY   & No  & No  & Yes & Marginal      & $-0.7\%$ \\
AirQuality & --  & --  & No  & Null          & $+0.1\%$ \\
\bottomrule
\end{tabular}
\end{table}

\subsection{Modification magnitude versus forecasting improvement}
\label{app:mod_magnitude}

Figure~\ref{fig:mod_magnitude_app} provides the per-dataset breakdown summarized in the main text. Datasets where bilevel improves performance (PeMS04/07/08) exhibit median $|\Delta w|$ on existing edges approximately three times smaller than datasets where bilevel is neutral or harmful (PeMS-BAY, AirQuality, METR-LA).

\begin{figure}[h!]
\centering
\includegraphics[width=0.7\linewidth]{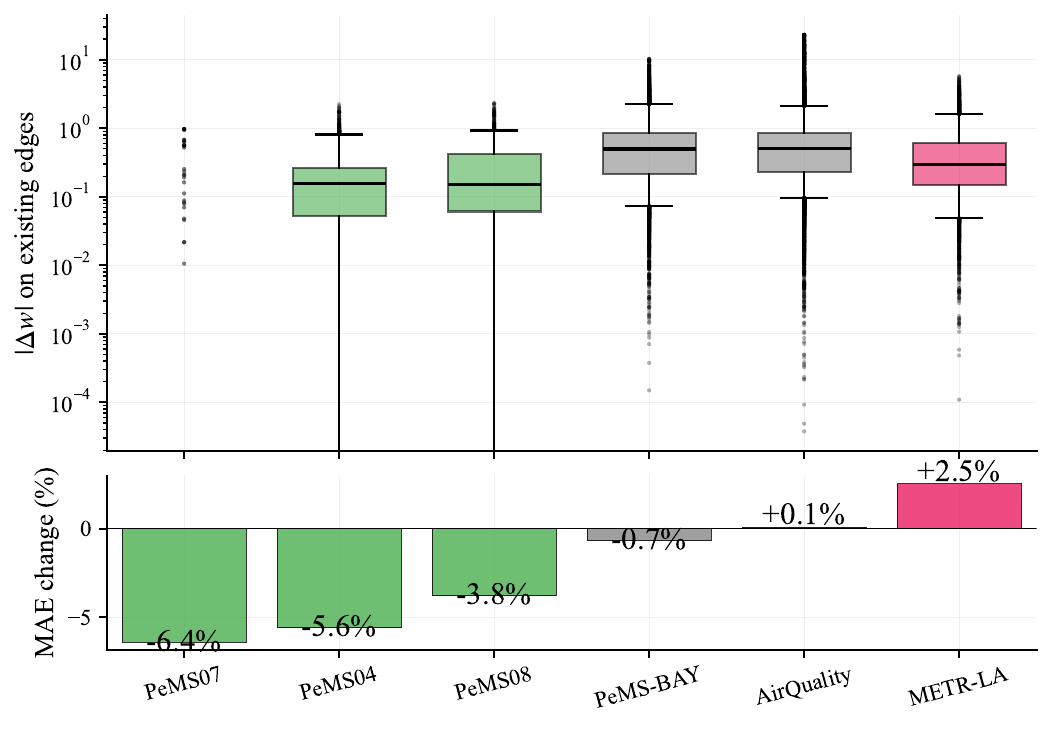}
\caption{Modification magnitude versus forecasting improvement across six ST datasets. Top: distribution of $|\Delta w|$ on existing edges (log scale). Bottom: $\Delta$ MAE\%. Spearman $\rho = -0.77$ ($p = 0.072$). Datasets where bilevel improves performance (PeMS04/07/08, green) exhibit the smallest median modifications; PeMS07 (median $|\Delta w| \approx 0$) achieves the largest improvement.}
\label{fig:mod_magnitude_app}
\end{figure}

\paragraph{Connection to the framework.} The anti-correlation aligns with the precondition diagnosis: on flow datasets where preconditions (i)--(iii) hold, bilevel applies selective minimal intervention to a graph already close to the task-relevant structure. On AirQuality where (iii) fails, bilevel applies aggressive modification ($|\Delta w|$ outliers reach $\sim 25$) that does not translate to forecasting benefit, consistent with policy-class limits rather than insufficient modification effort. This is the structural counterpart of precondition (iii): a parametrization that cannot represent the true dependency cannot produce the gain regardless of the modification budget the optimizer expends.

\subsection{Flow versus speed signal pattern}
\label{app:flow_vs_speed}

The flow/speed split observed across the four PeMS datasets aligns with preconditions (i)--(ii) in the framework. Flow-signal sensors record vehicle counts that obey approximate flow conservation across road segments: a perturbation at one sensor propagates as a wave whose direction and amplitude depend on inter-sensor connectivity beyond geometric distance. The static distance-based adjacency carries non-trivial structural slack with respect to this connectivity (precondition (i) holds), and the squared-error loss is informative with respect to inter-sensor structural changes because flow signals couple across sensors more strongly than per-node temporal patterns can absorb (precondition (ii) holds).

Speed-signal sensors record per-segment vehicle velocity, which under the Lighthill-Whitham-Richards (LWR) kinematic-wave model~\citep{lighthill1955kinematic} propagates locally along the road network with shockwaves whose speed and direction are determined by physical road geometry. The distance graph is therefore well-aligned with the dynamics (precondition (i) fails: structural slack is small), and per-node temporal patterns (rush-hour, weekday/weekend, holiday) account for the dominant fraction of the observed variance in modern STGNN backbones~\citep{cini2023taming} (precondition (ii) fails: the loss is flat in the structural direction). The flow/speed boundary is therefore not signal type per se but the alignment of the distance graph with the underlying transport physics, mediated by the share of variance absorbed by sensor-local effects.

\paragraph{Empirical confirmation.} On flow datasets (PeMS04/07/08), bilevel achieves between $-3.8\%$ and $-6.4\%$ MAE improvement with complete distributional separation across seeds. On speed datasets (METR-LA, PeMS-BAY), bilevel produces $-0.7\%$ and $+2.5\%$ respectively, and the frozen-$\phi$ control on METR-LA is also worse than vanilla, indicating that the inner channel itself fails on speed signals rather than being canceled by a harmful outer update. This is the empirical signature predicted by the framework when preconditions (i)--(ii) fail: neither the inner channel (no structural information to extract) nor the outer channel (no informative gradient) produces benefit.

\subsection{Practitioner decision flow}
\label{app:decision_tree}

Figure~\ref{fig:decision_tree} summarizes the resulting recommendation for practitioners as a decision flow over the three preconditions of Appendix~\ref{app:framework_def} together with the method-level scope conditions of Section~\ref{sec:framework-scope}.

\begin{figure}[h!]
\centering
\includegraphics[width=0.88\textwidth]{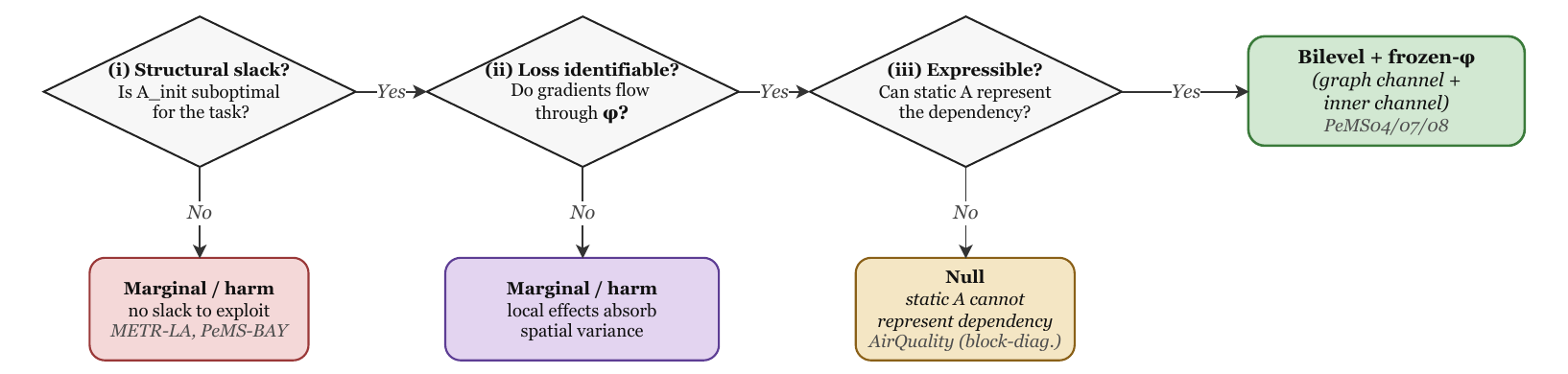}
\caption{Decision flow for attributing gains in bilevel GSL. Diamonds: decision nodes; boxes: recommended actions. Four regimes: frozen-$\phi$ + bilevel (decoupled, structurally suboptimal graph; primary); graph distillation (coupled or non-incremental); rewiring inapplicable (no external adjacency); vanilla suffices (signal aligned with graph).}
\label{fig:decision_tree}
\end{figure}

%% file: appendix_09_efficiency_repro.tex
\section{Efficiency}
\label{app:efficiency}

Bilevel incurs approximately $15\times$ training overhead due to $T$ inner steps per outer step ($4.5$h vs vanilla $0.3$h on PeMS04 DiffConv; Table~\ref{tab:hp_compute}), though inference cost is unchanged. The frozen-$\phi$ control achieves $78$--$101\%$ of bilevel's gain on flow datasets at approximately the same training cost as bilevel ($3.8$h, the $T$-step inner loop dominates). When graph-signal alignment is already high (speed datasets), practitioners can adopt the $T = 10$ schedule alone (Section~\ref{sec:discussion}).

\section{Reproducibility Details}
\label{app:reproducibility}

\paragraph{Hardware.} Experiments used L40S (48GB).

\paragraph{Software.} PyTorch 2.1, PyTorch Geometric, Torch Spatiotemporal~\citep{cini2022tsl}, PyTorch Lightning with manual optimization for ST bilevel; the GSLB fork~\citep{li2023gslb} for NC. All runs fix random seeds and enable deterministic CUDA operations. Code, configuration files, learned adjacencies, and experiment logs will be released upon acceptance. We additionally provide a self-contained frozen-$\phi$ wrapper compatible with the GSLB benchmarking suite~\citep{li2023gslb}, enabling practitioners to run the inner/graph decomposition on any LDS-compatible bilevel GSL method with a single configuration flag.

\paragraph{Statistical reporting.} ST and NC error bars are standard deviations across seeds. We report ST regression results to three-decimal test MAE and NC classification results to two-decimal test accuracy. NC channel significance is assessed by paired two-sided $t$-tests across seed indices. Inner-share confidence intervals are 95\% bootstrap percentile intervals ($B = 1{,}000$ paired resamples). For ST, effect sizes are large relative to per-seed noise (e.g., PeMS04 bilevel improvement of $1.17$ MAE versus per-seed std $\leq 0.08$), so we report mean$\pm$std without further significance tests.

\paragraph{Random seeds.} ST multi-seed experiments use $\{42, 123, 456, 789, 1024\}$ (5 seeds, following~\citep{cini2022tsl}). NC multi-seed experiments use $\{42, 123, 456, 789, 1024, 2025, 3047, 4096, 5555, 7777\}$ (10 seeds), doubling the 5-seed protocol of the original LDS paper~\citep{franceschi2019learning} to obtain tighter confidence intervals on the smaller absolute gains characteristic of the homophilous-NC regime. The first five NC seeds match the ST set; the remaining five are arbitrary fixed integers selected prior to running any experiments. Seeds are fixed across all configurations. We report all 10 NC seeds without omission, including any configurations where bilevel underperforms vanilla on individual seeds.

\FloatBarrier
\begin{table}[h!]
\centering
\caption{Hyperparameters: ST regime. Defaults apply to all spatio-temporal datasets unless otherwise noted; deviations for METR-LA and AirQuality are listed in the rightmost notes column.}
\label{tab:hp_st}
\small
\begin{tabular}{llp{6.5cm}}
\toprule
Category & Setting & Value (default / deviations) \\
\midrule
\multicolumn{3}{l}{\textit{Optimization}} \\
Optimizer & --- & Adam \\
Learning rate (inner) $\eta_\theta$ & --- & $1\times10^{-3}$ \\
Outer LR $\eta_\phi$ & --- & $1\times10^{-3}$ \\
LR schedule & --- & CosineAnnealingLR ($T_{\max}{=}$epochs, $\eta_{\min}{=}10^{-6}$) \\
Weight decay & --- & $0$ \\
Gradient clipping & --- & $5.0$ \\
Batch size & --- & $64$ \\
Epochs & total / warmup $E_w$ & $100$ / $10$ \\
Early stopping & patience (epochs) & $30$ \\
\midrule
\multicolumn{3}{l}{\textit{Architecture (decoupled backbone)}} \\
Spatial filter & DiffConv hops $K$ & $2$, bidirectional (default configuration) \\
Temporal module & TCN kernel / dilation & $3$ / $1$ (exponential, GeLU) \\
Hidden dim & --- & $64$ \\
Layers & spatial / temporal & $1$ / $2$ \\
Dropout & --- & $0$ \\
Input / output horizon & steps & $12 / 12$; AirQuality $24/24$ \\
\midrule
\multicolumn{3}{l}{\textit{Bilevel}} \\
Method & --- & First-order \citep{liu2019darts} \\
Inner steps $T$ & --- & $10$; METR-LA deviation $T = 1$ \\
Outer policy & parameterization & $A \odot \mathrm{softmax}_{\mathrm{row}}(W_\phi)$ \\
\midrule
\multicolumn{3}{l}{\textit{Static rewiring baselines}} \\
\multicolumn{3}{l}{FoSR \citep{karhadkar2023fosr}: \texttt{num\_iterations=20}, \texttt{initial\_power\_iters=10}} \\
\multicolumn{3}{l}{BORF \citep{nguyen2023revisiting}: \texttt{loops=10}, \texttt{batch\_add=4}, \texttt{batch\_remove=2}} \\
\multicolumn{3}{l}{SDRF \citep{topping2022understanding}: \texttt{num\_iterations=20}, curvature=BFC, \texttt{remove\_edges=true}, \texttt{removal\_bound=0.5}, $\tau{=}1$} \\
\multicolumn{3}{l}{GTR \citep{black2023understanding}: \texttt{num\_edges=20}} \\
\bottomrule
\end{tabular}
\\[0.4em]
\footnotesize Defaults shown for the DiffConv backbone; other backbones (ChebConv, MPGRU, DCRNN, AGCRN, GraphWaveNet) inherit the shared optimization and training settings and use their published architectural hyperparameters.
\end{table}

\FloatBarrier
\begin{table}[h!]
\centering
\caption{Hyperparameters: NC regime (LDS on Cora and Citeseer). Settings follow GSLB's per-dataset tuned configuration~\citep{li2023gslb}; values fall within the validated grids of the original LDS paper~\citep{franceschi2019learning}.}
\label{tab:hp_nc}
\small
\begin{tabular}{llp{6.5cm}}
\toprule
Category & Setting & Value \\
\midrule
\multicolumn{3}{l}{\textit{Backbone}} \\
GCN layers & --- & 2 \\
Hidden dim & --- & $128$ (Cora) / $64$ (Citeseer) \\
Dropout & --- & $0.5$ \\
\midrule
\multicolumn{3}{l}{\textit{Optimization}} \\
Optimizer (inner / outer) & --- & Adam / SGD with exponential decay \\
Inner LR $\eta_\theta$ & --- & $5\times10^{-3}$ (Cora) / $1\times10^{-2}$ (Citeseer) \\
Outer LR $\eta_\phi$ & --- & $1.0$ \\
Weight decay (inner) & --- & $5\times10^{-4}$ \\
Epochs & total / inner-only & $400$ / $400$ \\
Early stopping patience & --- & $20$ \\
Batch / split protocol & --- & Full-batch, Planetoid splits~\citep{yang2016planetoid} \\
\midrule
\multicolumn{3}{l}{\textit{Bilevel (LDS)}} \\
Inner steps $T$ & --- & $5$ (GSLB default) \\
Edge parameterization & --- & Bernoulli with sample count $S = 16$ \\
\midrule
\multicolumn{3}{l}{\textit{Frozen-$\phi$ control}} \\
$\phi$ initialization & --- & Original adjacency $A$ (matching LDS initialization) \\
Inner-loop schedule & --- & Identical to LDS \\
Outer step & --- & Skipped \\
\bottomrule
\end{tabular}
\end{table}

\begin{table}[h!]
\centering
\caption{Computational budget per configuration, single-seed measurement on NVIDIA L40S (48GB).}
\label{tab:hp_compute}
\small
\begin{tabular}{llcc}
\toprule
Regime & Setting & Wall-clock (h) / seed & Peak GPU mem (GB) \\
\midrule
ST (PeMS04, DiffConv) & Vanilla & $0.3$ & $1.7$ \\
ST (PeMS04, DiffConv) & Frozen-$\phi$ ($T{=}10$) & $3.8$ & $\sim 1.7$--$2.0$ \\
ST (PeMS04, DiffConv) & Bilevel ($T{=}10$)       & $4.5$ & $\sim 1.7$--$2.0$ \\
NC (Cora, GCN)         & Vanilla & ${<}0.01$ & $0.2$ \\
NC (Cora, GCN)         & Frozen-$\phi$        & $0.03$ & $1.5$ \\
NC (Cora, GCN)         & LDS                  & $0.08$ & $1.6$ \\
\bottomrule
\end{tabular}
\\[0.4em]
\footnotesize Bilevel/Vanilla overhead $\approx 15\times$. Frozen-$\phi$ retains the full $T$-step inner loop and accounts for $\sim 84\%$ of bilevel's wall-clock. Peak GPU memory is comparable across Frozen-$\phi$ and Bilevel configurations.
\end{table}

\paragraph{Total compute.} Approximately $1{,}600$ GPU-hours total: $\sim 1{,}490$ GPU-hours for ST experiments (5-seed main-table runs plus hyperparameter sweeps and replication runs), $\sim 24$ GPU-hours for NC experiments, and $\sim 80$ GPU-hours for scope-condition checks (GEN, Pro-GNN) and ablations not included in the main results.

\paragraph{Asset licenses.} Code dependencies: PyTorch (BSD-3), PyTorch Geometric (MIT), PyTorch Lightning (Apache 2.0), Torch Spatiotemporal~\citep{cini2022tsl} (MIT). The LDS reference implementation~\citep{franceschi2019learning} is distributed under a NEC Laboratories Europe academic non-commercial license; we use it for non-commercial research only and preserve the original license headers in any derived files. The GSLB benchmark~\citep{li2023gslb} and the static rewiring baselines (FoSR, SDRF, BORF, GTR) are distributed publicly without explicit license files; we credit each method through citation. Datasets: PeMS04/07/08 from Caltrans PeMS public sensor data; METR-LA and PeMS-BAY from the DCRNN release~\citep{li2018diffusion}; AirQuality from KDD Cup 2018~\citep{zheng2015airquality}; Cora and Citeseer under the LINQS terms via Planetoid splits~\citep{yang2016planetoid}.

%% file: main.bbl
\begin{thebibliography}{44}
\providecommand{\natexlab}[1]{#1}
\providecommand{\url}[1]{\texttt{#1}}
\expandafter\ifx\csname urlstyle\endcsname\relax
  \providecommand{\doi}[1]{doi: #1}\else
  \providecommand{\doi}{doi: \begingroup \urlstyle{rm}\Url}\fi

\bibitem[Alon and Yahav(2021)]{alon2021bottleneck}
Uri Alon and Eran Yahav.
\newblock On the bottleneck of graph neural networks and its practical implications.
\newblock In \emph{ICLR}, 2021.

\bibitem[Bai et~al.(2020)Bai, Yao, Li, Wang, and Wang]{bai2020adaptive}
Lei Bai, Lina Yao, Can Li, Xianzhi Wang, and Can Wang.
\newblock Adaptive graph convolutional recurrent network for traffic forecasting.
\newblock In \emph{NeurIPS}, 2020.

\bibitem[Banerjee et~al.(2022)Banerjee, Karhadkar, Wang, Alon, and Mont{\'u}far]{banerjee2022oversquashing}
Pradeep~Kumar Banerjee, Kedar Karhadkar, Yu~Guang Wang, Uri Alon, and Guido Mont{\'u}far.
\newblock Oversquashing in {GNN}s through the lens of information contraction and graph expansion.
\newblock In \emph{Allerton}, 2022.

\bibitem[Barrett and Dherin(2021)]{barrett2021implicit}
David~G.T. Barrett and Benoit Dherin.
\newblock Implicit gradient regularization.
\newblock In \emph{ICLR}, 2021.

\bibitem[Black et~al.(2023)Black, Wan, Nayyeri, and Wang]{black2023understanding}
Mitchell Black, Zhengchao Wan, Amir Nayyeri, and Yusu Wang.
\newblock Understanding oversquashing in {GNNs} through the lens of effective resistance.
\newblock In \emph{ICML}, 2023.

\bibitem[Chen et~al.(2020)Chen, Wu, and Zaki]{chen2020iterative}
Yu~Chen, Lingfei Wu, and Mohammed~J. Zaki.
\newblock Iterative deep graph learning for graph neural networks: Better and robust node embeddings.
\newblock In \emph{NeurIPS}, 2020.

\bibitem[Cini and Marisca(2022)]{cini2022tsl}
Andrea Cini and Ivan Marisca.
\newblock Torch spatiotemporal, 2022.
\newblock Software library.

\bibitem[Cini et~al.(2022)Cini, Marisca, and Alippi]{cini2022filling}
Andrea Cini, Ivan Marisca, and Cesare Alippi.
\newblock Filling the g\_ap\_s: Multivariate time series imputation by graph neural networks.
\newblock In \emph{ICLR}, 2022.

\bibitem[Cini et~al.(2023{\natexlab{a}})Cini, Marisca, Zambon, and Alippi]{cini2023taming}
Andrea Cini, Ivan Marisca, Daniele Zambon, and Cesare Alippi.
\newblock Taming local effects in graph-based spatiotemporal forecasting.
\newblock In \emph{NeurIPS}, 2023{\natexlab{a}}.

\bibitem[Cini et~al.(2023{\natexlab{b}})Cini, Zambon, and Alippi]{cini2023sparse}
Andrea Cini, Daniele Zambon, and Cesare Alippi.
\newblock Sparse graph learning from spatiotemporal time series.
\newblock \emph{Journal of Machine Learning Research}, 2023{\natexlab{b}}.

\bibitem[Di~Giovanni et~al.(2023)Di~Giovanni, Giusti, Barbero, Luise, Li{\`o}, and Bronstein]{digiovanni2023oversquashing}
Francesco Di~Giovanni, Lorenzo Giusti, Federico Barbero, Giulia Luise, Pietro Li{\`o}, and Michael~M. Bronstein.
\newblock On over-squashing in message passing neural networks: The impact of width, depth, and topology.
\newblock In \emph{ICML}, 2023.

\bibitem[Fatemi et~al.(2021)Fatemi, El~Asri, and Kazemi]{fatemi2021slaps}
Bahare Fatemi, Layla El~Asri, and Seyed~Mehran Kazemi.
\newblock {SLAPS}: Self-supervision improves structure learning for graph neural networks.
\newblock In \emph{NeurIPS}, 2021.

\bibitem[Franceschi et~al.(2019)Franceschi, Niepert, Pontil, and He]{franceschi2019learning}
Luca Franceschi, Mathias Niepert, Massimiliano Pontil, and Xiao He.
\newblock Learning discrete structures for graph neural networks.
\newblock In \emph{ICML}, 2019.

\bibitem[Gasteiger et~al.(2019)Gasteiger, Wei{\ss}enberger, and G{\"u}nnemann]{gasteiger2019diffusion}
Johannes Gasteiger, Stefan Wei{\ss}enberger, and Stephan G{\"u}nnemann.
\newblock Diffusion improves graph learning.
\newblock In \emph{NeurIPS}, 2019.

\bibitem[Hu et~al.(2022)Hu, Chang, Ma, and Shan]{hu2022efficient}
Minyang Hu, Hong Chang, Bingpeng Ma, and Shiguang Shan.
\newblock Learning continuous graph structure with bilevel programming for graph neural networks.
\newblock In \emph{IJCAI}, 2022.

\bibitem[Jiang et~al.(2023)Jiang, Wang, Yong, Jeph, Chen, Kobayashi, Song, Fukushima, and Suzumura]{jiang2023spatio}
Renhe Jiang, Zhaonan Wang, Jiawei Yong, Puneet Jeph, Quanjun Chen, Yasumasa Kobayashi, Xuan Song, Shintaro Fukushima, and Toyotaro Suzumura.
\newblock Spatio-temporal meta-graph learning for traffic forecasting.
\newblock In \emph{AAAI}, 2023.

\bibitem[Jin et~al.(2020)Jin, Ma, Liu, Tang, Wang, and Tang]{jin2020graph}
Wei Jin, Yao Ma, Xiaorui Liu, Xianfeng Tang, Suhang Wang, and Jiliang Tang.
\newblock Graph structure learning for robust graph neural networks.
\newblock In \emph{KDD}, 2020.

\bibitem[Karhadkar et~al.(2023)Karhadkar, Banerjee, and Mont{\'u}far]{karhadkar2023fosr}
Kedar Karhadkar, Pradeep~Kumar Banerjee, and Guido Mont{\'u}far.
\newblock {FoSR}: First-order spectral rewiring for addressing oversquashing in {GNNs}.
\newblock In \emph{ICLR}, 2023.

\bibitem[Li and Talwalkar(2019)]{li2019random}
Liam Li and Ameet Talwalkar.
\newblock Random search and reproducibility for neural architecture search.
\newblock In \emph{UAI}, 2019.

\bibitem[Li et~al.(2018)Li, Yu, Shahabi, and Liu]{li2018diffusion}
Yaguang Li, Rose Yu, Cyrus Shahabi, and Yan Liu.
\newblock Diffusion convolutional recurrent neural network: Data-driven traffic forecasting.
\newblock In \emph{ICLR}, 2018.

\bibitem[Li et~al.(2023)Li, Wang, Sun, Luo, Zhu, Chen, Luo, Zhou, Liu, Wu, Yu, and Wang]{li2023gslb}
Zhixun Li, Liang Wang, Xin Sun, Yifan Luo, Yanqiao Zhu, Dingshuo Chen, Yingtao Luo, Xiangxin Zhou, Qiang Liu, Shu Wu, Jeffrey~Xu Yu, and Liang Wang.
\newblock {GSLB}: The graph structure learning benchmark.
\newblock In \emph{NeurIPS Datasets and Benchmarks Track}, 2023.

\bibitem[Lighthill and Whitham(1955)]{lighthill1955kinematic}
M.~J. Lighthill and G.~B. Whitham.
\newblock On kinematic waves. {II}. {A} theory of traffic flow on long crowded roads.
\newblock \emph{Proceedings of the Royal Society A: Mathematical, Physical and Engineering Sciences}, 229\penalty0 (1178):\penalty0 317--345, 1955.

\bibitem[Liu et~al.(2023)Liu, Dong, Jiang, Deng, Deng, Chen, and Song]{liu2023staeformer}
Hangchen Liu, Zheng Dong, Renhe Jiang, Jiewen Deng, Jinliang Deng, Quanjun Chen, and Xuan Song.
\newblock Spatio-temporal adaptive embedding makes vanilla transformer {SOTA} for traffic forecasting.
\newblock In \emph{CIKM}, 2023.

\bibitem[Liu et~al.(2019)Liu, Simonyan, and Yang]{liu2019darts}
Hanxiao Liu, Karen Simonyan, and Yiming Yang.
\newblock {DARTS}: Differentiable architecture search.
\newblock In \emph{ICLR}, 2019.

\bibitem[Liu et~al.(2022)Liu, Zheng, Zhang, Chen, Peng, and Pan]{liu2022sublime}
Yixin Liu, Yu~Zheng, Daokun Zhang, Hongxu Chen, Hao Peng, and Shirui Pan.
\newblock Towards unsupervised deep graph structure learning.
\newblock In \emph{WWW}, 2022.

\bibitem[Manenti et~al.(2025)Manenti, Zambon, and Alippi]{manenti2025learning}
Alessandro Manenti, Daniele Zambon, and Cesare Alippi.
\newblock Learning latent graph structures and their uncertainty.
\newblock In \emph{ICML}, 2025.

\bibitem[Marisca et~al.(2025)Marisca, Bamberger, Alippi, and Bronstein]{marisca2025oversquashing}
Ivan Marisca, Jacob Bamberger, Cesare Alippi, and Michael~M. Bronstein.
\newblock Over-squashing in spatio-temporal graph neural networks.
\newblock In \emph{NeurIPS}, 2025.

\bibitem[Nguyen et~al.(2023)Nguyen, Hieu, Nguyen, Ho, Osher, and Nguyen]{nguyen2023revisiting}
Khang Nguyen, Nong~Minh Hieu, Vinh~Duc Nguyen, Nhat Ho, Stanley Osher, and Tan~Minh Nguyen.
\newblock Revisiting over-smoothing and over-squashing using {O}llivier-{R}icci curvature.
\newblock In \emph{ICML}, 2023.

\bibitem[Pei et~al.(2020)Pei, Wei, Chang, Lei, and Yang]{pei2020geomgcn}
Hongbin Pei, Bingzhe Wei, Kevin Chen-Chuan Chang, Yu~Lei, and Bo~Yang.
\newblock {Geom-GCN}: Geometric graph convolutional networks.
\newblock In \emph{ICLR}, 2020.

\bibitem[Platonov et~al.(2023)Platonov, Kuznedelev, Diskin, Babenko, and Prokhorenkova]{platonov2023critical}
Oleg Platonov, Denis Kuznedelev, Michael Diskin, Artem Babenko, and Liudmila Prokhorenkova.
\newblock A critical look at the evaluation of {GNN}s under heterophily: Are we really making progress?
\newblock In \emph{ICLR}, 2023.

\bibitem[Rajeswaran et~al.(2019)Rajeswaran, Finn, Kakade, and Levine]{rajeswaran2019meta}
Aravind Rajeswaran, Chelsea Finn, Sham~M. Kakade, and Sergey Levine.
\newblock Meta-learning with implicit gradients.
\newblock In \emph{NeurIPS}, 2019.

\bibitem[Shang et~al.(2021)Shang, Chen, and Bi]{shang2021discrete}
Chao Shang, Jie Chen, and Jinbo Bi.
\newblock Discrete graph structure learning for forecasting multiple time series.
\newblock In \emph{ICLR}, 2021.

\bibitem[Smith et~al.(2021)Smith, Dherin, Barrett, and De]{smith2021origin}
Samuel~L. Smith, Benoit Dherin, David~G.T. Barrett, and Soham De.
\newblock On the origin of implicit regularization in stochastic gradient descent.
\newblock In \emph{ICLR}, 2021.

\bibitem[Tang et~al.(2022)Tang, Qian, Liu, Du, Hu, and Li]{tang2022stlgsl}
Jiabin Tang, Tang Qian, Shijing Liu, Shengdong Du, Jie Hu, and Tianrui Li.
\newblock Spatio-temporal latent graph structure learning for traffic forecasting.
\newblock In \emph{IJCNN}, 2022.

\bibitem[Topping et~al.(2022)Topping, Di~Giovanni, Chamberlain, Dong, and Bronstein]{topping2022understanding}
Jake Topping, Francesco Di~Giovanni, Benjamin~Paul Chamberlain, Xiaowen Dong, and Michael~M. Bronstein.
\newblock Understanding over-squashing and bottlenecks on graphs via curvature.
\newblock In \emph{ICLR}, 2022.

\bibitem[Tori et~al.(2025)Tori, Holst, and Ginis]{tori2025effectiveness}
Floriano Tori, Vincent Holst, and Vincent Ginis.
\newblock The effectiveness of curvature-based rewiring and the role of hyperparameters in {GNN}s revisited.
\newblock In \emph{ICLR}, 2025.

\bibitem[Vicol et~al.(2022)Vicol, Lorraine, Pedregosa, Duvenaud, and Grosse]{vicol2022implicit}
Paul Vicol, Jonathan~P. Lorraine, Fabian Pedregosa, David Duvenaud, and Roger Grosse.
\newblock On implicit bias in overparameterized bilevel optimization.
\newblock In \emph{ICML}, 2022.

\bibitem[Wang et~al.(2021)Wang, Mou, Wang, Xiao, Ju, Shi, and Xie]{wang2021graph}
Ruijia Wang, Shuai Mou, Xiao Wang, Wanpeng Xiao, Qi~Ju, Chuan Shi, and Xing Xie.
\newblock Graph structure estimation neural networks.
\newblock In \emph{WWW}, 2021.

\bibitem[Wu et~al.(2019)Wu, Pan, Long, Jiang, and Zhang]{wu2019graph}
Zonghan Wu, Shirui Pan, Guodong Long, Jing Jiang, and Chengqi Zhang.
\newblock Graph {W}ave{N}et for deep spatial-temporal graph modeling.
\newblock In \emph{IJCAI}, 2019.

\bibitem[Yang et~al.(2020)Yang, Esperan{\c{c}}a, and Carlucci]{yang2020nas}
Antoine Yang, Pedro~M. Esperan{\c{c}}a, and Fabio~M. Carlucci.
\newblock {NAS} evaluation is frustratingly hard.
\newblock In \emph{ICLR}, 2020.

\bibitem[Yang et~al.(2016)Yang, Cohen, and Salakhutdinov]{yang2016planetoid}
Zhilin Yang, William~W. Cohen, and Ruslan Salakhutdinov.
\newblock Revisiting semi-supervised learning with graph embeddings.
\newblock In \emph{ICML}, 2016.

\bibitem[Zhang et~al.(2020)Zhang, Chang, Meng, Xiang, and Pan]{zhang2020stgsl}
Qi~Zhang, Jianlong Chang, Gaofeng Meng, Shiming Xiang, and Chunhong Pan.
\newblock Spatio-temporal graph structure learning for traffic forecasting.
\newblock In \emph{AAAI}, 2020.

\bibitem[Zheng et~al.(2015)Zheng, Yi, Li, Li, Shan, Chang, and Li]{zheng2015airquality}
Yu~Zheng, Xiuwen Yi, Ming Li, Ruiyuan Li, Zhangqing Shan, Eric Chang, and Tianrui Li.
\newblock Forecasting fine-grained air quality based on big data.
\newblock In \emph{KDD}, 2015.

\bibitem[Zhou et~al.(2023)Zhou, Zhou, Mao, Zhou, Chen, Tan, Zha, Feng, Chen, and Wang]{zhou2023opengsl}
Zhiyao Zhou, Sheng Zhou, Bochao Mao, Xuanyi Zhou, Jiawei Chen, Qiaoyu Tan, Daochen Zha, Yan Feng, Chun Chen, and Can Wang.
\newblock {OpenGSL}: A comprehensive benchmark for graph structure learning.
\newblock In \emph{NeurIPS Datasets and Benchmarks Track}, 2023.

\end{thebibliography}
